\documentclass[twoside]{article}

\usepackage{smile}

\usepackage[accepted]{aistats2021}

\usepackage{lipsum}

\newcommand\blfootnote[1]{%
  \begingroup
  \renewcommand\thefootnote{}\footnote{#1}%
  \addtocounter{footnote}{-1}%
  \endgroup
}

\usepackage{natbib}

\usepackage[utf8]{inputenc} % allow utf-8 input
\usepackage[T1]{fontenc}    % use 8-bit T1 fonts
\usepackage{hyperref}       % hyperlinks
\usepackage{url}            % simple URL typesetting
\usepackage{booktabs}       % professional-quality tables
\usepackage{amsfonts}       % blackboard math symbols
\usepackage{nicefrac}       % compact symbols for 1/2, etc.
\usepackage{microtype}      % microtypography
\usepackage{lipsum}
\usepackage{graphicx}
\usepackage{microtype}
\usepackage{subfigure}
\usepackage[shortlabels]{enumitem}
\usepackage{xcolor}
\def\norm#1{\left\lVert#1\right\rVert}
\def\|#1\|{\norm{#1}}

\begin{document}
\runningtitle{Regularization Matters in Training Overparametrized Neural Networks}
\runningauthor{Tianyang Hu*, Wenjia Wang*, Cong Lin, Guang Cheng}

\twocolumn[

\aistatstitle{Regularization Matters: A Nonparametric Perspective on Overparametrized Neural Network}

\aistatsauthor{Tianyang Hu* 
% \thanks{These authors contributed equally to this work.} 
\And Wenjia Wang* \And  Cong Lin \And Guang Cheng}

\aistatsaddress{ hu478@purdue.edu\\
Purdue University \And  wenjiawang@ust.hk\\
Hong Kong University of\\
Science and Technology \And 52174404011@stu.ecnu.edu.cn\\
East China \\Normal University \And chengg@purdue.edu\\
Purdue University } ]

\begin{abstract}
%There are a lot of recent advances in theoretical deep learning. 
Overparametrized neural networks trained by gradient descent (GD) can provably overfit any training data. 
However, the generalization guarantee may not hold for noisy data. 
From a nonparametric perspective, this paper studies how well overparametrized neural networks can recover the true target function in the presence of random noises. 
We establish a lower bound on the $L_2$ estimation error with respect to the GD iterations, which is away from zero without a delicate scheme of early stopping. 
In turn, through a comprehensive analysis of $\ell_2$-regularized GD trajectories, we prove that for overparametrized one-hidden-layer ReLU neural network with the $\ell_2$ regularization: 
(1) the output is close to that of the kernel ridge regression with the corresponding neural tangent kernel; 
(2) minimax {optimal} rate of the $L_2$ estimation error can be achieved. 
Numerical experiments confirm our theory and further demonstrate that the $\ell_2$ regularization approach improves the training robustness and works for a wider range of neural networks.
\end{abstract}

\section{INTRODUCTION}
Deep learning has shown outstanding empirical successes and demonstrates superior performance in many standard machine learning tasks, such as image classification \citep{krizhevsky2012imagenet, lecun2015deep, he2016deep}, generative modeling \citep{goodfellow2014generative,arjovsky2017wasserstein}, etc. 
Despite common accusations of being a black box with no theoretical guarantee, deep neural network (DNN) tends to achieve higher accuracy than other classical methods in various prediction tasks, which attracts plenty of interests from researchers. 
In contrast to the huge empirical success, little is yet settled from the theoretical side why DNN outperforms other methods. Without enough understanding, practical use of deep learning models could be inefficient and unreliable.\blfootnote{* These authors contributed equally to this work.}

Recently, many efforts have been devoted to provable deep learning methods with algorithmic guarantees, particularly training overparametrized neural networks by gradient descent (GD) or other gradient-based optimization. 
It has been shown that with enough overparametrization, e.g., neural network width tends to infinity, training DNN resembles a kernel method with a specific kernel called as ``neural tangent kernel'' (NTK) \citep{jacot2018neural}. 
In the NTK regime, GD can provably minimize the training error to zero in both regression \citep{du2018gradient,li2018learning,arora2019fine, zou2019improved} and classification \citep{ji2019implicit, ji2019polylogarithmic, lyu2019gradient} settings. 
%Most existing work focuses on the optimization perspective and generalization error of overparametrized DNNs. 
Corresponding generalization error bounds are developed to ensure prediction performance on unseen data. 
However, a closer inspection of these generalization results reveals that they only hold under the noiseless assumption, i.e., the response variable is deterministic given the explanatory variables. 
%In other words, overparametrized DNN trained with GD doesn't generalize under noisy data. 
For overparametrized neural networks, the training loss can be minimized to zero so that the generalization error equals the population loss, which cannot be zero in the presence of noises. As random noises are ubiquitous in the real world, theoretical guarantees and provable learning algorithms that take into account of random noises are much needed in practice. 

In contrast, classic nonparametric statistics literature demonstrate that in the presence of noises, the $L_2$ estimation error can still go to zero with possibly optimal rates as established in \cite{stone1982optimal}. 
To further investigate how overparametrized neural networks trained via GD work and how well they can learn the underlying true function with noisy data, we consider the classic nonparametric regression setting. 
Suppose we observe data $ \{(\bx_i, y_i)\}_{i=1}^n$, given by
\begin{align}\label{eqrelation}
    y_i = f^*(\bx_i)+\epsilon_i,
\end{align}
where $f^*$ is the ground truth, $\bx_i\in \RR^d$, 
%assumed to belong to some hypothesis class $\cF^*$
and $\epsilon_i$'s are i.i.d. random noises with mean 0 and finite variance $\sigma^2$. In this work, we consider neural network estimators $\hat f$ produced by overparametrized one-hidden-layer ReLU neural networks, where the number of neurons can be much larger than the sample size, and investigate how fast the $L_2$ estimation error $\|\hat{f}-f^*\|_2$ converges to zero as sample size grows.
% The goal is to construct a neural network estimator $\hat{f}$ from data to estimate $f^*$ and investigate how fast the $L_2$ estimation error $[\EE(\hat{f}-f^*)^2]^{1/2}$ converges to zero as sample size grows. 

Note that the $L_2$ convergence rate critically depends on the assumptions of the true function, e.g., linearity, smoothness, boundedness, etc., based on which minimax lower bounds are established \citep{siegel1957nonparametric}. 
An estimation method is said to be \textit{minimax-optimal} if its convergence rate achieves the lower bound, indicating that it performs the best in the worst possible scenario. 
The above nonparametric perspective provides a sharp characterization of the employed estimation method and complements the existing optimization/generalization framework. 
%In this work, we specifically investigate one-hidden-layer ReLU neural networks trained by GD. 

The main contributions of this paper are:
\begin{itemize}
    \item 
    We prove that overparametrized one-hidden-layer ReLU neural networks trained using GD do not recover the true function in the classic nonparametric regression setting (\ref{eqrelation}), i.e., the $L_2$ estimation error is bounded away from zero as sample size goes to infinity. To predict well on unseen data, a delicate early stopping rule has to be deployed. 
    
    \item
    We analyze the $\ell_2$-regularized GD trajectory and
    show that the $\ell_2$ penalty on network weights amounts to penalizing the reproducing kernel Hilbert space (induced by NTK) norm of the associated neural network. With $\ell_2$ regularization, overparametrized neural network trained by GD resembles the solution of kernel ridge regression. 
    % This correspondence is nontrivial and technically challenging.
    
    \item
    We further prove that by adding proper $\ell_2$ regularization, overparametrized neural network trained by GD achieves the \textit{minimax-optimal} $L_2$ convergence rate $n^{-{d}/{(4d-2)}}$, in recovering the ground truth in (\ref{eqrelation}). 
    
\end{itemize}
% technical difficulty:
% Our analysis on the $\ell_2$-regularized GD broadens the current scope of the NTK literature and 
The correspondence between overparametrized neural network trained by $\ell_2$-regularized GD and kernel ridge regression is nontrivial and technically challenging.
In spite of the well-established equivalence between NTK and infinite-width DNN trained by GD, there is a huge technical gap for finite-width overparametrized neural networks, especially when the training objective includes explicit regularization terms.

To sum up, this work broadens the current scope of the NTK literature and connects the recent advances in deep learning theory, e.g., analyzing the trajectory of GD updates, implicit bias of overparametrization, etc., to the classical results in nonparametric statistics. 
More specifically, our findings not only contribute to the theoretical (in particular, nonparametric) understanding of training overparametrized DNN on noisy data but also promotes the use of $\ell_2$ penalty or weight decay in practice for better theoretical guarantees.

\section{RELATED WORKS}

\paragraph{Neural Tangent Kernel} 
The seminal paper \citep{jacot2018neural} proves that the evolution of DNNs during training can be described by the so-called neural tangent kernel (NTK), which is central to characterize the convergence and generalization behaviors. 
%It is shown that while the NTK is random at initialization and varies during training, in the infinite-width limit it converges to an explicit limiting kernel that stays constant during training. 
\cite{du2018gradient, arora2019fine, li2018learning} investigate specifically for one-hidden-layer ReLU neural networks and show explicitly that with enough overparametrization, the weight vectors and the corresponding NTK do not change much during GD training. 
Similar investigations have been done for other neural networks and other settings \citep{zou2019improved, ji2019polylogarithmic}.
%The above-mentioned works mostly focus on the optimization and generalization perspectives with no prediction error guarantee for unseen data. Some 
Among others, \cite{arora2019fine,cao2019generalization} provide generalization error bounds and provable learning scenarios, but only hold for {noiseless} data. 

For noisy data, explicit regularizations have recently been considered in the NTK literature.
\cite{wei2019regularization} promote the $\ell_2$ penalty when using NTK by showing that in a constructed classification example, sample efficiency can benefit from the regularization.
\cite{hu2020simple} consider classification with noisy labels and propose to add $\ell_2$ regularization to ensure robustness. 
However, their analyses only apply to the kernel estimator directly using NTK and only relate to infinite width neural networks, which greatly restricts the model class capacity. 
As pointed out before, bridging the technical gap between NTK and finite-width overparametrized neural networks is technically challenging when the training objective includes an $\ell_2$ regularization term and we should not take it for granted. 
\cite{geifman2020similarity} demonstrate the similarity between the Laplace kernels and ReLU NTKs. 
However, in order for NTK to be a good characterization of neural network training, how wide is wide enough remains an active field of research \citep{nitanda2019gradient}. 
In comparison, we directly analyze GD trajectories of training finite-width neural networks (with and without $\ell_2$ regularization) and prove that the corresponding NTK solutions can be well-approximated after a polynomial number of GD iterations. 
%In comparison, our algorithm-based analysis demonstrates that NTK solutions can be well-approximated by neural networks trained for a polynomial number of GD iterations.
To the best of our knowledge, we are among the first to rigorously establish the $L_2$ convergence rate for trained neural networks under noisy data. 
\cite{nitanda2020optimal} recently provide similar convergence rate analysis by considering a particular penalized stochastic gradient descent algorithm but they require the neural network width to be exponential with $n$. 

\paragraph{Nonparametric Regression} In nonparametric statistics, \cite{stone1982optimal} shows that when $f^*$ is $d$-variate and $\beta$-time differentiable, the optimal rate of convergence for the $L_2$ estimation error is $n^{-\beta/(2\beta+d)}$. Many popular methods such as kernel methods, Gaussian process, splines, etc., achieve this rate.
It has been recently shown that DNN (with certain structures) can also achieve optimal convergence rates \citep{yarotsky2017error, schmidt2017nonparametric, bauer2019deep,liu2019optimal} and even for non-smooth functions \citep{imaizumi2018deep}. 
However, this type of results has two limitations. Firstly, they only apply to the empirical risk minimizer or some specially constructed DNNs without any algorithmic guarantee. Secondly, the theoretical analysis relies on delicate complexity control of the DNN family and cannot handle overparametrization, which is very common in practice. 
Therefore, the aforementioned results are less helpful in understanding deep neural network models with overparametrization and highly non-convex optimization properties.

Our algorithm-dependent statistical analysis bridges the gap between these two types of research. Based on the GD trajectories and the corresponding NTK, we are able to analyze the trained overparametrized neural networks within the nonparametric framework and show they can also achieve the optimal convergence rate with proper regularizations. 
 
% \cite{lee2020generalized, nitanda2020optimal} 

%provides valuable insights on how overparametrized neural networks performs on noisy data. 
% By utilizing the minimax framework in statistics, we are able to derive lower bound on the convergence rate and prove the optimality of trained overparametrized neural networks.

\section{PRELIMINARIES}
% We first introduce some preliminaries, as well as the settings in this work.
\paragraph{Notation}
For any function $f(\bx):\cX\to\RR$, denote %$\|f\|_p=\int_{\cX}|f(\bx)|^p \text{d}P_X(\bx)$  for $p\in[1,\infty)$ and 
$\|f\|_\infty=\sup_{\bx\in\cX}|f(\bx)|$ and $\|f\|_p=(\int_{\cX} |f(\bx)|^p d\bx)^{1/p}$. 
For any vector $\bx$, $\|\bx\|_p$ denotes its $p$-norm, for $1\leq p \leq \infty$.
$L_p$ and $l_p$ are used to distinguish function norms and vector norms. 
%For a given subset $B$ of $\cX$, we let $\|f\|_{\infty, B}=\sup_{\bx\in B}|f(\bx)|$.
For two given sequences $\{a_n\}_{n\in \NN}$ and $\{b_n\}_{n\in \NN}$ of real numbers, we write $a_n\lesssim b_n$ if there exists a constant $C>0$ such that $a_n\le C b_n$ for all sufficiently large $n$. Let $\Omega(\cdot)$ be the counterpart of $O(\cdot)$ that $a_n=\Omega(b_n)$ means $a_n\gtrsim b_n$.
Further, $a_n = \tilde{O}(b_n)$ and $a_n = \tilde{\Omega}(b_n)$ are used to indicate there are specific requirements for the multiplicative constants.
We write $a_n\asymp b_n$ if $a_n\lesssim b_n$ and $a_n\gtrsim b_n$. 
%denote $a_n\le c_1b_n$ and $a_n\ge c_2b_n$ for some \textit{specific} constants $c_1,c_2$. 
Let $[N]=\{1,\dots, N\}$ for $N\in\NN$ and let $\lambda_{\min}(\bA)$ be the minimum eigenvalue of a symmetric matrix $\bA$.
% For $a,b \in \RR$, denote $a\vee b =\max\{a, b\}$. 
% and $a\wedge b =\min\{a, b\}$. 
We use $\II$ to denote the indicator function and $\bI_d$ to denote the $d\times d$ identity matrix.
% Independently, identically distributed is abbreviated as i.i.d.. 
%$\sigma(\cdot)$ to denote the ReLU function $\sigma(z) = z\vee 0$, and 
$N(\mathbf{\bmu}, \bSigma)$ represents Gaussian distribution with mean ${\bmu}$ and covariance $\bSigma$ and 
${\rm poly}(t_1,t_2,\ldots)$ denotes some polynomial function with arguments $t_1,t_2,\ldots$. 

%\paragraph{Problem Setup.} 
% Assume data $ \{(\bx_i, y_i)\}_{i=1}^n$ are coming from
% \begin{align}\label{eqrelation}
%     y_i = f^*(\bx_i)+\epsilon_i,
% \end{align}
% where the true function $f^*$ is assumed to belong to some hypothesis class $\cF^*$ and $\epsilon_i$'s are i.i.d. Gaussian noises with mean zero and finite variance.
% % standard deviation $\sigma$. 
% As is usually assumed in the NTK literature \citep{arora2019fine, hu2020simple, bietti2019inductive}, we consider functions defined on the unit sphere $\mathbb{S}^{d-1}$, and $\bx_i \in \mathbb{S}^{d-1}$, i.e., $\|\bx_i\|_2=1$ for any $i=1,2,\cdots,n$.
% % we assume all $\bx_i$'s to be on the unit sphere $\mathbb{S}^{d-1}$, i.e. $\|\bx_i\|_2=1$ for any $i=1,2,\cdots,n$.
% Throughout this work, we assume that $\bx_i$'s are uniformly distributed on $\mathbb{S}^{d-1}$. Let $\by=(y_1,\cdots,y_n)^\top$, $\by^*=(f^*(x_1),\cdots,f^*(x_n))^\top$ and $\bepsilon=(\epsilon_1,\cdots,\epsilon_n)^\top$. Then we can write 
% $\by=\by^*+\bepsilon.$

% Squared loss $\Phi(f) = \frac{1}{2}\sum_{i=1}^n(y_i-f(\bx_i))^2$ is usually used to measure the empirical performance.
% The goal of nonparametric regression is to construct an estimator $\hat{f}\in\cF$ from data to estimate the true function $f^*\in\cF^*$. The quantity of interest is the prediction $L_2$ loss 
% % on unseen data $\bx$
% on the entire space 
% % $\cX$
% $\mathbb{S}^{d-1}$, i.e. $\|\hat{f}-f^*\|_2$ and we want to investigate how fast it converges to zero. 

\paragraph{Neural Network Setup} 
Consider the one-hidden-layer ReLU neural network family $\mathcal{F}$ with $m$ nodes in the hidden layer, expressed as 
\[f_{\bW,\ba}(\bx) = \frac{1}{\sqrt{m}}\sum_{r=1}^{m} a_r \sigma(\bw_r^\top \bx),\]
where $\bx\in \mathbb{R}^d$ denotes the input, $\bW=(\bw_1,\cdots,\bw_m)\in \mathbb{R}^{d\times m}$ is the weight matrix in the hidden layer, $\ba=(a_1,\cdots,a_m)^\top\in \mathbb{R}^m$ is the weight vector in the output layer, $\sigma(z) = \max\{0,z\}$ is the rectified linear unit (ReLU).  
The initial values of the weights are independently generated from 
\begin{align*}
    \bw_r(0) ~\sim~ N(\zero,\tau^2\bI_m),~~a_r~\sim~ \mathrm{unif}\{-1, 1\},~~\forall r\in [m].
\end{align*}
When $m\gg n$,
the neural network is highly overparametrized. %\cite{arora2019fine,li2018learning} show that gradient descent (GD) can effectively minimize the empirical loss and the resulting neural network is not far away from its initialization. 
As is usually assumed in the NTK literature \citep{arora2019fine, hu2020simple, bietti2019inductive}, we consider data on the unit sphere $\mathbb{S}^{d-1}$, i.e., $\|\bx_i\|_2=1$ for any $i\in [n]$.
% we assume all $\bx_i$'s to be on the unit sphere $\mathbb{S}^{d-1}$, i.e. $\|\bx_i\|_2=1$ for any $i=1,2,\cdots,n$.
Throughout this work, we further assume that $\bx_1,\ldots,\bx_n$ are { uniformly} distributed on $\mathbb{S}^{d-1}$ so that $\EE_{\bx\sim {\rm unif}(\mathbb{S}^{d-1})}(\hat{f}(\bx)-f^*(\bx))^2$ and $\|f-f^*\|_2^2$ are equal up to a constant multiplier and thus will be used interchangeably.  

\paragraph{Gradient Descent}
Let $\by=(y_1,\cdots,y_n)^\top$ and $\bepsilon=(\epsilon_1,\cdots,\epsilon_n)^\top$. Denote $u_i = f_{\bW,\ba}(\bx_i)$ to be the network's prediction on $\bx_i$ and let $\bu = (u_1,...,u_n)^\top$. 
Without loss of generality, we consider fixing the second layer $\ba$ after initialization and only training the first layer $\bW$ by GD. 
Fixing the last layer is not a strong restriction since $a\cdot\sigma(z) = \mbox{sign}(a)\cdot\sigma(|a|z)$ and we can always reparametrize the network to have all $a_i$'s to be either $1$ or $-1$. 
Denote the empirical squared loss as
$
    \Phi(\bW) = \frac{1}{2}\|\by-\bu\|_2^2.
$
The gradient of $\Phi(\bW)$ w.r.t. $\bw_r$ can be written as
\begin{align*}
    \frac{\partial \Phi(\bW)}{\partial \bw_r} = \frac{1}{\sqrt{m}}a_r\sum_{i=1}^n (u_i - y_i)\mathbb{I}_{r,i}\bx_i,\quad r\in[m],
\end{align*}
where $\mathbb{I}_{r,i} = \mathbb{I}\{\bw_r^\top\bx_i\geq 0\}$. Then 
the GD update rule at the $k$-th iteration is given by
\begin{align*}
    \bw_r(k+1)=\bw_r(k)-\eta\frac{\partial \Phi(\bW)}{\partial \bw_r}\biggm|_{\bW=\bW(k)},
\end{align*}
where $\eta>0$ is the step size (a.k.a. learning rate). In the rest of this work, we use $k$ to index variables at the $k$-th iteration, e.g., $u_i(k) = f_{\bW(k),\ba}(\bx_i)$, etc. Define $\mathbb{I}_{r,i}(k)=\mathbb{I}\{\bw_r(k)^{\top}\bx_i\ge 0\}$, $\bZ(k)\in\RR^{md\times n}$ that
\begin{align*}
    \bZ(k)=\frac{1}{\sqrt{m}}
    \begin{pmatrix}
    a_1\mathbb{I}_{1,1}(k)\bx_1 & \dots & a_1\mathbb{I}_{1,n}(k)\bx_n\\
    \vdots & \ddots & \vdots \\
    a_m\mathbb{I}_{m,1}(k)\bx_1 & \dots & a_m\mathbb{I}_{m,n}(k)\bx_n
    \end{pmatrix}
\end{align*}
and $\bH(k)=\bZ(k)^{\top}\bZ(k)$. It is shown that matrices $\bZ(k)$ and $\bH(k)$ are close to $\bZ(0)$ and $\bH(0)$, respectively for any $k$, when $m$ is sufficiently large \citep{arora2019fine}. We can rewrite the GD update rule as 
\begin{align}
 \label{GD}
    \mathrm{vec}(\bW(k+1))=\mathrm{vec}(\bW(k))-\eta \bZ(k)(\bu(k)-\by),
\end{align}
where $\mathrm{vec}(\bW) =(\bw_1^\top,\cdots,\bw_m^{\top})^\top\in \mathbb{R}^{md\times 1}$ is the vectorized weight matrix. 
% Also define $\mathbb{I}_{r}(\bx) = \mathbb{I}\{\bw_r^\top\bx\geq 0\}$ and
% \begin{align*}
% \bz(\bx)=\frac{1}{\sqrt{m}}
%     \begin{pmatrix}
%     a_1\mathbb{I}_{1}(\bx)\bx\\
%     \vdots\\
%     a_m\mathbb{I}_{m}(\bx)\bx
%     \end{pmatrix}\in\RR^{md\times 1},
% \end{align*}
% Let $\bZ = (\bz(\bx_1),...,\bz(\bx_n))|_{\bW=\bW(k)}\in \mathbb{R}^{md\times n}$
% define $\bH(l) = \bZ(k)^\top\bZ(k)\in \mathbb{R}^{n\times n}$, note that
% the $ij$-th entry of $\bH(l)$ is
% \begin{align*}
%     \bH_{ij}(k)=\frac{\bx_i^\top\bx_j}{m}\sum_{r=1}^m\mathbb{I}_{r,i}\mathbb{I}_{r,j}\biggm|_{\bW=\bW(k)}.
% \end{align*}

\paragraph{Kernel Ridge Regression with NTK} %Reproducing Kernel Hilbert Space.
The study of one-hidden-layer ReLU neural networks is closely related to the NTK defined as
\begin{align}\label{ntkh}
    h(\bs,\bt) = & \mathbb{E}_{\bw \sim N(0,\bI_d)}\rbr{\bs^\top\bt \ \mathbb{I}\{\bw^\top\bs \geq 0,\bw^\top\bt \geq 0\}}\nonumber\\
    =& \frac{\bs^\top\bt(\pi - \arccos(\bs^\top\bt))}{2\pi},
\end{align}
where $\bs,\bt$ are $d$-dimensional vectors. It can be shown that $h$ is positive definite on the unit sphere $\mathbb{S}^{d-1}$ \citep{bietti2019inductive}. Let the Mercer decomposition of $h$ be
$
    h(\bs,\bt) = \sum_{j=0}^\infty \lambda_j\varphi_j(\bs)\varphi_j(\bt),
$
where $\lambda_1\geq \lambda_2\geq...\geq 0$ are the eigenvalues, and $\{\varphi_j\}_{j=1}^{\infty}$ is an orthonormal basis. 

The following lemma states the decay rate of eigenvalues of the NTK associated with one-hidden-layer ReLU neural networks, as a key technical contribution of this work.
\begin{lemma}\label{lemeigendecay}
Let $\lambda_j$ be the eigenvalues of NTK $h$ defined above. Then we have $\lambda_j \asymp j^{-\frac{d}{d-1}}$.
\end{lemma} 
Let $\mathcal{N}$ denote the reproducing kernel Hilbert space (RKHS) generated by $h$ on $\mathbb{S}^{d-1}$, equipped with norm $\|\cdot\|_{\mathcal{N}}$. 
For an unknown function $f^*\in \cN$, the kernel ridge regression minimizes
\begin{align}\label{regpro}
    \min_{f\in \cN}\frac{1}{2}\sum_{i=1}^n(y_i - f(\bx_i))^2 + \frac{\mu}{2}\|f\|_{\mathcal{N}}^2,
\end{align}
where $\mu>0$ is a tuning parameter controlling the regularization strength. The representer theorem says that the solution to \eqref{regpro} can be written as
\begin{align}\label{eqsolukrr}
    \hat f(\bx) = h(\bx,\bX)(\bH^\infty + \mu \bI_n)^{-1}\by
\end{align}
for any point $\bx\in\RR^d$, where 
$h(\bx,\bX) = (h(\bx,\bx_1),...,h(\bx,\bx_n))\in\RR^{1\times n} $
and $\bH^{\infty}=\rbr{h(\bx_i,\bx_j)}_{n\times n}
$ ($\bH^{\infty}$ is usually called the NTK matrix). 
In the following theorem, we show that the function $\hat{f}$ is close to the true function $f^*$ under the $L_2$ metric. 
\begin{theorem}\label{thmkrr}
Let $\hat{f}$ be as in (\ref{eqsolukrr}). By choosing $\mu \asymp n^{(d-1)/(2d-1)}$, we have
\begin{align*}
    \|\hat{f} - f^*\|_{2}^2 = O_{\mathbb{P}}\rbr{n^{-\frac{d}{2d-1}}},\quad \|\hat{f}\|_{\mathcal{N}}^2 = O_{\mathbb{P}}(1).
\end{align*}
\end{theorem}
The proof of the convergence rate requires an accurate characterization of the complexity of $\cN$, which is determined by the eigenvalues and eigenfunction expansion of the NTK $h$. 
If the eigenvalues decay at rate $\lambda_j\asymp j^{-2\nu}$, the corresponding minimax optimal rate is $n^{-2\nu/(2\nu+1)}$ \citep{yuan2016minimax,raskutti2014early}.
Building on the the eigenvalue decay rate established in Lemma \ref{lemeigendecay}, it can be shown that the $L_2$ estimation rate in Theorem \ref{thmkrr} is {minimax-optimal}. %\citep{yuan2016minimax,raskutti2014early}.
% According to \cite{stone1985additive, yang1999information}, the above $L_2$ estimation rate is {minimax-optimal}.  

%\alert{add minimax optimality}.
% The function $\hat{g}_n$ plays an important role in our analysis. As we will see later, 

% Therefore, it can be used to generate a RKHS, denoted by $\mathcal{N}$. Let $\|g\|_{\mathcal{N}}$ be the norm of $g\in \mathcal{N}$.

In the rest of this work, we assume that $f^*\in\cN$.

\section{PROBLEMS OF GRADIENT DESCENT FROM THE NONPARAMETRIC PERSPECTIVE
%Early Stopping of Gradient Descent Algorithm
}\label{secwithoutp}

In this section, we consider training overparametrized neural networks with the GD update rule (\ref{GD}). 
Among others, \cite{arora2019fine,du2018gradient} prove that as iteration $k\to \infty$, the training data are interpolated, achieving zero training loss. 
However, in the presence of noises, i.e., $\epsilon_i$ in (\ref{eqrelation}), such an overfitting to the training data can be harmful for recovering the ground truth.
The following theorem shows that if $k$ is too small or too large, the $L_2$ estimation error of the trained neural network is bounded away from zero.
%in order for the trained neural network to generalize in the noisy regression case, certain early stopping rule has to be deployed. 

\begin{theorem}\label{thm:gd}
% Suppose the GD iteration stops at iteration number $K$. Suppose $\tau = O\rbr{\frac{\lambda_0\delta}{n}}$ and $m\ge \frac{1}{\tau^2}{\rm poly}\rbr{n, \frac{1}{\lambda_0}, \frac{1}{\delta}}$. \alert{WW: $\eta$'s rate?} Let $\lambda_0$ be the largest number that with probability at least $1-\delta$, $\lambda_{\min}(\bH^\infty)\geq \lambda_0$. If $K \ll n/\lambda_0$ \alert{\bf WW: check it is $\ll$ or $\lesssim$} or $K \gg n^2/\lambda_0^2$, then there exists a constant $c>0$ such that as $n\to\infty$, with high probability $$\|f_{\bW(k),\ba}-f^*\|_2>c.$$ \alert{\bf WW: check is it high probability or $O_{\mathbb{P}}$, or $\Omega(1)$.}
Fix a failure probability $\delta\in(0,1)$. Let $\lambda_0$ be the largest number that with probability at least $1-\delta$, $\lambda_{\min}(\bH^\infty)\geq \lambda_0$. Suppose $m\ge\tau^{-2}{\rm poly}\rbr{n, \frac{1}{\lambda_0}, \frac{1}{\delta}}$, $\eta =\tilde{O}\rbr{\frac{\lambda_0}{n^2}}$, and $\tau =\tilde{O}\rbr{\frac{\lambda_0\delta}{n}}$. For sufficiently large $n$, if the iteration $k=\tilde{\Omega} \rbr{\frac{\log n}{\eta\lambda_0}}$ or $k=\tilde{O}\rbr{ \frac{1}{n\eta}}$, then with probability at least $1-2\delta$, we have \[\mathbb{E}_{\bepsilon}\|f_{\bW(k),\ba}-f^*\|_{2}^2  =\Omega(1).\]
% $\eta \leq C_1\frac{\lambda_0}{n^2}}$, and $\tau \leq \frac{C_2\lambda_0\delta}{n}$. If $K\geq C_3\rbr{\frac{\log n}{\eta\lambda_0}}$ or $K\leq \frac{C_4}{n\eta}$, then we have with probability at least $1-2\delta$, $$\mathbb{E}_{\bepsilon}\|f_{\bW(k),\ba}-f^*\|_{2}^2  \ge  C_5.$$
% . With , the following statements are true.
% \begin{enumerate}
%     \item[(a)] Let $I_1 = \|(\bH^{\infty})^{-1}h(\bX,\cdot)\|_2$. Suppose $\tau \leq \frac{C_2\lambda_0\delta}{n}I_1$ for some constant $C_2>0$. If , then $$\mathbb{E}_{\bepsilon}\|f_{\bW(k),\ba}-f^*\|_{2}^2  \ge  C_4I_1^2 - O\rbr{n^{-1}}.$$
%     \item[(b)] Suppose $\tau \leq \frac{C_5\lambda_0\delta}{n}$ for some constant $C_5>0$. If , then $$\mathbb{E}_{\bepsilon}\|f_{\bW(k),\ba}-f^*\|_{2}^2  \ge  C_7 - O\rbr{n^{-1}}.$$
% \end{enumerate}
% The constants $C_i$, $i\in [5]$ depend on $f^*$ but not $n$.

\end{theorem}

The conditions on $m,\eta,$ and $\tau$ have the same rates as those in Theorem 5.1 of \cite{arora2019fine}, but the constants requirements are different. The probability $1-2\delta$ in Theorem \ref{thm:gd} comes from the randomness of $\lambda_{\min}(\bH^\infty)$ and $(\bW(0), \ba)$. 
Theorem \ref{thm:gd} states that the estimation error for non-regularized one-hidden-layer neural networks is bounded away from zero by some constant if trained for too short or too long. The latter scenario indicates that overfitting is harmful in terms of the $L_2$ estimation error. Similar results have been shown in \cite{kohler2019over} for specifically designed overparametrized DNNs that is a linear combination of $\Omega(n^{10d^2})$ smaller neural networks, which is much more restrictive than ours.

% The result is not surprising when considering a simple case where $d = 1$ and $f^*\equiv 0$. As $k\to\infty$, the neural network converges to the minimum RKHS norm solution that interpolates all the training data \citep{bietti2019inductive}. 
% The $L_2$ loss in this case if just $\|\hat{f}\|_2$, 
% the derivatives around all $x_i$'s need to diverge to infinity, which clearly contradicts the minimum norm condition. 
In order to have low $L_2$ estimation errors, Theorem \ref{thm:gd} implies that the iteration number $k$ must satisfy
$(\eta\lambda_0)^{-1}\log n\lesssim k \lesssim (n\eta)^{-1}$. However, deriving a precise order of $k$, which leads to the optimal rate of convergence, could be extremely challenging. 
Alternatively, we consider the infinite-width limit of one-hidden-layer ReLU networks, i.e., directly using the NTK (\ref{ntkh}) in kernel regression. This may shed some light on the optimal stopping time for practical overparametrized neural networks. 

In kernel regression, the objective becomes
% In practice, we usually terminate the iterative algorithm in finite steps for kernel methods \citep{raskutti2014early, wei2017early, liu2018early}. NTK works can have good prediction if proper early stopping rule is applied. 
% Since training overparametrized one-hidden-layer neural network can be well approximated by the corresponding kernel $\bH^\infty$, it suffices to consider gradient descent directly using the linearized model, 
\begin{align}
\label{eqn:approx_kernel}
\min_{f\in\cN} \frac{1}{2}\sum_{i=1}^n(y_i-f(\bx_i))^2,
\end{align}
whose solution can be explicitly expressed as $h(\bx,\bX)(\bH^\infty)^{-1}\by$, by setting $\mu=0$ in (\ref{eqsolukrr}). 
However, inverting the kernel matrix can be computationally intensive. In practice, gradient-based methods are often applied to solve (\ref{eqn:approx_kernel}) \citep{raskutti2014early}.
The following theorem establishes estimation error results for the NTK estimators trained by GD, complementary to Theorem~\ref{thm:gd}. 
\begin{theorem}
\label{thm:kernel_earlystop}
Consider using GD to optimize (\ref{eqn:approx_kernel}) with a sufficiently small step size $\eta$ depending on $n$ (but not on $k$). There exists a stopping time $k^*$ depending on data, such that
\[\EE\|\hat{f}_{k^*}-f^*\|_2^2 = O\rbr{n^{-\frac{d}{2d-1}}},\] where $\hat f_{k}$ is the predictor obtained at the $k$-th iteration. Moreover, if $k\to\infty$, the interpolated estimator $\hat{f}_{\infty}$ satisfies
\[\EE\|\hat{f}_\infty-f^*\|_2^2 = \Omega(1).\]
\end{theorem}

To specify the optimal stopping time $k^*$ in Theorem \ref{thm:kernel_earlystop}, we first introduce the local empirical Rademacher complexity defined as
\[
\hat{\cR}_{\bH^\infty}(\varepsilon) := \rbr{ \frac{1}{n}
  \sum_{i=1}^n \min \big \{ {\hat{\lambda}_i}/{n}, \varepsilon^2 \big \}}^{1/2},
\]
which relies on the eigenvalues $\hat{\lambda}_1 \geq
\cdots \geq \hat{\lambda}_n > 0$ of  $\bH^\infty$. 
Then, the stopping time $k^*$ is defined to be
\begin{align}
\label{eqn:kstar}
k^* & := \argmin \biggr \{ k \in \NN \, \mid
\hat{\cR}_{\bH^\infty} \big(\frac{1}{\sqrt{\eta{k}}}\big) > \frac{1}{2 e \sigma
\eta {k}} \biggr \} - 1.
\end{align}
In essence, the optimal stopping time decreases with the noise level $\sigma$ and increases with the model complexity, measured by the eigenvalues of $\bH^\infty$.

\begin{remark}($k^*$ for neural networks)
To derive the order of $k^*$ for overparametrized neural network, a sharp characterization of the eigen-distribution of $\bH^\infty$ is needed. To the best of the authors' knowledge, no such results are available yet. Even though as $m\to \infty$, neural network resembles its linearization (NTK), it doesn't necessarily mean such a stopping rule can be easily derived for finite-width neural networks. In general, theoretical guarantees of an early stopping rule for training overparametrized neural networks is challenging and left for future work. %and computationally intensive.
%since at each iteration, we need to calculate the eigenvalues of the . 
\end{remark}

Besides early stopping, explicit regularizations are usually employed in deep learning models to balance the bias-variance trade-off and prevent overfitting, for example, weight decay \citep{krogh1992simple}, batch normalization \citep{ioffe2015batch}, dropout \citep{srivastava2014dropout}, etc., to prevent overfitting. 
In the next section, we investigate the $\ell_2$ regularization \citep{bilgic2014fast, van2017l2, phaisangittisagul2016analysis} and demonstrate its effectiveness in the nonparametric regression setting.
% which has been extensively studied especially in statistics as a way to balance the bias-variance trade-off and prevent overfitting.
% Let $K$ be the total number of iterations. 
% In this work we are interested in the $L_2$ prediction error $\|\hat{f} - f^*\|_{2}$, where
% \begin{align}\label{hatf}
%     \hat{f}(\bx) = \frac{1}{\sqrt{m}}\sum_{r=1}^{m} a_r \sigma(\bw_r(k)^\top \bx).
% \end{align}

\section{$\ell_2$-REGULARIZED GRADIENT DESCENT FOR NOISY DATA}
\label{secwithp}
%Since GD requires delicate choices of the early stopping time, 
Without any regularization, GD overfits the training data and the estimation error is bounded away from zero. Instead, we propose using the $\ell_2$-regularized gradient descent defined as
\begin{align}\label{modifiedgd}
    {\rm vec}(\bW_D(k+1)) =& {\rm vec}(\bW_D(k)) - \eta_1 \bZ_D(k)(\bu_D(k)-\by) \nonumber\\
    &- \eta_2\mu{\rm vec}(\bW_D(k)),
\end{align}
where $\eta_1,\eta_2>0$ are step sizes, and $\mu>0$ is a tuning parameter. It can be easily seen that \eqref{modifiedgd} is the GD update rule on the following loss function
\begin{align}\label{eqn:phi1}
    \Phi_1(\bW) = \frac{1}{2}\|\by-\bu\|_2^2 + \frac{\mu}{2}\|{\rm vec}(\bW)\|_2^2.
\end{align}

%\begin{remark}($\ell_2$ Regularization and Weight Decay)
% Weight decay is a popular training technique for learning DNNs. 
%The idea of weight decay is simply to prevent overfitting. % Every time we update a weight $w$ with the gradient $\nabla J$ in respect to $w$, we also subtract from it $\lambda\cdot w$. This gives the weights a tendency to decay towards zero, hence the name. 
The $\ell_2$ regularization has long been used in practical training neural networks and is equivalent to  ``weight decay'' \citep{krogh1992simple} when using GD \citep{loshchilov2017decoupled}. 
% It is known that the $\ell_2$ regularization and weight decay \citep{krogh1992simple} are equivalent for standard stochastic gradient descent (after being rescaled by some learning rate). 
In the NTK literature, $\ell_2$ regularization is also considered as a way to improve generalization \citep{wei2019regularization, hu2020simple}. However, we are among the first to directly analyze the $\ell_2$-regularized GD trajectories of overparametrized neural networks and show its connection to kernel ridge regression using NTK.
%\end{remark} %https://towardsdatascience.com/understanding-the-scaling-of-l%C2%B2-regularization-in-the-context-of-neural-networks-e3d25f8b50db
In the rest of this work, we use subscript $D$ to denote the variables under the regularized GD (\ref{modifiedgd}), e.g., $\bu_D(k)$ for the predictions at the $k$-th iteration.

\begin{theorem}\label{thm:withPenalty}
Let $\lambda_0$ be the largest number such that with probability at least $1-\delta_n$, $\lambda_{\min}(\bH^\infty)\geq \lambda_0$, and $\delta_n \rightarrow 0$ as $n$ goes to infinity\footnote{Potential dependency of $\lambda_0$ on $n$ is suppressed for notational simplicity.}.  For sufficiently large $n$, suppose $\mu \asymp n^{\frac{d-1}{2d-1}}$, $\eta_1 \asymp \eta_2 = o(n^{-\frac{3d-1}{2d-1}})$, $\tau=O(1)$, $m\geq \tau^{-2}{\rm ploy}(n,\lambda_0^{-1})$, and the iteration number $k$ satisfies 
$\log \rbr{{\rm ploy}_1(n,\tau,1/\lambda_0)} \lesssim \eta_2\mu k \lesssim \log \rbr{{\rm ploy}_2(\tau,1/n,\sqrt{m})}.$
% satisfies that $(1-\eta_2\mu)^k\geq \tau^{-1}{\rm ploy}(n,1/\sqrt{m})$, $(1-\eta_2\mu)^{2K} = O(1)$, and $\frac{n\tau^2}{\lambda_0}(1-\eta_2\mu)^{2K} = O(1)$. 
Then we have
\begin{gather} 
    \|\bu_D(k) - \bH^\infty(C\mu I+\bH^\infty)^{-1}\by\|_2 =  O_{\mathbb{P}}\left(\sqrt{n}(1-\eta_2 \mu)^k \right),\label{thmwithpstate1}\\
    \|{\rm vec}(\bW_D(k)) - (1-\eta_2\mu)^{k}{\rm vec}(\bW_D(0))\|_2 =  O_{\mathbb{P}}(1),\label{thmwithpstate2}
\end{gather}
for some constant $C>0$. 
Moreover, during the training process, the mean squared loss satisfies \begin{align}
\label{eqn:Phi}
    \Phi(\bW_D(k))/n \le  (1-\eta_2 \mu)^k\Phi(\bW_D(0))/n + O_{\PP}(1).
\end{align}
\end{theorem}
In the above theorem, three upper bounds are provided. In \eqref{thmwithpstate1}, we provide an upper bound on the difference between the prediction using one-hidden-layer neural networks and that obtained by \eqref{eqsolukrr}, which converges to zero as the sample size goes to infinity. This indicates that the $\ell_2$ penalty on neural network weights has similar effects to penalizing the RKHS norm as in (\ref{regpro}). 
Combining \eqref{thmwithpstate1} and Theorem \ref{thmkrr}, we can conclude that the $\ell_2$-regularized one-hidden-layer ReLU neural network recovers the true function on the training data points $\bx_1,\ldots,\bx_n$.

In \eqref{thmwithpstate2}, we provide an upper bound on the distance between the weight matrix at the $k$-th iteration and the  ``decayed''  initialization $\bW_D(0)$. Under the conditions in Theorem \ref{thm:withPenalty}, 
their distance measured in Frobenius norm is bounded by some constant depending on the underlying true function.
%they are not too far away from each other with probability tending to one, up to a constant depending on the underlying function. 
Unlike the results in \cite{arora2019fine}, the upper bound presented in \eqref{thmwithpstate2} {does not} depend on data. Therefore, as long as the underlying function is within the RKHS generated by NTK, the total movement of all the weights is not large even if the data observed are corrupted by noises.

In (\ref{eqn:Phi}), we give a characterization of how the training objective decreases over iterations, which is reminiscent of Theorem 4.1 in \cite{du2018gradient}. Unlike the results without regularization, our $\ell_2$-regularized objective is not expected to converge to zero, i.e., no data interpolation, which is essential to ensure the best trade-off between the bias and variance.
% good prediction for unseen data with relatively small RKHS norm.
\begin{remark}\label{rmk:k}
(More iterations)
The required iteration number $k$ in Theorem \ref{thm:withPenalty} is approximately $(\eta_2\mu)^{-1}$, up to a logarithmic term. We believe the upper bound on $k$ is not necessary and may be relaxed.
% Theorem \ref{thm:withPenalty} requires the iteration number to be upper bounded, which we believe to be an artifact of the proof. 
The stated results are expected to hold if $k\to\infty$ and we conjecture that the output will converge to the optimal solution of kernel ridge regression as in (\ref{eqsolukrr}). Simulation results in Section \ref{sec:simulation} support our conjecture and we leave the technical proof for future work.
% With the appropriate setting of the penalty parameter $\mu$, kernel ridge regression is known to achieve minimax-optimal error \citep{friedman2004gradient, sara2000, mendelson2002geometric}.
\end{remark}

Next, we extend the results in Theorem \ref{thm:withPenalty} and establish the $L_2$ convergence rate for neural networks trained with $\ell_2$-regularized GD.
\begin{theorem}\label{thm:withPenaltyl2small}
Suppose the assumptions of Theorem \ref{thm:withPenalty} hold. Then we have
\[\|f_{\bW_D(k),\ba}-f^*\|_{2}^2 = O_{\mathbb{P}}(n^{-\frac{d}{2d-1}}).\]
\end{theorem}
The above theorem states that with probability tending to one, the neural network estimator can still recover the true function with the optimal convergence rate of $n^{-\frac{d}{2(2d-1)}}$, demonstrating the effectiveness of the $\ell_2$ regularization for noisy data. 
% Comparing to other optimality results established for neural networks \citep{schmidt2017nonparametric, bauer2019deep}, we highlight two distinctive attributes of our analysis:
% (1) our convergence rate is obtainable using the $\ell_2$-regularized GD; 
% (2) our result applies to overparametrized neural networks. By incorporating the algorithm into our analysis, 
Unlike other optimality results established for neural networks \citep{schmidt2017nonparametric, bauer2019deep}, 
our convergence rate result applies to overparametrized networks and is obtainable using the $\ell_2$-regularized GD.

\section{NUMERICAL STUDIES}
\label{sec:simulation}
In practice, regularization techniques are widely used in training deep learning models. 
Among others, \cite{van2017l2, caruana2001overfitting, prechelt1998early,zhang2016understanding,lewkowycz2020training} have investigated the effectiveness of $\ell_2$ regularization and early stopping in training DNNs, and comprehensive comparisons have been made empirically against other regularization techniques. 
Therefore, one major goal of this section is not to show state-of-the-art performance using $\ell_2$ regularization, but to use it as an example to illustrate, from a nonparametric perspective, the necessity of regularization in training overparametrized neural networks with GD. Another goal is to demonstrate the robustness of our theory when some underlying assumptions are violated, e.g., one hidden layer, ReLU activation function and data on a sphere, etc. 
% To illustrate, we conduct simulations using the following setup. 
\begin{figure*}[!htpb]
  \centering
  \subfigure[$f^*_1$]{
    \includegraphics[scale = 0.35]{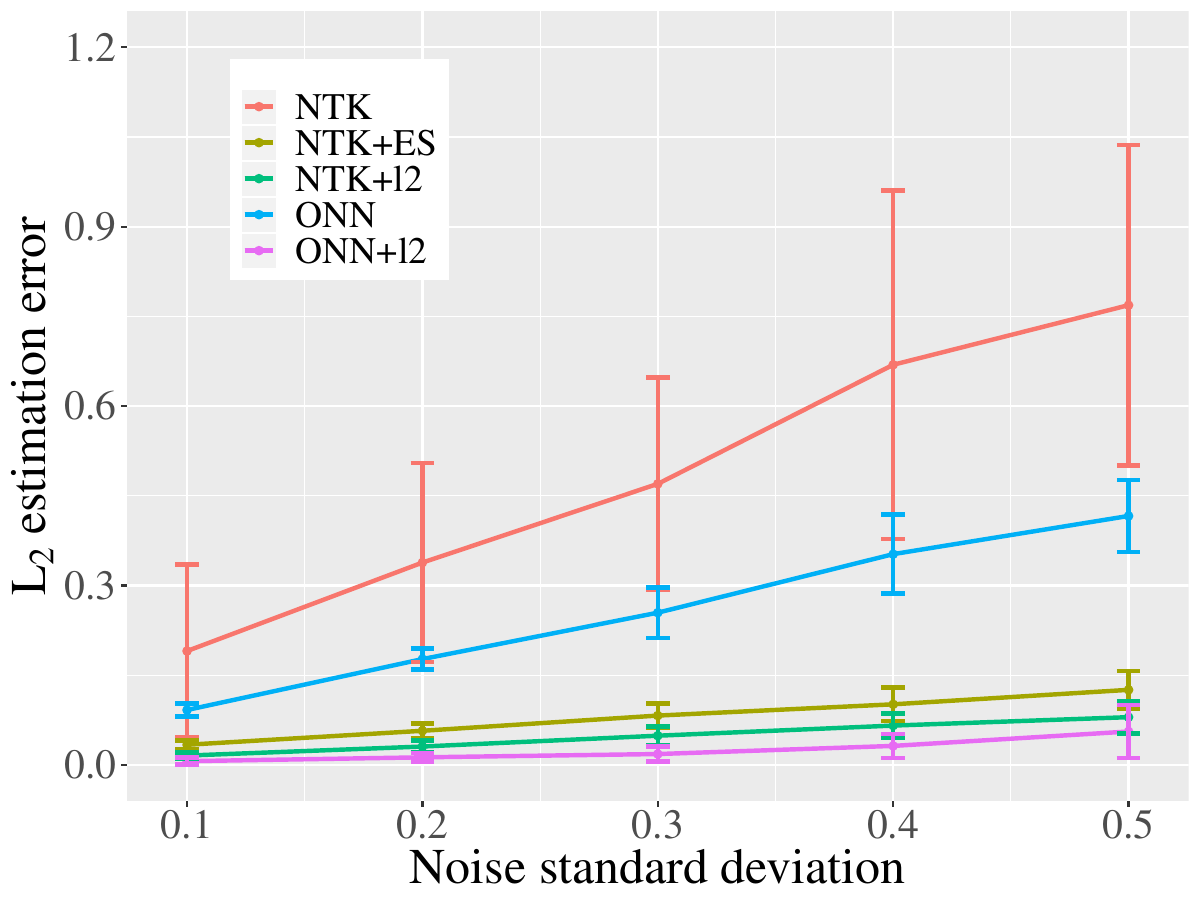}  
    \label{fig:lossnoif1}
     }\hspace{0.4cm}
    \subfigure[$f^*_2$]{
    \includegraphics[scale = 0.35]{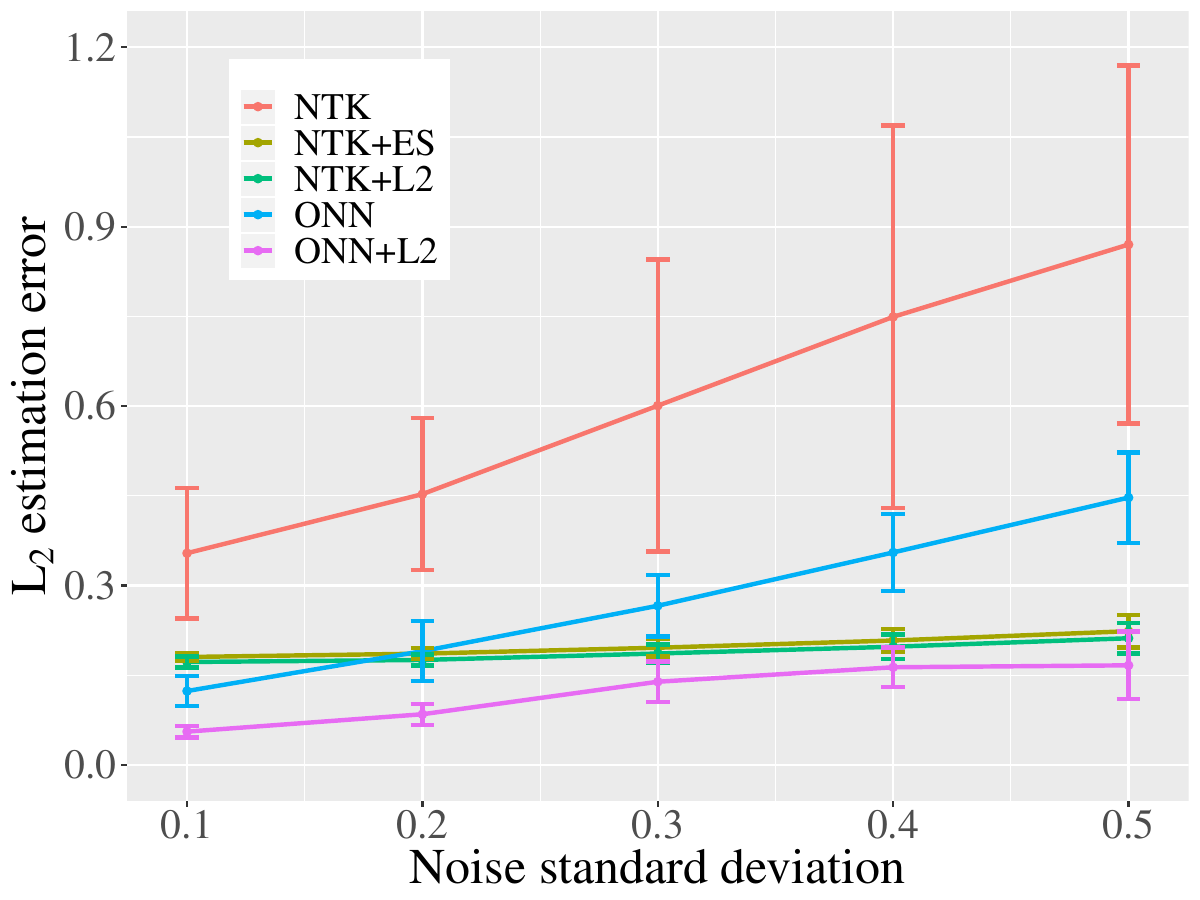}  
    \label{fig:lossnoif2}
  }
  \caption{The $L_2$ estimation errors are shown for all methods vs. $\sigma$, with their standard deviations plotted as vertical bars. 
  Similarly for both $f^*_1$ and $f^*_2$, we observe that NTK and ONN do not recover the true function well. Early stopping and $\ell_2$ regularization perform similarly for NTK, especially for $f^*_2$. ONN$+\ell_2$ performs the best in both cases.}
  \label{fig:lossnoi}
\end{figure*}
% \begin{figure}[h]
% \vspace{.3in}
% \includegraphics[scale = 0.3]{plots/f1_res.pdf}       \label{fig:lossnoif1}
% \vspace{.3in}
% \caption{Sample Figure Caption}
% \end{figure}

Specifically, we consider NTK without regularization (NTK), NTK with early stopping\footnote{As specified in Theorem \ref{thm:kernel_earlystop}, the optimal stopping time $k^*$ in (\ref{eqn:kstar}) depends on $\sigma$, which is to be estimated from data. In our simulation, we directly use the true value. The GD algorithm can found in Appendix \ref{sec:numerical}} (NTK+ES), NTK with $\ell_2$ regularization (NTK+$\ell_2$), overparametrized neural network with and without $\ell_2$ regularization, denoted as ONN and ONN$+\ell_2$, respectively. 
For ONN, we use two-hidden-layer ReLU neural networks and $m=500$ for each layer.
% The neural network used is a two-hidden-layer with ReLU activation function and $m = 500$ for each layer. 
To train the neural networks, instead of GD, we consider the more popular RMSProp optimizer \citep{hinton2012neural} with the default setting. For ONN$+\ell_2$ and NTK$+\ell_2$, the tuning parameter $\mu$ is selected by cross-validation. 
%We also implement the early stopping rule\footnote{The optimal stopping time $k^*$ in (\ref{eqn:kstar}) depends on $\sigma$, which is to be estimated from data. In our simulation, we directly use the true value.} using NTK as specified in Theorem \ref{thm:kernel_earlystop}. 

\subsection{Simulated Data}
Consider the $d=2$ case where the training data points $\bx_1,\ldots,\bx_n$ are i.i.d. sampled from unif$([-1,1]^2)$. We set $n = 100$ and let noises follow $N(0, \sigma^2)$. Two target functions are considered: $f^*_1(\bx) = 0$ and $f^*_2(\bx) = \bx^\top\bx$. 
The $L_2$ estimation error is approximated using a noiseless test dataset $\{(\bar{\bx}_i, f^*(\bar{\bx}_i))\}_{i=1}^{1000}$ where $\bar{\bx}_i$'s are new samples i.i.d. from unif$([-1,1]^2)$. 
We choose $\sigma=0.1,0.2,...,0.5$ and for each $\sigma$ value, $100$ replications are run to estimate the mean and standard deviation of the $L_2$ estimation error. Results are presented in Figure \ref{fig:lossnoi}.
More details and results can be found in Appendix \ref{sec:numerical}.

\subsection{Real Data}
To showcase our results on the $L_2$ estimation, an ideal dataset is one that can be well-fitted by neural networks so that we can treat it as noiseless and then manually inject random noises. 
Inspired by the numerical studies in \cite{hu2020simple}, we consider the MNIST dataset (digits 5 vs. 8 relabeled as $-1$ and $1$), where the test accuracy can reach over 99\% by shallow fully connected neural networks \citep{lecun1998gradient}.
Even though the dataset is for classification, we can treat the labels as continuous and learn the true function under the proposed regression setting. 
We use $\by^*$ to denote the true labels and manually add noises $\bepsilon$ to the training data, where each element of $\bepsilon$ follows $N(0,\sigma^2)$ independently. The perturbed labels are denoted by $\by = \by^* +\bepsilon$.
% Let the true labels be $\by^*$ and let $\by = \by^* +\bepsilon$ be the response variable as in (\ref{eqrelation}), where each element of $\bepsilon$ follows $N(0,\sigma^2)$. 
By gradually increase $\sigma$, we investigate how ONN and ONN$+\ell_2$ perform under the additive label noises setting.

\begin{remark}\label{rmk:noise}(Additive label noises)
To manually inject noises to classification data, many works consider replacing part of the labels by random labels \citep{zhang2016understanding,arora2019fine}. 
However, such noises are not i.i.d. and cannot be applied to the regression setting. 
% The two methods are connected, e.g., $\sigma = 1$ will result in around 16\% of the labels to flip sign but they are fundamentally different.
Similar additive label noises are also considered in \cite{hu2020simple}. 
\end{remark} 

The training dataset contains $n = 11272$ vectorized images of dimension $d = 784$. The test dataset size is 1866. 
For ONN$+\ell_2$, our training objective function is $\Phi_1$ as in (\ref{eqn:phi1}) and setting $\mu=0$ corresponds to the objective function of training ONN. On test dataset, which is \textit{not contaminated} by noises, we use the sign of the output for classification and calculate the misclassification rate as a measure of estimation performance. To be more specific, a test image $\bar{\bx}$ is classified as label 8 if $\hat{f}(\bar{\bx})\ge 0$, and label 5 if $\hat{f}(\bar{\bx})< 0$, where $\hat f$ is the neural network estimator.
The misclassification rate is the percentage of incorrect classifications on the test images.
We choose $\sigma=0,0.25,...,1.5$ and for each $\sigma$ value, $100$ replications are run to estimate the mean and standard deviation of the test misclassification rate. 
How the training root mean square error (RMSE) and test misclassification rate evolve during training when $\sigma=1$ for ONN and ONN$+\ell_2$ is also investigated.
The results are reported in Figure \ref{fig:mnist}. 
%Under $\sigma = 1$, how the sum of squares of all the weights, training and test misclassification rates evolve across iterations are shown in Figure \ref{fig:mnist_iter}. Without regularization, the sum of squares of ONN appears to diverge linearly with iteration while in ONN$+\ell_2$, it stays flat.  
% We found some NTK+ES
More details and results can be found in Appendix \ref{sec:numerical}.

\begin{figure*}[!htpb]
  \centering
  \subfigure[]{
    \includegraphics[scale = 0.35]{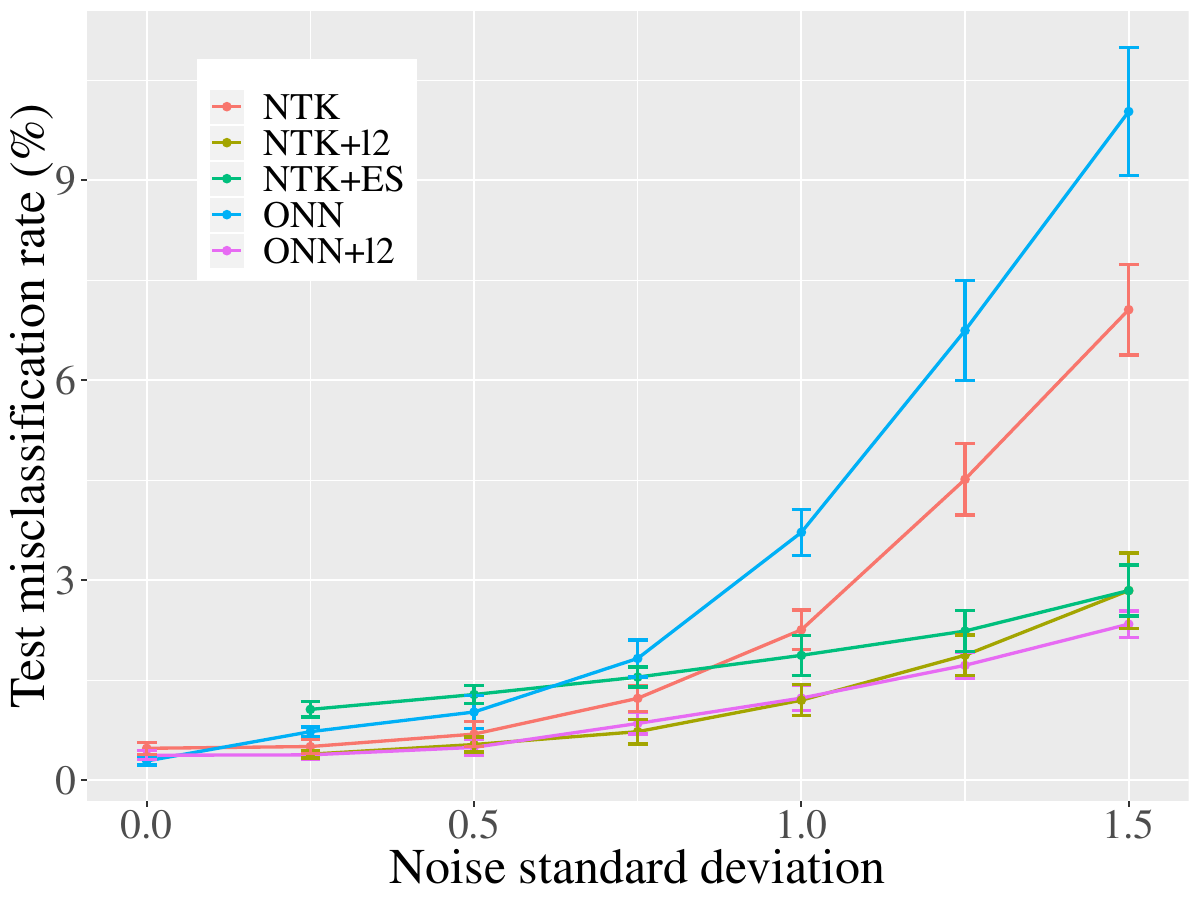}  
    \label{fig:mnist_acc}
    }\hspace{0.4cm}
    \subfigure[]{
    \includegraphics[scale = 0.35]{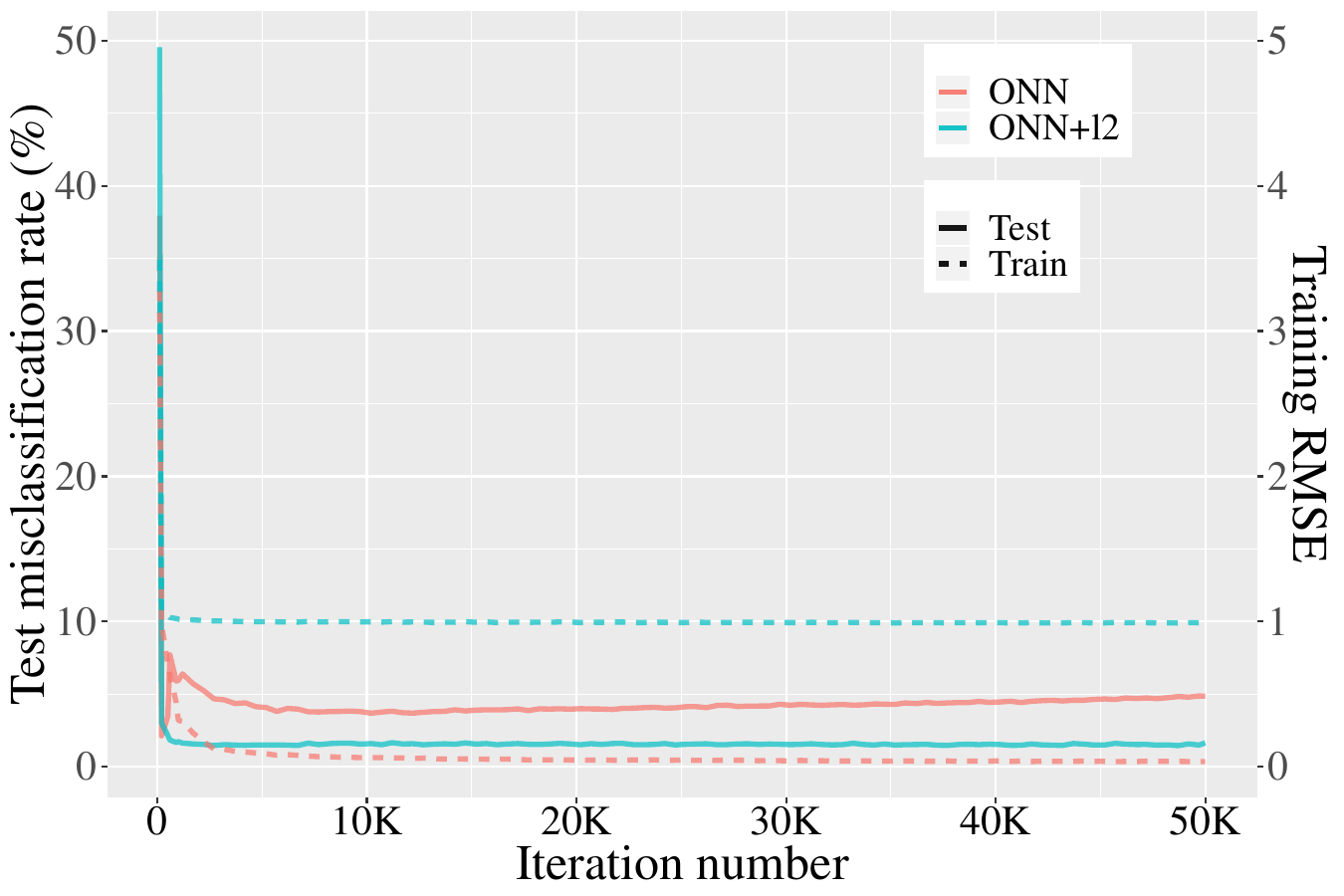} 
    \label{fig:mnist_iter}
  }
  \caption{
  %Figure (a) shows the test misclassification rates for all methods (except NTK+ES, which is deferred to Appendix \ref{sec:numerical}) vs. $\sigma$ with their standard deviations plotted as vertical bars. As $\sigma$ increases, all misclassification rates increase but NTK$+\ell_2$ and ONN$+\ell_2$ perform significantly better than NTK and ONN with smaller misclassification rate and better stability, i.e., the standard deviation is smaller. 
  Figure (a) shows the test misclassification rates for all methods vs. $\sigma$ with their standard deviations plotted as vertical bars.
    NTK+ES for $\sigma=0$ is omitted since $k^*$ is not well-defined when $\sigma=0$ and NTK+ES in this case should be the same as NTK, i.e. $k^*=\infty$. 
    As $\sigma$ increases, all misclassification rates increase but NTK$+\ell_2$ and ONN$+\ell_2$ perform significantly better than NTK and ONN with smaller misclassification rate and better stability, i.e., the standard deviation is smaller. 
    The NTK+ES is the green line and it performs the worst when $\sigma\le 0.5$ but better than NTK and ONN when $\sigma\ge 1$.
  Figure (b) shows how the training RMSE and test misclassification rate evolve across iterations for ONN and ONN$+\ell_2$ when $\sigma=1$. 
  %A close up view of the first 1K iterations is shown in the top right corner. 
  For both methods, the training RMSEs decrease fast in the first 1K iterations. 
  However, as the ONN training RMSE flattens after 10K iterations, its test misclassification rate goes up while that for ONN$+\ell_2$ remains flat even after 50K iterations, which supports our conjecture in Remark \ref{rmk:k}.
  Figure (b) also reveals the potential early stopping time for ONN around iteration 10K, which has test misclassification rate comparable to that of ONN$+\ell_2$.}
  \label{fig:mnist}
\end{figure*}
% Our simulation results 

\begin{remark}(NTK+ES)
The performance of NTK+ES is shown in Figure \ref{fig:mnist_acc}. Unlike in the simulated dataset where NTK+ES and NTK$+\ell_2$ perform almost identically, NTK+ES performs noticeably worst for the MNIST dataset, especially when $\sigma$ is small. 
One possible explanation lies in our additive label noise setting. Even though we treat the labels as continuous during training, the reported misclassification rate only depends on the sign of the label. If $\sigma$ is small, the probability of changing signs is small. This may be one of the reasons that NTK, ONN perform relatively well for small $\sigma$'s, since if the signs remain the same, it is not very harmful to overfit the labels. 
Note that NTK+$\ell_2$ and ONN+$\ell_2$ choose small $\mu$'s such that it is not very different from NTK and ONN. 
The stopping rule in NTK+ES, on the other hand, doesn't take the classification setting into consideration and tends to underestimate the stopping time when the additive label noises are small. 
Nonetheless, we don't recommend NTK+ES for handling large datasets. Firstly, the noise level $\sigma$ needs to be estimated, which brings extra instability to the algorithm. Secondly, NTK+ES is very computationally intensive, especially for the eigenvalues of the NTK matrix. %The optimal stopping time is only for GD (not for adaptive gradient-based algorithms) and is often very large. 
\end{remark}

\section{CONCLUSION AND DISCUSSION}
From a nonparametric perspective, this paper studies overparametrized neural networks trained with GD and establishes optimal $L_2$ convergence rates for trained neural network estimators under the $\ell_2$ regularization. 
On one hand, our result broadens the NTK literature by incorporating an explicit penalty term in the training objective.
On the other hand, our convergence analysis extends the statistical theory of deep neural networks by bringing algorithmic guarantees into the network estimator and offsetting the extra complexity from overparametrization through delicate GD analysis. 
Our simulation results corroborate the theoretical analysis and imply that the assumptions of our theory may be relaxed. More investigations along this direction would advance our statistical understandings of deep learning.
% Focus more on the statistical optimality and the statistics framework? 
% In parallel but complementary to the learning framework (optimization/generalization)?
% There remain various problems to be further investigated.
For example, our work can be further improved by relaxing the sphere assumption on the input data and 
% assumptions on the learning rate $\eta_1,\eta_2$ and 
the iteration number $k$ imposed in Theorems \ref{thm:withPenalty} and \ref{thm:withPenaltyl2small}. 
Additionally, although our theoretical analysis depends on the exact formula of the NTK associated with one-hidden layer ReLU neural network, it is possible to extend our theory to multi-layer DNNs as empirically shown in numerical experiments. 
In fact, it has been shown that the RKHS generated by the multi-layer NTK is equivalent to the one-hidden NTK \citep{chen2020deep}. Therefore, one possible approach for generalizing our theory is based on this equivalence.
% Similar extensions has been made in the NTK literature without explicit regularizations \citep{cao2019generalization}. 

The nonparametric perspective is potentially helpful in understanding other popular regularization techniques, e.g., batch normalization \citep{ioffe2015batch}, data augmentation \citep{dao2019kernel}, knowledge distillation \citep{hinton2015distilling}, etc. On the other hand, novel and problem-specific regularization approaches may be motivated during the convergence analysis that inspires better performance in practice. 

% Note: Our theoretical analysis depends on the exact formula of NTK of one-hidden layer NN. However, the numeric example indicates that the results should also hold for multi-layer NN. In fact, it has been shown that the RKHS generated by the multi-layer NTK is equivalent to the one-hidden NTK (cite). Therefore, one possible approach for generalizing our theory is based on this equivalence. 

% \section*{Broader Impact}
% %In order to provide a balanced perspective, authors are required to include a statement of the potential broader impact of their work, including its ethical aspects and future societal consequences. Authors should take care to discuss both positive and negative outcomes.
% This work contributes to the theoretical understanding of training overparametrized neural networks. In particular, the incorporation of regularization in the training of deep learning models is encouraged and may further empower AI applications. The negative impact on society is the potential job loss due to automation. 

\bibliography{references}

\begin{thebibliography}{57}
\providecommand{\natexlab}[1]{#1}
\providecommand{\url}[1]{\texttt{#1}}
\expandafter\ifx\csname urlstyle\endcsname\relax
  \providecommand{\doi}[1]{doi: #1}\else
  \providecommand{\doi}{doi: \begingroup \urlstyle{rm}\Url}\fi

\bibitem[Krizhevsky et~al.(2012)Krizhevsky, Sutskever, and
  Hinton]{krizhevsky2012imagenet}
Alex Krizhevsky, Ilya Sutskever, and Geoffrey~E Hinton.
\newblock Imagenet classification with deep convolutional neural networks.
\newblock In \emph{Advances in Neural Information Processing Systems}, pages
  1097--1105, 2012.

\bibitem[LeCun et~al.(2015)LeCun, Bengio, and Hinton]{lecun2015deep}
Yann LeCun, Yoshua Bengio, and Geoffrey Hinton.
\newblock Deep learning.
\newblock \emph{Nature}, 521\penalty0 (7553):\penalty0 436, 2015.

\bibitem[He et~al.(2016)He, Zhang, Ren, and Sun]{he2016deep}
Kaiming He, Xiangyu Zhang, Shaoqing Ren, and Jian Sun.
\newblock Deep residual learning for image recognition.
\newblock In \emph{Proceedings of the IEEE Conference on Computer Vision and
  Pattern Recognition}, pages 770--778, 2016.

\bibitem[Goodfellow et~al.(2014)Goodfellow, Pouget-Abadie, Mirza, Xu,
  Warde-Farley, Ozair, Courville, and Bengio]{goodfellow2014generative}
Ian Goodfellow, Jean Pouget-Abadie, Mehdi Mirza, Bing Xu, David Warde-Farley,
  Sherjil Ozair, Aaron Courville, and Yoshua Bengio.
\newblock Generative adversarial nets.
\newblock In \emph{Advances in Neural Information Processing Systems}, pages
  2672--2680, 2014.

\bibitem[Arjovsky et~al.(2017)Arjovsky, Chintala, and
  Bottou]{arjovsky2017wasserstein}
Martin Arjovsky, Soumith Chintala, and L{\'e}on Bottou.
\newblock Wasserstein {GAN}.
\newblock \emph{arXiv preprint arXiv:1701.07875}, 2017.

\bibitem[Jacot et~al.(2018)Jacot, Gabriel, and Hongler]{jacot2018neural}
Arthur Jacot, Franck Gabriel, and Cl{\'e}ment Hongler.
\newblock Neural tangent kernel: Convergence and generalization in neural
  networks.
\newblock In \emph{Advances in Neural Information Processing Systems}, pages
  8571--8580, 2018.

\bibitem[Du et~al.(2018)Du, Zhai, Poczos, and Singh]{du2018gradient}
Simon~S Du, Xiyu Zhai, Barnabas Poczos, and Aarti Singh.
\newblock Gradient descent provably optimizes over-parameterized neural
  networks.
\newblock \emph{arXiv preprint arXiv:1810.02054}, 2018.

\bibitem[Li and Liang(2018)]{li2018learning}
Yuanzhi Li and Yingyu Liang.
\newblock Learning overparameterized neural networks via stochastic gradient
  descent on structured data.
\newblock In \emph{Advances in Neural Information Processing Systems}, pages
  8157--8166, 2018.

\bibitem[Arora et~al.(2019)Arora, Du, Hu, Li, and Wang]{arora2019fine}
Sanjeev Arora, Simon~S Du, Wei Hu, Zhiyuan Li, and Ruosong Wang.
\newblock Fine-grained analysis of optimization and generalization for
  overparameterized two-layer neural networks.
\newblock \emph{arXiv preprint arXiv:1901.08584}, 2019.

\bibitem[Zou and Gu(2019)]{zou2019improved}
Difan Zou and Quanquan Gu.
\newblock An improved analysis of training over-parameterized deep neural
  networks.
\newblock In \emph{Advances in Neural Information Processing Systems}, pages
  2053--2062, 2019.

\bibitem[Ji and Telgarsky(2019{\natexlab{a}})]{ji2019implicit}
Ziwei Ji and Matus Telgarsky.
\newblock The implicit bias of gradient descent on nonseparable data.
\newblock In \emph{Conference on Learning Theory}, pages 1772--1798,
  2019{\natexlab{a}}.

\bibitem[Ji and Telgarsky(2019{\natexlab{b}})]{ji2019polylogarithmic}
Ziwei Ji and Matus Telgarsky.
\newblock Polylogarithmic width suffices for gradient descent to achieve
  arbitrarily small test error with shallow {ReLU} networks.
\newblock \emph{arXiv preprint arXiv:1909.12292}, 2019{\natexlab{b}}.

\bibitem[Lyu and Li(2019)]{lyu2019gradient}
Kaifeng Lyu and Jian Li.
\newblock Gradient descent maximizes the margin of homogeneous neural networks.
\newblock \emph{arXiv preprint arXiv:1906.05890}, 2019.

\bibitem[Stone(1982)]{stone1982optimal}
Charles~J Stone.
\newblock Optimal global rates of convergence for nonparametric regression.
\newblock \emph{The Annals of Statistics}, pages 1040--1053, 1982.

\bibitem[Siegel(1957)]{siegel1957nonparametric}
Sidney Siegel.
\newblock Nonparametric statistics.
\newblock \emph{The American Statistician}, 11\penalty0 (3):\penalty0 13--19,
  1957.

\bibitem[Cao and Gu(2019)]{cao2019generalization}
Yuan Cao and Quanquan Gu.
\newblock Generalization error bounds of gradient descent for learning
  overparameterized deep {ReLU} networks.
\newblock \emph{arXiv preprint arXiv:1902.01384}, 2019.

\bibitem[Wei et~al.(2019)Wei, Lee, Liu, and Ma]{wei2019regularization}
Colin Wei, Jason~D Lee, Qiang Liu, and Tengyu Ma.
\newblock Regularization matters: Generalization and optimization of neural
  nets vs their induced kernel.
\newblock In \emph{Advances in Neural Information Processing Systems}, pages
  9709--9721, 2019.

\bibitem[Hu et~al.(2020)Hu, Li, and Yu]{hu2020simple}
W~Hu, Z~Li, and D~Yu.
\newblock Simple and effective regularization methods for training on noisily
  labeled data with generalization guarantee.
\newblock In \emph{International Conference on Learning Representations}, 2020.

\bibitem[Geifman et~al.(2020)Geifman, Yadav, Kasten, Galun, Jacobs, and
  Basri]{geifman2020similarity}
Amnon Geifman, Abhay Yadav, Yoni Kasten, Meirav Galun, David Jacobs, and Ronen
  Basri.
\newblock On the similarity between the laplace and neural tangent kernels.
\newblock \emph{NeurIPS 2020}, 2020.

\bibitem[Nitanda et~al.(2019)Nitanda, Chinot, and Suzuki]{nitanda2019gradient}
Atsushi Nitanda, Geoffrey Chinot, and Taiji Suzuki.
\newblock Gradient descent can learn less over-parameterized two-layer neural
  networks on classification problems.
\newblock \emph{arXiv preprint arXiv:1905.09870}, 2019.

\bibitem[Nitanda and Suzuki(2020)]{nitanda2020optimal}
Atsushi Nitanda and Taiji Suzuki.
\newblock Optimal rates for averaged stochastic gradient descent under neural
  tangent kernel regime.
\newblock \emph{arXiv preprint arXiv:2006.12297}, 2020.

\bibitem[Yarotsky(2017)]{yarotsky2017error}
Dmitry Yarotsky.
\newblock Error bounds for approximations with deep {ReLU} networks.
\newblock \emph{Neural Networks}, 94:\penalty0 103--114, 2017.

\bibitem[Schmidt-Hieber(2017)]{schmidt2017nonparametric}
Johannes Schmidt-Hieber.
\newblock Nonparametric regression using deep neural networks with {ReLU}
  activation function.
\newblock \emph{arXiv preprint arXiv:1708.06633}, 2017.

\bibitem[Bauer et~al.(2019)Bauer, Kohler, et~al.]{bauer2019deep}
Benedikt Bauer, Michael Kohler, et~al.
\newblock On deep learning as a remedy for the curse of dimensionality in
  nonparametric regression.
\newblock \emph{The Annals of Statistics}, 47\penalty0 (4):\penalty0
  2261--2285, 2019.

\bibitem[Liu et~al.(2019)Liu, Boukai, and Shang]{liu2019optimal}
Ruiqi Liu, Ben Boukai, and Zuofeng Shang.
\newblock Optimal nonparametric inference via deep neural network.
\newblock \emph{arXiv preprint arXiv:1902.01687}, 2019.

\bibitem[Imaizumi and Fukumizu(2018)]{imaizumi2018deep}
Masaaki Imaizumi and Kenji Fukumizu.
\newblock Deep neural networks learn non-smooth functions effectively.
\newblock \emph{arXiv preprint arXiv:1802.04474}, 2018.

\bibitem[Bietti and Mairal(2019)]{bietti2019inductive}
Alberto Bietti and Julien Mairal.
\newblock On the inductive bias of neural tangent kernels.
\newblock In \emph{Advances in Neural Information Processing Systems}, pages
  12873--12884, 2019.

\bibitem[Yuan et~al.(2016)Yuan, Zhou, et~al.]{yuan2016minimax}
Ming Yuan, Ding-Xuan Zhou, et~al.
\newblock Minimax optimal rates of estimation in high dimensional additive
  models.
\newblock \emph{The Annals of Statistics}, 44\penalty0 (6):\penalty0
  2564--2593, 2016.

\bibitem[Raskutti et~al.(2014)Raskutti, Wainwright, and Yu]{raskutti2014early}
Garvesh Raskutti, Martin~J Wainwright, and Bin Yu.
\newblock Early stopping and non-parametric regression: an optimal
  data-dependent stopping rule.
\newblock \emph{The Journal of Machine Learning Research}, 15\penalty0
  (1):\penalty0 335--366, 2014.

\bibitem[Kohler and Krzyzak(2019)]{kohler2019over}
Michael Kohler and Adam Krzyzak.
\newblock Over-parametrized deep neural networks do not generalize well.
\newblock \emph{arXiv preprint arXiv:1912.03925}, 2019.

\bibitem[Krogh and Hertz(1992)]{krogh1992simple}
Anders Krogh and John~A Hertz.
\newblock A simple weight decay can improve generalization.
\newblock In \emph{Advances in Neural Information Processing Systems}, pages
  950--957, 1992.

\bibitem[Ioffe and Szegedy(2015)]{ioffe2015batch}
Sergey Ioffe and Christian Szegedy.
\newblock Batch normalization: Accelerating deep network training by reducing
  internal covariate shift.
\newblock \emph{arXiv preprint arXiv:1502.03167}, 2015.

\bibitem[Srivastava et~al.(2014)Srivastava, Hinton, Krizhevsky, Sutskever, and
  Salakhutdinov]{srivastava2014dropout}
Nitish Srivastava, Geoffrey Hinton, Alex Krizhevsky, Ilya Sutskever, and Ruslan
  Salakhutdinov.
\newblock Dropout: a simple way to prevent neural networks from overfitting.
\newblock \emph{The Journal of Machine Learning Research}, 15\penalty0
  (1):\penalty0 1929--1958, 2014.

\bibitem[Bilgic et~al.(2014)Bilgic, Chatnuntawech, Fan, Setsompop, Cauley,
  Wald, and Adalsteinsson]{bilgic2014fast}
Berkin Bilgic, Itthi Chatnuntawech, Audrey~P Fan, Kawin Setsompop, Stephen~F
  Cauley, Lawrence~L Wald, and Elfar Adalsteinsson.
\newblock Fast image reconstruction with l2-regularization.
\newblock \emph{Journal of magnetic resonance imaging}, 40\penalty0
  (1):\penalty0 181--191, 2014.

\bibitem[Van~Laarhoven(2017)]{van2017l2}
Twan Van~Laarhoven.
\newblock L2 regularization versus batch and weight normalization.
\newblock \emph{arXiv preprint arXiv:1706.05350}, 2017.

\bibitem[Phaisangittisagul(2016)]{phaisangittisagul2016analysis}
Ekachai Phaisangittisagul.
\newblock An analysis of the regularization between l2 and dropout in single
  hidden layer neural network.
\newblock In \emph{2016 7th International Conference on Intelligent Systems,
  Modelling and Simulation (ISMS)}, pages 174--179. IEEE, 2016.

\bibitem[Loshchilov and Hutter(2017)]{loshchilov2017decoupled}
Ilya Loshchilov and Frank Hutter.
\newblock Decoupled weight decay regularization.
\newblock \emph{arXiv preprint arXiv:1711.05101}, 2017.

\bibitem[Caruana et~al.(2001)Caruana, Lawrence, and
  Giles]{caruana2001overfitting}
Rich Caruana, Steve Lawrence, and C~Lee Giles.
\newblock Overfitting in neural nets: Backpropagation, conjugate gradient, and
  early stopping.
\newblock In \emph{Advances in Neural Information Processing Systems}, pages
  402--408, 2001.

\bibitem[Prechelt(1998)]{prechelt1998early}
Lutz Prechelt.
\newblock Early stopping-but when?
\newblock In \emph{Neural Networks: Tricks of the Trade}, pages 55--69.
  Springer, 1998.

\bibitem[Zhang et~al.(2016)Zhang, Bengio, Hardt, Recht, and
  Vinyals]{zhang2016understanding}
Chiyuan Zhang, Samy Bengio, Moritz Hardt, Benjamin Recht, and Oriol Vinyals.
\newblock Understanding deep learning requires rethinking generalization.
\newblock \emph{arXiv preprint arXiv:1611.03530}, 2016.

\bibitem[Lewkowycz and Gur-Ari(2020)]{lewkowycz2020training}
Aitor Lewkowycz and Guy Gur-Ari.
\newblock On the training dynamics of deep networks with $ l\_2 $
  regularization.
\newblock \emph{arXiv preprint arXiv:2006.08643}, 2020.

\bibitem[Hinton et~al.()Hinton, Srivastava, and Swersky]{hinton2012neural}
Geoffrey Hinton, Nitish Srivastava, and Kevin Swersky.
\newblock Neural networks for machine learning lecture 6a overview of
  mini-batch gradient descent.

\bibitem[LeCun et~al.(1998)LeCun, Bottou, Bengio, and
  Haffner]{lecun1998gradient}
Yann LeCun, L{\'e}on Bottou, Yoshua Bengio, and Patrick Haffner.
\newblock Gradient-based learning applied to document recognition.
\newblock \emph{Proceedings of the IEEE}, 86\penalty0 (11):\penalty0
  2278--2324, 1998.

\bibitem[Chen and Xu(2020)]{chen2020deep}
Lin Chen and Sheng Xu.
\newblock Deep neural tangent kernel and laplace kernel have the same rkhs.
\newblock \emph{arXiv preprint arXiv:2009.10683}, 2020.

\bibitem[Dao et~al.(2019)Dao, Gu, Ratner, Smith, De~Sa, and
  R{\'e}]{dao2019kernel}
Tri Dao, Albert Gu, Alexander~J Ratner, Virginia Smith, Christopher De~Sa, and
  Christopher R{\'e}.
\newblock A kernel theory of modern data augmentation.
\newblock \emph{Proceedings of machine learning research}, 97:\penalty0 1528,
  2019.

\bibitem[Hinton et~al.(2015)Hinton, Vinyals, and Dean]{hinton2015distilling}
Geoffrey Hinton, Oriol Vinyals, and Jeff Dean.
\newblock Distilling the knowledge in a neural network.
\newblock \emph{arXiv preprint arXiv:1503.02531}, 2015.

\bibitem[Cao et~al.(2019)Cao, Fang, Wu, Zhou, and Gu]{cao2019towards}
Yuan Cao, Zhiying Fang, Yue Wu, Ding-Xuan Zhou, and Quanquan Gu.
\newblock Towards understanding the spectral bias of deep learning.
\newblock \emph{arXiv preprint arXiv:1912.01198}, 2019.

\bibitem[van~de Geer(2000)]{sara2000}
Sara van~de Geer.
\newblock \emph{Empirical Processes in M-estimation}.
\newblock Cambridge University Press, 2000.

\bibitem[van~de Geer(2014)]{van2014uniform}
Sara van~de Geer.
\newblock On the uniform convergence of empirical norms and inner products,
  with application to causal inference.
\newblock \emph{Electronic Journal of Statistics}, 8\penalty0 (1):\penalty0
  543--574, 2014.

\bibitem[Kimeldorf and Wahba(1971)]{kimeldorf1971some}
George Kimeldorf and Grace Wahba.
\newblock Some results on tchebycheffian spline functions.
\newblock \emph{Journal of mathematical analysis and applications}, 33\penalty0
  (1):\penalty0 82--95, 1971.

\bibitem[Varga(2010)]{varga}
Richard~S Varga.
\newblock \emph{Gershgorin and His Circles}, volume~36.
\newblock Springer Science \& Business Media, 2010.

\bibitem[Bach(2017)]{bach2017breaking}
Francis Bach.
\newblock Breaking the curse of dimensionality with convex neural networks.
\newblock \emph{The Journal of Machine Learning Research}, 18\penalty0
  (1):\penalty0 629--681, 2017.

\bibitem[Atkinson and Han(2012)]{atkinson2012spherical}
Kendall Atkinson and Weimin Han.
\newblock \emph{Spherical Harmonics and Approximations on the Unit Sphere: An
  Introduction}, volume 2044.
\newblock Springer Science \& Business Media, 2012.

\bibitem[Costas and Christopher(2014)]{costas2014spherical}
Efthimiou Costas and Frye Christopher.
\newblock \emph{Spherical Harmonics in $p$ Dimensions}.
\newblock World Scientific, 2014.

\bibitem[Brauchart and Dick(2013)]{brauchart2013characterization}
Johann~S Brauchart and Josef Dick.
\newblock A characterization of {S}obolev spaces on the sphere and an extension
  of {S}tolarsky’s invariance principle to arbitrary smoothness.
\newblock \emph{Constructive Approximation}, 38\penalty0 (3):\penalty0
  397--445, 2013.

\bibitem[Wang et~al.(2014)Wang, Wang, and Wang]{wang2014entropy}
He~Ping Wang, Kai Wang, and Jing Wang.
\newblock Entropy numbers of {B}esov classes of generalized smoothness on the
  sphere.
\newblock \emph{Acta Mathematica Sinica, English Series}, 30\penalty0
  (1):\penalty0 51--60, 2014.

\bibitem[Glorot and Bengio(2010)]{glorot2010understanding}
Xavier Glorot and Yoshua Bengio.
\newblock Understanding the difficulty of training deep feedforward neural
  networks.
\newblock In \emph{Proceedings of the thirteenth international conference on
  artificial intelligence and statistics}, pages 249--256, 2010.

\end{thebibliography}

\appendix
\newpage
\onecolumn
\aistatstitle{
Supplementary Materials}
\section{More notation}
We introduce some additional notation to be used in the Appendix.
% By the interpolation inequality, for any function $g\in \mathcal{N}$, we have
% \begin{align*}
%     \|g\|_{L_\infty}\lesssim \|g\|_{\cN},
% \end{align*}
Denote $\by^*=(f^*(x_1),\cdots,f^*(x_n))^\top$ as the the vector of underlying function's functional values at sample points.
Let $\mathbb{I}_{r}(\bx) = \mathbb{I}\{\bw_r^\top\bx\geq 0\}$ and
\begin{align}\label{z(x)}
\bz(\bx)=\frac{1}{\sqrt{m}}
    \begin{pmatrix}
    a_1\mathbb{I}_{1}(\bx)\bx\\
    \vdots\\
    a_m\mathbb{I}_{m}(\bx)\bx
    \end{pmatrix}\in\RR^{md\times 1}.
\end{align}
Thus, $\bZ(k) = (\bz(\bx_1),...,\bz(\bx_n))|_{\bW=\bW(k)}$. When the context is clear, we omit the dimension and write $\bI_d$ as $\bI$.

\section{Proof of Lemma \ref{lemeigendecay}}
\label{sec:pflemeigendecay}
We will use the following lemma, which states the Mercer decomposition of $h$ as in (\ref{ntkh}). 
\begin{lemma}[Mercer decomposition of NTK $h$]\label{lemmercerh}
For any $\bs,\bt \in \mathbb{S}^{d-1}$, we have the following decomposition of the NTK,
\begin{align*}
    h(\bs,\bt) = \sum_{k=0}^\infty \mu_k \sum_{j=1}^{N(d,k)}Y_{k,j}(\bs)Y_{k,j}(\bt),
\end{align*}
where $Y_{k,j}$, $j=1,...,N(d,k)$ are spherical harmonic polynomials of degree $k$, and the non-negative eigenvalues $\mu_k$ satisfy $\mu_k \asymp k^{-d}$, and $\mu_k = 0$ if $k=2j+1$ for $k\geq 2$.
\end{lemma}
The proof of Lemma \ref{lemmercerh} is similar to the proof of Proposition 5 in \cite{bietti2019inductive}. The difference is that the Proposition 5 in \cite{bietti2019inductive} considers the kernel function 
\begin{align*}
    h_1(\bs,\bt) = 4h(\bs,\bt) + \frac{\sqrt{1-(\bs^\top\bt)^2}}{\pi},
\end{align*}
and we only need to consider the kernel function $h(\bs,\bt)$.
A generalization of Proposition 5 in \cite{bietti2019inductive} can be found in Theorem 3.5 of \cite{cao2019towards}.

Note that in the proof of Lemma \ref{lemmercerh},
\begin{align*}
    N(d,j) = \frac{2j+d-2}{j}\left(\begin{array}{c}
     j+d-3\\
     d-2 
\end{array}\right) = \frac{\Gamma(j+d-2)}{\Gamma(d-1)\Gamma(j)},
\end{align*}
where $\Gamma$ is the Gamma function. By the Stirling approximation, we have $\Gamma(x) \approx \sqrt{2\pi}x^{x-1/2}e^{-x}$. Therefore, we have the number $N(d,j)$ is equivalent to $j^{d-2}$. Thus, by Lemma \ref{lemmercerh}, the $j$-th eigenvalue $\lambda_j$ can be denoted by
\begin{align*}
    \lambda_j = \mu_l, \mbox{ for } \sum_{i=1}^{l-1}N(d,2i) \leq j < \sum_{i=1}^{l}N(d,2i),
\end{align*}
which can be approximated by
$\lambda_j \asymp \mu_l, \mbox{ for } (2l-2)^{d-1} \leq j < (2l)^{d-1}.$
By Lemma \ref{lemmercerh}, we have $\mu_l\asymp l^{-d}$, which implies $\lambda_j \asymp j^{-\frac{d}{d-1}}$.

\section{Proof of Theorem \ref{thmkrr}}
\label{sec:kernel_rate}
Let $\mathcal{G}$ be a metric space equipped with a metric $d_g$. The $\delta$-covering number of the metric space $(\mathcal{G},d_g)$, denoted by $N(\delta,\mathcal{G},d_g)$, is the minimum integer $N$ so that there exist $N$ distinct balls in $(\mathcal{G},d_g)$ with radius $\delta$, and the union of these balls covers $\mathcal{G}$. Let $H(\delta,\mathcal{G},d_g) = \log N(\delta,\mathcal{G},d_g)$ be the entropy of the metric space $(\mathcal{G},d_g)$. We first present an upper bound on the entropy of the metric space $(\mathcal{N}, \|\cdot\|_\infty)$, where the proof can be found in 
Appendix \ref{sec:kernel_decay}.
\begin{lemma}\label{lementbound}
Let $\mathcal{N}$ be the reproducing kernel Hilbert space generated by the NTK $h$ defined in \eqref{ntkh}, equipped with norm $\|\cdot\|_{\mathcal{N}}$. The entropy $H(\delta,\mathcal{N}(1),\|\cdot\|_{\infty})$ can be bounded by
\begin{align}
\label{eqn:entropy}
    H(\delta,\mathcal{N}(1),\|\cdot\|_{\infty}) \leq A_0 \delta^{-\frac{2(d-1)}{d}},
\end{align}
where $\mathcal{N}(1) = \{f:f\in \mathcal{N}, \|f\|_{\mathcal{N}}\leq 1\}$, and $A_0>0$ is a constant not depending on $\delta$. 
\end{lemma}
For the regression problem, consider a general penalized least-square estimator 
\[\hat{f} := \argmin_{f\in\cN}\rbr{\frac{1}{n}\sum_{i=1}^n(y_i-f(\bx_i))^2+\lambda_n^2I^v(f)},\]
where $\lambda_n>0$ is the smoothing parameter and $I:\cN\to [0,\infty)$ is a pseudo-norm measuring the complexity. We use the RKHS norm $\|f\|_{\cN}$ in our case.
Let $\|\cdot\|_n$ denote the empirical norm. The following lemma establishes the rate of convergence for the estimator $\hat{f}$. 
\begin{lemma}[Lemma 10.2 in \cite{sara2000}] 
\label{lemma:10.2}
Assume Gaussian noises and entropy bound $H(\delta,\mathcal{N}(1),\|\cdot\|_n) \leq A \delta^{-\alpha}$ for some constants $A>0$ and $0<\alpha<2$. If $v\ge\frac{2\alpha}{2+\alpha}$, $I(f^*)>0$ and 
\[\lambda_n^{-1}=O_{\PP}\rbr{n^{1/(2+\alpha)}}I^{(2v-2\alpha+v\alpha)/2(2+\alpha)}(f^*).\]
Then we have \[\|\hat{f} - f^*\|_n = O_{\PP}(\lambda_n)I^{v/2}(f^*)\]
and $I(\hat{f})=O_{\PP}(1)I(f^*)$.
\end{lemma}

To bound the difference between empirical norm and $L_2$ norm, we utilize the following lemma. For a class of functions ${\cal F}$, define for $z > 0$
\begin{equation*}\label{calJ.definition}
{J}_\infty( z , {\cal F} ) := C_0  \inf_{\delta >0 } \biggl [z  \int_{\delta  / 4}^1 \sqrt {{\cal H}_{ {\infty} } ( uz/2, {\cal F} )} du +
\sqrt {n} \delta z  \biggr ]  . 
\end{equation*} 
\begin{lemma}[Theorem 2.2 in \cite{van2014uniform}]
\label{lemma:geer} 
Let
\[ R:=  \sup_{f \in {\cal F} } \| f   \|_2  , \ K:= \sup_{f \in {\cal F}} \| f \|_{\infty} 
%\ \hat{R}:=  \sup_{f \in {\cal F} } \| f   \|_n. 
\]
% Then
% $$ \EE \biggl (\sup_{f \in {\cal F} } \biggl |  \| f \|_n^2 - \| f \|^2 \biggr | \biggr ) \le 2 J_{\infty} (K, {\cal F}) R / \sqrt n+
% 4 J_{\infty}^2 (K, {\cal F })/n . $$
Then, for all $t >0$, with probability at least $1- \exp[-t]$, 
\[\sup_{f \in {\cal F}} \biggl | \| f \|_n^2 - \| f \|_2^2 \biggr | / C_1 \le
\frac{2 R J_{\infty} (K, {\cal F} ) +RK \sqrt t }{  \sqrt n} + \frac{ 4 J_{\infty}^2 (K, {\cal F})  + K^2 t}{n} \]
where $C_1>0$ is some constant not depending on $n$. 
%Also $\sqrt  {\EE \hat R^2 } \le R + 2 J_{\infty} (K, {\cal F} )/\sqrt n .$

\end{lemma}

\begin{proof}[Proof of Theorem \ref{thmkrr}]
Consider our estimator $\hat{f}$ as in (\ref{eqsolukrr}), in which case, $v= 2$ and $I(f)$ is the RKHS norm of $f$. Since $\|f\|_n\le \|f\|_\infty$, Lemma \ref{lementbound} indicates that $ \alpha = 2(d-1)/d < 2$.
By choosing $\lambda_n \asymp n^{-d/(4d-2)}$, which corresponds to $\mu\asymp n^{(d-1)/(2d-1)}$ in (\ref{regpro}), Lemma \ref{lemma:10.2} yields that 
\[\|\hat{f} - f^*\|^2_n=O_{\PP}(n^{-d/(2d-1)})\quad  \mbox{and}\quad  \|\hat{f}\|^2_{\cN}=O_{\PP}(1).\]
Now we use Lemma \ref{lemma:geer} to obtain a bound on $\|\hat{f} - f^*\|_2$.
First consider $\{f-f^*:f\in\cN(1)\}$, where $\cN(1) = \{f\in \cN, \|f\|_{\cN}\le 1\}$. Thus, we have $K,R = O(1)$. By the entropy bound in Lemma \ref{lementbound}, we have $J_\infty(z,\cN(1))\le 2C_0z^{1/d}$.
Therefore, Lemma $\ref{lemma:geer}$ yields 
\[\sup_{f \in {\cN(1)}} \biggl | \| f-f^* \|_n^2 - \| f-f^* \|_2^2 \biggr | =O_{\PP}\rbr{\sqrt{\frac{{1}}{n}}}.\]
Combined with $\|\hat{f}-f^*\|_n^2 = O_{\PP}(n^{-d/(2d-1)})$, we can conclude that for any $t>0$ large enough, $\|\hat{f}-f^*\|_2^2 = O(\sqrt{t/n})$ with probability at least $1-\exp(-t)$. 
Utilizing Lemma \ref{lemma:geer} again with
$R = O(\sqrt{t/n})$ we have for some $C>0$,
\begin{align*}
    \PP\rbr{\sup_{f \in {\cG(R)}} \biggl | \| f-f^* \|_n^2 - \| f-f^* \|_2^2 \biggr | \le  {{\frac{{Ct}}{n}}}}\geq 1-e^{-t},
\end{align*}
where $\cG(R):=\{f\in\cN(1): \|f-f^*\|_2\le R\}$. Notice that $\hat{f}\in\cG(R)$ with probability at least $1-\exp(-t)$. 
Therefore, $\|\hat{f} - f^*\|_2^2=O(n^{-d/(2d-1)} + {t}/n)$ with probability at least $1-2\exp(-t)$.
\end{proof}

\section{Proofs of main theorems in Section \ref{secwithoutp}}
For brevity, let $\hat{f}_k = f_{\bW(k),\ba}$.  For two positive semidefinite matrices $\bA$ and $\bB$, we write $\bA\ge \bB$ to denote that $\bA-\bB$ is positive semidefinite and $\bA>\bB$ to denote that $\bA - \bB$ is positive definite. This partial order of positive semidefinite matrices is also known as Loewner order. 
We focus on the $L_2$ loss of our estimator $\hat{f}_k$ after $k$ GD updates. 
Let $\tilde{f}$ denote the kernel regression solution with kernel $h(\cdot, \cdot)$ that interpolates all $\{(\bx_i,f^*(\bx_i))\}_{i=1}^n$, i.e., 
\begin{align}
\label{eqn:gx}
   {g}(\bx) = h(\bx,\bX)(\bH^\infty)^{-1}\by^*.
\end{align}

We first provide some lemmas used in this section. The proofs of lemmas are presented in Appendix \ref{sec:kernel_decay}. Lemma \ref{lemwithoutp1} states some basic inequalities that are also used in the proof of Theorem \ref{thm:withPenalty}. Lemma \ref{lem:intpRKHS} provides the convergence rate of interpolant using NTK. Lemmas \ref{lem:wtoint} can be found in \cite{arora2019fine}. Lemma \ref{lem:uky} is implied by the proof in \cite{arora2019fine}. Lemma \ref{lem:bounds4} provides some bounds on the related quantities used in the proofs of Theorems \ref{thm:gd} and \ref{thm:withPenaltyl2small}. Lemma \ref{lem:Loewner} provide some properties of Loewner order.
\begin{lemma}\label{lemwithoutp1}
Let $\mu$ be as in Theorem \ref{thmkrr}. Then we have
\begin{align*}
    h(\bs,\bs) - h(\bs,\bX)(\bH^\infty)^{-1}h(\bX,\bs) & \geq 0,\\
    \int_{\bx\in\Omega}h(\bx,\bX)(\bH^\infty + \mu \bI)^{-2}h(\bX,\bx) d\bx = & O_{\mathbb{P}}(n^{-\frac{d}{2d-1}}),\\
    \int_{\bx\in\Omega}h(\bx,\bx)  - h(\bx,\bX)(\bH^\infty)^{-1}h(\bX,\bx) d\bx = & O_{\mathbb{P}}(n^{-\frac{1}{2d-1}}),
\end{align*}
where $h(\bx,\bX) = (h(\bx,\bx_1),...,h(\bx,\bx_n))$ and $h(\bX,\bx) = h(\bx,\bX)^\top$.
\end{lemma}

\begin{lemma}\label{lem:intpRKHS}
Assume the true function $f^*\in\cN$ with finite RKHS norm, then $g(\bx)$ defined (\ref{eqn:gx}) satisfies
\begin{align*}
\| g - f^*\|_{2} = O_{\PP}\rbr{n^{-1/2}}.
\end{align*}
\end{lemma}
% \begin{proof}
% Given that $g$ and $f^*$ have the same value at all $\bx_i$'s, the empirical norm $\|g-f^*\|_n=0$.
% Notice that both $g$ and $f^*$ are in the RKHS generated by the NTK $h$, denoted by $\cN$. Utilizing Lemma \ref{lementbound} and \ref{lemma:geer} similarly as in the proof of Theorem \ref{thmkrr}, we have $R, K = O(1)$ and $J_\infty(z,\cN)\lesssim z^{1/d}$, which leads to 
% $$\sup_{h \in {\cG(R)}} \biggl | \| h \|_n^2 - \| h \|_2^2 \biggr | =O_{\PP}\rbr{\sqrt{\frac{{1}}{n}}},$$
% where $\cG(R):=\{g\in\cN(1): \|g-g^*\|_2\le R\}$.
% Therefore, we can conclude that $\|g-f^*\|_2=O_{\PP}(n^{-1/2})$.
% \end{proof}

\begin{lemma}[Lemma C.1 in \cite{arora2019fine}]
\label{lem:wtoint}
If $\lambda_0=\lambda_{min}(\bH^{\infty})>0$, $m=\Omega\rbr{\frac{n^6}{\lambda_0^4\tau^2\delta^3}}$ and $\eta=O\rbr{\frac{\lambda_0}{n^2}}$, with probability at least $1-\delta$ over the random initialization, we have
\begin{align*}
    \|\bw_r(k)-\bw_r(0) \|_2 \le R_0,~~\forall ~r\in [m],  \forall \ k\ge 0,
\end{align*}
where $R_0=\frac{4\sqrt{n}\|\by-\bu(0)\|_2}{\sqrt{m}\lambda_0}$.
\end{lemma}

\begin{lemma}[\cite{arora2019fine}]\label{lem:uky}
Denote $u_i(k) = f_{\bW(k),\ba}(\bx_i)$ to be the network's prediction on the $i$-th input and let $\bu(k) = (u_1(k),...,u_n(k))^\top\in \RR^n$ denote all $n$ predictions on the points $\bx_1,...,\bx_n$ at iteration $k$. We have
\begin{align*}
    \bu(k)-\by=(\bI-\eta\bH^{\infty})^k(\bu(0)-\by)+\be(k)
\end{align*}
where 
\begin{align*}
 \|\be(k)\|_2 = O\rbr{k\rbr{1-\frac{\eta\lambda_0}{4}}^{k-1}\frac{\eta n^{5/2}\|\by-\bu(0)\|_2^2}{\sqrt{m}\lambda_0\tau\delta}}.
 \end{align*}
\end{lemma}
\begin{lemma}\label{lem:bounds4}
With probability  at least $1-\delta$, we have
\begin{enumerate} 
    \item [(a)] $\|\bZ(k)-\bZ(0)\|_F = O\rbr{\frac{n^{3/4}\|\by-\bu(0)\|_2^{1/2}}{\sqrt{m^{1/2}\lambda_0\tau\delta}}}$;
    \item[(b)] $\|\bH(0)-\bH^{\infty}\|_F=O\rbr{\frac{n\sqrt{\log(n/\delta)}}{\sqrt{m}}}$;
    \item [(c)] $\|\bz_0(\cdot)^\top\bZ(0)-h(\cdot,\bX)\|_2 =  O\rbr{\frac{\sqrt{n}\sqrt{\log(n/\delta)}}{\sqrt{m}}}$;
    \item[(d)] $\|z_0(\cdot)^{\top}{\rm vec}(\bW(0)) \|_2 = O\rbr{\tau \sqrt{\log (1/\delta)}}$.
    % and $\|\bu(0)\|_2 = \|\bZ(0)^{\top}{\rm vec}(\bW(0))\|_2= O\rbr{\sqrt{n}\tau \sqrt{\log (n/\delta)}}$
\end{enumerate}
% \begin{proof}
% The proof of (a) and (b) can be found in \cite{arora2019fine}.

% For (c), the $i$-th coordinates of $\bz_0(\bx)^\top\bZ(0)$ and $h(\bx,\bX)$ are
% \begin{align*}
%     \frac{1}{m}\sum_{r=1}^m \bx^\top\bx_i\mathbb{I}\{\bw^\top_r(0)\bx\ge 0\}\mathbb{I}\{\bw^\top_r(0)\bx_i\ge 0\}, \text{~~and~~} \mathbb{E}_{\bw\sim N(0,\bI)}[\bx^\top\bx_i\mathbb{I}\{\bw^\top\bx\ge 0\}\mathbb{I}\{\bw^\top\bx_i\ge 0\}],
% \end{align*}
% respectively. $\forall i \in  [n]$, $(\bz_0(\bx)^\top\bZ(0))_i$ is the average of $m$ i.i.d. random variables, which have expectation $h_i(\bx,\bX)$ and bounded in $[0,1]$. For any fixed $\bx$, by Hoeffding's inequality, with probability at least $1-\delta^*$,
% \begin{align*}
%     |(\bz_0(\bx)^\top\bZ(0))_i- h_i(\bx,\bX)| \le \sqrt{\frac{\log(2/\delta^*)}{2m}}
% \end{align*}
% holds. By defining $\delta=n\delta^*$ and applying a union bound over all $i\in [n]$, with probability at least $1-\delta$, we have
% \begin{align*}
%     \|\bz_0(\bx)^\top\bZ(0)-h(\bx,\bX)\|_2^2 = O\rbr{n\frac{\log(2n/\delta)}{2m}}
% \end{align*}

% \end{proof}
\end{lemma}

\begin{lemma}[Properties of Loewner order] \label{lem:Loewner}  For two positive semi-definite matrices $\bA$ and $\bB$, 
\begin{enumerate}
    \item [(a).] Suppose $\bA$ is non-singular, then $\bA\ge\bB\Longleftrightarrow \lambda_{max}(\bB\bA^{-1})\le 1$ and $\bA>\bB\Longleftrightarrow \lambda_{\max}(\bB\bA^{-1})>1$, where $\lambda_{\max}(\cdot)$ denotes the maximum eigenvalue of the input matrix.
    \item[(b).] Suppose $\bA$, $\bB$ and $\bQ$ are positive definite, $\bA$ and $\bB$ are exchangeable, then $\bA\ge \bB \Longrightarrow \bA\bQ\bA \ge \bB\bQ\bB$.
\end{enumerate}
\end{lemma}
% Let the initial neural network be $f_0$.
% Note that 
% \begin{align*}
%     \sup_{f\in\hat{\cF}}\|f-f^*\|&\le  \sup_{f\in\hat{\cF}}\|f-f_0\|+\|f_0-f^*\|\\
%     &\le  \sup_{f\in\hat{\cF}}\|f-f_0\|+\|f_0-f^*\|_n + I_2.
% \end{align*}
% For the first term $\sup_{f\in\hat{\cF}}\|f-f_0\|$, consider the event $\{\norm{\bw_r(k) - \bw_r(0)}_2<R\}$ and $ A_{r}(x) := \{|w_r(0)^\topx|\leq R\}$. Then,
% \begin{align*}
%   \sup_{f\in\hat{\cF}}\|f-f_0\|
%     &\le \frac{R}{\sqrt{m}}\|\sum_{r=1}^m\II\{A_r(\bx)\}\| = O\rbr{\frac{R^2\sqrt{m}}{\tau}\vee R\sqrt{\log{\frac{1}{\delta}}}}.
% \end{align*}
% Hence, we have
% \begin{align*}
%     \sup_{f\in\hat{\cF}}\|f-f^*\|&\le \sup_{f\in\hat{\cF}}\|f-f_0\|+\|f_0-f^*\|_n + I_2\\
%     &= O\rbr{\frac{R^2\sqrt{m}}{\tau}\vee \frac{1}{\delta}}
% \end{align*}

\subsection{Proof of Theorem \ref{thm:gd}}
For notational simplification, we use $\hat f_k = f_{\bW(k),\ba}$. Define 
\begin{align}\label{tildeg}
    \tilde{f}_k(\bx)={\rm vec}(\bW(k))^\top\bz_0(\bx),
\end{align}
where $\bz_0(\bx)=\bz(\bx)|_{\bW=\bW(0)}$.
Then we can write the following decomposition
\begin{align}\label{eq2Xinpf41}
    \hat{f}_k-f^* = (\hat{f}_k- \tilde{f}_k) + (\tilde{f}_k - g) + (g - f^*)= \Delta_1+\Delta_2+\Delta_3,
\end{align}
where $g$ is as in \eqref{eqn:gx}.

Before the proof, we provide a road map of this proof. We first show that $\|\Delta_1\|_{2}$ and $\|\Delta_3\|_{2}$ are small. We then show the term $\|\Delta_2\|_{2}$ can be large if the iteration number is too small or too large. Intuitively, if the iteration number if too small, the resulting estimator $\tilde{f}_k$ is not well-trained. On the other hand, if the iteration number is too large, then the resulting estimator $\tilde{f}_k$ could be over-fitted. In either case, the error term $\|\Delta_2\|_{2}$ is large.

It follows from Lemma \ref{lem:intpRKHS} that
\begin{align}\label{Delta3b}
    \|\Delta_3\|_{2}=O_{\PP}\rbr{\sqrt{\frac{{1}}{n}}}.
\end{align} 
% Next, we evaluate $\Delta_1$ and $\Delta_2$. 

For $\Delta_1$, under the assumptions of Lemma \ref{lem:wtoint}, with high probability, we have $\|\bw_r(k)-\bw_r(0) \|_2 \le R_0$. Thus, for fixed $\bx$, we have
\begin{align*}
  | \bw_r(k)^\top\bx-\bw_r(0)^\top\bx | \le \|\bw_r(k)-\bw_r(0) \|_2 \|\bx\|_2\le R_0.
\end{align*}
Define event
\begin{align*}
    B_{r}(\bx)=\{|\bw_r(0)^\top\bx|\leq R_0\}, \forall r\in [m].
\end{align*}
If $\mathbb{I}\{B_{r}(\bx)\}=0$, then we have $\mathbb{I}_{r,k}(\bx)=\mathbb{I}_{r,0}(\bx)$, where $\mathbb{I}_{r,k}(\bx)=\mathbb{I}\{\bw_r(k)^\top\bx\ge 0\}$. Therefore, for any fixed $\bx$, we have
\begin{align*}
%   \Delta_1  &=  \\
       |\hat{f}_k(\bx)- \tilde{f}_k(\bx)| &=\left|\frac{1}{\sqrt{m}}\sum_{r=1}^m a_r (\mathbb{I}_{r,k}(\bx)-\mathbb{I}_{r,0}(\bx))\bw_r(k)^\top\bx\right|\\
        &= \left|\frac{1}{\sqrt{m}}\sum_{r=1}^m a_r \mathbb{I}\{B_{r}(\bx)\} (\mathbb{I}_{r,k}(\bx)-\mathbb{I}_{r,0}(\bx))\bw_r(k)^\top\bx\right|\\
        &\le \frac{1}{\sqrt{m}}\sum_{r=1}^m \mathbb{I}\{B_{r}(\bx)\}|\bw_r(k)^\top\bx|\\
        &\le \frac{1}{\sqrt{m}} \sum_{r=1}^m \mathbb{I}\{B_{r}(\bx)\}\rbr{|\bw_r(0)^\top\bx|+ |\bw_r(k)^\top\bx-\bw_r(0)^\top\bx |}\\
        &\le  \frac{2R_0}{\sqrt{m}} \sum_{r=1}^m \mathbb{I}\{B_{r}(x)\}
\end{align*}
Recall that $\|\bx\|_2=1$, which implies that $\bw_r(0)^\top\bx$ is distributed as $N (0,\tau^2)$. Therefore, we have 
\begin{align*}
    \mathbb{E}[\mathbb{I}\{B_r(x)\}] = \PP\rbr{|\bw_r(0)^\top\bx|\leq R_0}
    =\int_{-R_0}^{R_0}\frac{1}{\sqrt{2\pi}\tau}\exp\left\{-\frac{u^2}{2\tau^2}\right\}du\le \frac{2R_0}{\sqrt{2\pi}\tau}.
\end{align*}
By Markov's inequality, with probability at least $1-\delta$, we have 
\begin{align*}
     \sum_{r=1}^m \mathbb{I}\{B_{r}(x)\} \le \frac{2mR_0}{\sqrt{2\pi}\tau\delta}.
\end{align*}
Thus, we have
\begin{align}\label{Delta1b}
    \|\Delta_1\|_{2} \le  \frac{2R_0}{\sqrt{m}} \|\sum_{r=1}^m \mathbb{I}\{B_{r}(\cdot)\}\|_{2}\le \frac{4\sqrt{m}R_0^2}{\sqrt{2\pi}\tau\delta}= O\rbr{\frac{n\|\by-\bu(0)\|_2^2}{\sqrt{m}\tau\lambda_0^2\delta}}.
\end{align}
Next, we evaluate $\Delta_2$. Recall that the GD update rule is
\begin{align*}
    \mathrm{vec}(\bW(j+1))=\mathrm{vec}(\bW(j))-\eta \bZ(j)(\bu(j)-\by), j\ge 0.
\end{align*}
Applying Lemma \ref{lem:uky}, we can get 
\begin{align*}
   & {\rm vec}(\bW(k)) - {\rm vec}(\bW(0)) \\
   = & \sum_{j=0}^{k-1} ({\rm vec}(\bW(j+1)) - {\rm vec}(\bW(j)))\\
    = & -\sum_{j=0}^{k-1} \eta \bZ(j)(\bu(j) - \by)\\
    = & \sum_{j=0}^{k-1} \eta\bZ(j)(\bI - \eta \bH^{\infty})^j(\by-\bu(0))-\sum_{j=0}^{k-1} \eta\bZ(j)\be(j)\\
    = & \sum_{j=0}^{k-1} \eta\bZ(0)(\bI - \eta \bH^{\infty})^j(\by-\bu(0))+
    \sum_{j=0}^{k-1} \eta(\bZ(j)-\bZ(0))(\bI - \eta \bH^{\infty})^j(\by-\bu(0))
    -\sum_{j=0}^{k-1} \eta\bZ(j)\be(j)\\
    = & \sum_{j=0}^{k-1} \eta\bZ(0)(\bI - \eta \bH^{\infty})^j(\by-\bu(0))+\zeta(k).
\end{align*}
For the first term of $\zeta(k)$, applying Lemma \ref{lem:bounds4} (a), with probability at least $1-\delta$, we get 
\begin{align*}
    &\|\sum_{j=0}^{k-1} \eta(\bZ(j)-\bZ(0))(\bI - \eta \bH^{\infty})^j(\by-\bu(0))\|_2\\
    \le & \sum_{j=0}^{k-1} O\rbr{\frac{n^{3/4}\|\by-\bu(0)\|_2^{1/2}}{\sqrt{m^{1/2}\lambda_0\tau\delta}}} \eta\|\bI - \eta \bH^{\infty}\|_2^j\|(\by-\bu(0))\|_2\\
    \le &O\rbr{\frac{n^{3/4}\|\by-\bu(0)\|_2^{3/2}}{\sqrt{m^{1/2}\lambda_0\tau\delta}}}\sum_{j=0}^{k-1} \eta(1-\eta\lambda_0)^j\\
    = &  O\rbr{\frac{n^{3/4}\|\by-\bu(0)\|_2^{3/2}}{m^{1/4}\tau^{1/2}\lambda_0^{3/2}\delta^{1/2}}}.
\end{align*}

Denote that $z_i(j)=z(\bx_i)|_{\bW=\bW(j)}$. By (\ref{z(x)}), we have $\|\bz_i(j)\|_2 \le 1$. Thus, 
\begin{align}\label{boundZ(K)}
    \|\bZ(j)\|_F = \rbr{\sum_{i=1}^n\|\bz_i(j)\|_2^2}^{\frac{1}{2}} \le \sqrt{n}~~,\forall ~j\ge 0.
\end{align}
 For the second term of $\zeta(k)$, we have
\begin{align*}
    &\|\sum_{j=0}^{k-1} \eta\bZ(j)\be(j)\|_2 \\
    \le & \sum_{j=0}^{k-1}\eta\|\bZ(j)\|_F\|\be(j)\|_2\\
    \le & \sum_{j=0}^{k-1}\eta \sqrt{n}  O\rbr{j\rbr{1-\frac{\eta\lambda_0}{4}}^{j-1}\frac{\eta n^{5/2}\|\by-\bu(0)\|_2^2}{\sqrt{m}\tau\lambda_0\delta}}\\
    % \alert{\mbox{WW: more details are needed here}}
    = & O\rbr{\frac{ n^{3}\|\by-\bu(0)\|_2^2}{\sqrt{m}\lambda_0^3\tau\delta}}.
\end{align*}
Therefore,
\begin{align}\label{bound:zeta_K}
    \|\zeta(k)\|_2 =  O\rbr{\frac{n^{3/4}\|\by-\bu(0)\|_2^{3/2}}{m^{1/4}\tau^{1/2}\lambda_0^{3/2}\delta^{1/2}}} + O\rbr{\frac{ n^{3}\|\by-\bu(0)\|_2^2}{\sqrt{m}\lambda_0^3\tau\delta}}. 
\end{align}

Define $\bG_k=\sum_{j=0}^{k-1} \eta(\bI - \eta \bH^{\infty})^j$. Recalling that $\by=\by^*+\bepsilon$, for fixed $\bx$, we have 
\begin{align}\label{eq1Xinpf41}
    \tilde{f}_k(\bx) - g(\bx) =& \bz_0(\bx)^\top{\rm vec}(\bW(k))-h(\bx,\bX)(\bH^\infty)^{-1}\by^* \nonumber\\
     =& \bz_0(\bx)^\top \bigl[ \bZ(0)\bG_k(\by-\bu(0))+\zeta(k) +{\rm vec}(\bW(0))\bigr]\nonumber\\
      =& \bigl[ h(\bx,\bX)(\bG_k-(\bH^\infty)^{-1})\by^*+h(\bx,\bX)\bG_k\bepsilon \bigr]+ \bigl[\bz_0(\bx)^\top\bZ(0)-h(\bx,\bX)\bigr]\bG_k\by\nonumber\\
      & + \bigl[\bz_0(\bx)^\top{\rm vec}(\bW(0)) + \bz_0(\bx)^\top\zeta(k)-\bz_0(\bx)^\top\bZ(0)\bG_k\bu(0)\bigr]\nonumber\\
      =& \Delta_{21}(\bx)+\Delta_{22}(\bx)+\Delta_{23}(\bx).
\end{align}
Using Lemma \ref{lem:bounds4} (c), we can bound $\Delta_{22}$ as
\begin{align}\label{Delta22b}
    \|\Delta_{22}\|_2 \le& \|\bz_0(\bx)^\top\bZ(0)-h(\bx,\bX)\|_2 \|\bG_k\by\|_2\nonumber\\
    \le & O\rbr{\frac{\sqrt{n}\sqrt{\log(n/\delta)}}{\sqrt{m}}} \|(\bH^{\infty})^{-1}\by\|_2\nonumber\\
   = & O\rbr{\frac{\sqrt{n}\sqrt{\log(n/\delta)}\|\by\|_2}{\sqrt{m}\lambda_0}}.
\end{align}
Since the $i$-th coordinate of $\bu(0)$ is 
\begin{align*}
    u_i(0)=\bz_0(\bx_i)^\top{\rm vec}(\bW(0)) = \sum_{r=1}^m a_r \bw(0)^{\top}\bx_i\mathbb{I}\{\bw(0)^{\top}\bx_i\},
\end{align*}
where $a_r\sim{\rm unif}\{1,-1\}$ and $\bw(0)^{\top}\bx_i \sim N(0,\tau^2)$, it is easy to prove that $u_i(0)$ has zero mean and variance $\tau^2$.
This implies $\mathbb{E}[\|\bu(0)\|_2^2]=O(n\tau^2)$. By Markov's inequality, with probability at least $1-\delta$, we have $\|\bu(0)\|_2= O\rbr{\frac{\sqrt{n}\tau}{\delta}}$. Similar to \eqref{boundZ(K)}, we can obtain $\|\bZ(0)\|_F=O(\sqrt{n})$. Thus, 
\begin{align}\label{eqn:zZGu}
    &|\bz_0(\bx)^\top\bZ(0)\bG_k\bu(0)|
    \le \|\bz_0(\bx)\|_2\|\bZ(0)\|_F\|\bG_k\bu(0)\|_2
    \le  \sqrt{n}\|(\bH^{\infty})^{-1}\bu(0)\|_2 = O\rbr{\frac{n\tau}{\lambda_0\delta}}.
\end{align}
Combining Lemma \ref{lem:bounds4} (d), (\ref{bound:zeta_K}) and (\ref{eqn:zZGu}), we obtain
\begin{align}\label{Delta23b}
    \|\Delta_{23}\|_2  \le & \|\bz_0(\cdot)^\top{\rm vec}(\bW(0))\|_2+\|\bz_0(\cdot)\|_2\|\zeta(k)\|_2+\|\bz_0(\cdot)^\top\bZ(0)\bG_k\bu(0)\|_2 \nonumber\\ 
    = &   O\rbr{\tau \sqrt{\log (1/\delta)}}+O\rbr{\frac{n^{3/4}\|\by-\bu(0)\|_2^{3/2}}{m^{1/4}\tau^{1/2}\lambda_0^{3/2}\delta^{1/2}}} + O\rbr{\frac{ n^{3}\|\by-\bu(0)\|_2^2}{\sqrt{m}\lambda_0^3\tau\delta}}+O\rbr{\frac{n\tau}{\lambda_0\delta}}\nonumber\\
    = & O\rbr{\frac{n^{3/4}\|\by-\bu(0)\|_2^{3/2}}{m^{1/4}\tau^{1/2}\lambda_0^{3/2}\delta^{1/2}}} + O\rbr{\frac{ n^{3}\|\by-\bu(0)\|_2^2}{\sqrt{m}\lambda_0^3\tau\delta}}+O\rbr{\frac{n\tau}{\lambda_0\delta}}.
\end{align}

% Since $\bu(0)=\bZ(0)^{T}{\rm vec}(\bW(0))$, the second term of $\Delta_2$ can be written as
% \begin{align}\label{eqn:Gamma2}
%     \Gamma_2 = & \bz_0(\bx)^\top{\rm vec}(\bW(0))-\bz_0(\bx)^\top\bZ(0)\bG_K\bu(0)\nonumber\\
%     &=\bz_0(\bx)^\top(\bI-\bZ(0)(\bH^{\infty})^{-1}\bZ(0)^\top){\rm vec}(\bW(0)) + \bz_0(\bx)^\top\bZ(0)((\bH^{\infty})^{-1}-\bG_K)\bu(0)\nonumber\\
%     = & F_1+ (\bz_0(\bx)^\top\bZ(0)-h(\bx,\bX))((\bH^{\infty})^{-1}-\bG_K)\bu(0)+h(\bx,\bX)((\bH^{\infty})^{-1}-\bG_K)\bu(0)\nonumber\\
%     = & F_1+ F_2 + h(\bx,\bX)((\bH^{\infty})^{-1}-\bG_K)\bu(0)
% \end{align}

% For $F_2$, using $\|\bW(0)\|_F = O\rbr{\sqrt{m\tau/\delta}}$
% \begin{align*}
%     \|F_2\|_2 \le & \|\bz_0(\bx)^\top\bZ(0)-h(\bx,\bX)\|_2 \|(\bH^{\infty})^{-1}-\bG_K\|_F\|\bu(0)\|_2\\
%     \le & O\rbr{\frac{\sqrt{n}\sqrt{\log(n/\delta)}}{\sqrt{m}}} \|(\bH^{\infty})^{-1}\|_F\|\bZ(0)\|_F\|\bW(0)\|_F\\
%     = &O\rbr{\frac{n^{3/2}\sqrt{\tau}\sqrt{\log(n/\delta)}}{\lambda_0\sqrt{\delta}}}
% \end{align*}

% Combining \ref{eqn:Gamma1} and \ref{eqn:Gamma2}, we get
% \begin{align*}
%   \Delta_2 = &   h(\bx,\bX)((\bH^{\infty})^{-1}-\bG_K)\by^*+h(\bx,\bX)\bG_K\bepsilon+ F_1+ F_2+L_1+\Gamma_3\\
%   = & h(\bx,\bX)((\bH^{\infty})^{-1}-\bG_K)\by^*+h(\bx,\bX)\bG_K\bepsilon + \Xi_0.
% \end{align*}
By \eqref{eq2Xinpf41} and \eqref{eq1Xinpf41}, we can rewrite $\hat{f}_k-f^*$ as
\begin{align*}
    \hat{f}_k-f^* &=\Delta_{21}+(\Delta_1+\Delta_3+\Delta_{22}+\Delta_{23}):=\Delta_{21} + \Xi,
\end{align*}
Next we show that the expected value of $\|\Xi\|_2^2$ over noise, $\mathbb{E}_{\bepsilon} \|\Xi\|_2^2$, is small. Note that we have
\begin{align}\label{bnd:y}
    \mathbb{E}_{\bepsilon} \|\by\|_2^2 = \mathbb{E}_{\bepsilon} \|\by^*+\bepsilon\|_2^2\le 2\by^{*\top}\by^*+2\mathbb{E}_{\bepsilon} \bepsilon^{\top} \bepsilon = O(n).
\end{align}
By Markov's inequality, with probability $1-\delta$ over random initialization, we have
\begin{align}\label{bnd:y-u0}
    \mathbb{E}_{\bepsilon} \|\by-\bu(0)\|_2 \le & \rbr{\mathbb{E}_{\bepsilon} \|\by-\bu(0)\|_2^2 }^{\frac{1}{2}}\nonumber\\
    \le & \rbr{\frac{3\mathbb{E}_{\bW(0),\ba}\left[\bu(0)^{\top}\bu(0)+\by^{*\top}\by^*+\mathbb{E}_{\bepsilon}\bepsilon^{\top} \bepsilon\right]}{\delta} }^{\frac{1}{2}}\nonumber\\
    = &  O\rbr{\sqrt{\frac{n(1+\tau^2)}{\delta}}} =  O\rbr{\sqrt{\frac{n}{\delta}}},
\end{align}
where the last equality of \ref{bnd:y-u0} is because $\tau^2\lesssim 1$.
By \eqref{Delta3b}, \eqref{Delta1b}, \eqref{Delta22b}, \eqref{Delta23b}, \eqref{bnd:y} and \eqref{bnd:y-u0}, $\mathbb{E}_{\bepsilon} \|\Xi\|_2^2$ can be upper bounded as
\begin{align*}
   \mathbb{E}_{\bepsilon} \|\Xi\|_2^2 \le & 4\mathbb{E}_{\bepsilon} (\|\Delta_1\|_2^2+\|\Delta_3\|_2^2+\|\Delta_{22}\|_2^2+\|\Delta_{23}\|_2^2)\\
    = &\mathbb{E}_{\bepsilon}\left[ O\rbr{\frac{n^2\|\by-\bu(0)\|_2^4}{m\tau^2\lambda_0^4\delta^2}}+ O\rbr{\frac{1}{n}}+O\rbr{\frac{n\log(n/\delta)\|\by\|_2^2}{m\lambda_0^2}}\right]+4\mathbb{E}_{\bepsilon}\|\Delta_{23}\|_2^2\\
   \le & O\rbr{\frac{n^4}{m\tau^2\lambda_0^4\delta^4}}+ O\rbr{\frac{1}{n}}+O\rbr{\frac{n^2\log(n/\delta)}{m\lambda_0^2\delta}}+O\rbr{\frac{n^2\tau^2}{\lambda_0^2\delta^2}}+\\
   &+ \mathbb{E}_{\bepsilon}\left[O\rbr{\frac{n^{3/2}\|\by-\bu(0)\|_2^{3}}{m^{1/2}\tau\lambda_0^{3}\delta}} + O\rbr{\frac{ n^{6}\|\by-\bu(0)\|_2^4}{m\tau^2\lambda_0^6\delta^2}}\right]\\
   = & O\rbr{\frac{n^4}{m\tau^2\lambda_0^4\delta^4}}+ O\rbr{\frac{1}{n}}+O\rbr{\frac{n^2\log(n/\delta)}{m\lambda_0^2\delta}}+O\rbr{\frac{n^2\tau^2}{\lambda_0^2\delta^2}}\\
   & + O\rbr{\frac{n^{3}}{\sqrt{m}\tau\lambda_0^{3}\delta^{5/2}}}
    + O\rbr{\frac{n^{8}}{m\tau^2\lambda_0^6\delta^4}}\\
    = & O\rbr{\frac{1}{n}}+O\rbr{\frac{n^2\tau^2}{\lambda_0^2\delta^2}} + \frac{{\rm poly}\rbr{n, \frac{1}{\lambda_0}, \frac{1}{\delta}}}{m^{\frac{1}{2}}\tau}.
\end{align*}
% \alert{write all these conditions in Theorem 4.1. The second term is not $O\rbr{\frac{1}{\sqrt{n}}}$}, we have $ \|\Xi\|_2 = O_{\mathbb{P}}\rbr{n^{-\frac{1}{2}}}$.

In the following, we will evaluate $\Delta_{21}$ and discuss how the iteration number $k$ would affect the $L_2$ estimation error $\|\hat{f}_k-f^*\|_{2}^2$. 
\paragraph{Case 1: The iteration number $k$ cannot be too small}
By taking expectation of $\|\Delta_{21}\|_{2}^2$ over the noise,
% Since 
% \begin{align*}
%   (\bH^{\infty})^{-1}-\bG_K=(\bH^{\infty})^{-1} (\bI-\eta\bH^{\infty} )^k= (\bI-\eta\bH^{\infty} )^k(\bH^{\infty})^{-1}, 
% \end{align*}
we have 
\begin{align*}
    \mathbb{E}_{\bepsilon}\|\Delta_{21}\|_{2}^2 &= \int_{\bx\in \Omega}h(\bx,\bX)\bigl[(\bH^{\infty})^{-1}-\bG_k)\by^*\by^{*\top}((\bH^{\infty})^{-1}-\bG_k)+\bG_k^2\bigr]h(\bX,\bx)d\bx\\
    &=\int_{\bx\in \Omega}h(\bx,\bX)(\bH^{\infty})^{-1}\bM_k(\bH^{\infty})^{-1}h(\bX,\bx)d\bx,
\end{align*}
where 
\begin{align}\label{eq:M_k}
    \bM_k =& (\bI-\eta\bH^{\infty} )^k\bS (\bI-\eta\bH^{\infty})^k+(\bI - (\bI-\eta\bH^{\infty} )^k)^{2}\nonumber\\
    =& [(\bI-\eta\bH^{\infty})^k-(\bS+\bI)^{-1}](\bS+\bI)[(\bI-\eta\bH^{\infty})^k-(\bS+\bI)^{-1}]+\bI-(\bS+\bI)^{-1}
\end{align}
and $\bS=\by^*\by^{*\top}$. If $k\geq C_0\rbr{\frac{\log n}{\eta\lambda_0}}$ for some constant $C_0>1$, we have
    \begin{align*}
        (\bI-\eta\bH^{\infty} )^k \le (1-\eta\lambda_0)^k\bI \le \exp\{-\eta\lambda_0k\}
        \bI \le \exp\{-C_0\log n\}\bI=\frac{1}{n^{C_0}}\bI,
    \end{align*}
    Since $1+\|\by^*\|_2^2 \le C_1 n$ for some constant $C_1$, we have
    \begin{align*}
        \lambda_{\max}\rbr{\frac{1}{n^{C_0}}(\bS+\bI)} = \frac{1+\|\by^*\|_2^2 }{n^{C_0}} \le \frac{C_1}{n^{C_0-1}}< 1.
    \end{align*}
By Lemma \ref{lem:Loewner} (a), we have
\begin{align*}
  (\bI-\eta\bH^{\infty} )^k\le \frac{1}{n^{C_0}}\bI < (\bS+\bI)^{-1}.
\end{align*} 
Therefore, we have
\begin{align*}
    (\bS+\bI)^{-1} - (\bI-\eta\bH^{\infty} )^k \ge (\bS+\bI)^{-1} -\frac{1}{n^{C_0}}\bI,
\end{align*}
where $(\bS+\bI)^{-1} - (\bI-\eta\bH^{\infty} )^k$ and $(\bS+\bI)^{-1} -n^{-C_0}\bI$ are positive definite matrices. It is also obvious that the two matrices are exchangeable. By Lemma \ref{lem:Loewner} (b) and (\ref{eq:M_k}), we have
\begin{align*}
  \bM_k \ge \rbr{1-\frac{1}{n^{C_0}}}^2\bI+\frac{1}{n^{2C_0}}\bS.
\end{align*}
Then we have
    \begin{align*}
         \mathbb{E}_{\bepsilon}\|\Delta_{21}\|_{2}^2  \ge \rbr{1-\frac{1}{n^{C_0}}}^2I_1 + \frac{1}{n^{2C_0}} I_2\ge c_0I_1
    \end{align*}
    where $c_0 \in (0,1)$ is a constant, 
\begin{align*}
     I_1 = \int h(\bx,\bX)(\bH^{\infty})^{-2}h(\bX,\bx)d\bx, \text{~~~and ~~~} I_2 =\int[h(\bx,\bX)(\bH^{\infty})^{-1}\by^*]^2d\bx.
\end{align*} 
% Next, we show a lower bound of $I_1$. For any fixed $\bx$, by the Cauchy-Schwarz inequality, we have
% \begin{align}
%     h(\bx,\bX)(\bH^{\infty})^{-2}h(\bX,\bx) \geq & h(\bx,\bX)(\bH^{\infty})^{-1/2}(\bH^{\infty} + n\bI)^{-1}(\bH^{\infty})^{-1/2}h(\bX,\bx)\nonumber\\
%     \geq & 
% \end{align}
By the Cauchy-Schwarz inequality, we have
\begin{align}\label{PredErBigK}
  \mathbb{E}_{\bepsilon} \|\hat{f}_k-f^*\|_{2}^2 = &\mathbb{E}_{\bepsilon}\|\Delta_{21}+\Xi\|_2^2 \nonumber\\
  \ge& \frac{1}{2}\mathbb{E}_{\bepsilon}\|\Delta_{21}\|_{2}^2-\mathbb{E}_{\bepsilon}\|\Xi\|_{2}^2 \nonumber\\
  \ge & \frac{c_0}{2}I_1-O\rbr{\frac{1}{n}}-O\rbr{\frac{n^2\tau^2}{\lambda_0^2\delta^2}} - \frac{{\rm poly}\rbr{n, \frac{1}{\lambda_0}, \frac{1}{\delta}}}{m^{\frac{1}{2}}\tau}.
\end{align}
Let $\tau \leq C_3\frac{\lambda_0\delta}{n}\|(\bH^{\infty})^{-1}h(\bX,\cdot)\|_2$ for some constant $C_3>0$ such that the third term of (\ref{PredErBigK}) is bounded by $\frac{c_0}{4}\|(\bH^{\infty})^{-1}h(\bX,\cdot)\|_2^2$.
% Since 
% Assume $\tau = O\rbr{\frac{\lambda_0\delta\|(\bH^{\infty})^{-1}h(\bX,\cdot)\|_2}{n}}$ such that the third term of (\ref{PredErBigK}) is $\frac{c_1}{2}\|(\bH^{\infty})^{-1}h(\bX,\cdot)\|_2^2$ and $0<c_1<c_0$. 
% Since 
% \begin{align*}
%     \|(\bH^{\infty})^{-1}h(\bX,\cdot)\|_2^2 \le \frac{1}{\lambda_0^2}\|h(\bX,\cdot)\|_2^2=O\rbr{\frac{n}{\lambda_0^2}}, 
% \end{align*}
% this implies $\tau = O\rbr{\frac{\delta}{\sqrt{n}}}$. 
Therefore, $\mathbb{E}_{\bepsilon} \|\hat{f}_k-f^*\|_{2}^2$ can be lower bounded as
\begin{align}\label{pf41Klarge}
   \mathbb{E}_{\bepsilon} \|\hat{f}_k-f^*\|_{2}^2 \ge C^*_1\|(\bH^{\infty})^{-1}h(\bX,\cdot)\|_2^2 - O\rbr{\frac{1}{n}},
\end{align}
where $C^*_1>0$ is a constant. Note that $I_1$ is $\mathbb{E}_{\epsilon}\|\hat f_\infty - g^*\|_2^2$, where $g^*\equiv 0$ and 
$\hat f_\infty$ is the interpolated estimator of $g^*$, as in Theorem \ref{thm:kernel_earlystop}. Therefore, by Theorem \ref{thm:kernel_earlystop}, there exists a constant $c_1$ such that $\mathbb{E}_{\epsilon}\|\hat f_\infty - g^*\|_2^2\geq c_1$, which implies $I_1\geq c_1$. Taking $n$ large enough such that the second term in \eqref{pf41Klarge} is smaller than $C_1^*c_1$, we finish the proof of the case that $k$ is large.
\paragraph{Case 2: The iteration number $k$ cannot be too large} We can rewrite $\Delta_{21}$ as
\begin{align*}
    \Delta_{21} =&  h(\bx,\bX)\bG_k(\by^*+\bepsilon)-h(\bx,\bX)(\bH^\infty)^{-1}\by^*\\
     = & \Delta_{21}^*-h(\bx,\bX)(\bH^\infty)^{-1}\by^*.
\end{align*}
Since
\begin{align*}
   \bG_k=\sum_{j=0}^{k-1} \eta(\bI - \eta \bH^{\infty})^j =\sum_{j=0}^{k-1} \eta\sum_{i=1}^n(1 - \eta \lambda_i)^j\bv_i\bv_i^{\top}\le \eta k\bI,
\end{align*}
we have
\begin{align*}
    \mathbb{E}_{\bepsilon}\|\Delta_{21}^*\|_{2}^2 = & \int_{\bx\in \Omega}h(\bx,\bX)\bG_k(\bS+\bI)\bG_kh(\bX,\bx)d\bx\\
    \le & \eta^2 k^2\int_{\bx\in \Omega}h(\bx,\bX)(\bS+\bI)h(\bX,\bx)d\bx\\
     = & \eta^2 k^2\rbr{\int_{\bx\in \Omega}\bigl[h(\bx,\bX)\by^*\bigr]^2d\bx+ \|h(\cdot,\bX)\|_2^2}\\
    %  \le & \eta^2 k^2\rbr{1+\|\by^*\|_2^2}\|h(\cdot,\bX)\|_2^2\\
     = & O\rbr{\eta^2 k^2n^2}.
\end{align*}
Therefore,
\begin{align}\label{PredErSmallK}
  \mathbb{E}_{\bepsilon} \|\hat{f}_k-f^*\|_{2}^2 = & \mathbb{E}_{\bepsilon}\|\Delta_{21}^*+\Xi-h(\cdot,\bX)(\bH^\infty)^{-1}\by^*\|_2^2\nonumber\\
   \ge & \frac{1}{2}\|h(\cdot,\bX)(\bH^\infty)^{-1}\by^*\|_2^2-\mathbb{E}_{\bepsilon}\|\Delta_{21}^*+\Xi\|_{2}^2\nonumber\\
   \ge & \frac{1}{2}\|h(\cdot,\bX)(\bH^\infty)^{-1}\by^*\|_2^2-2\mathbb{E}_{\bepsilon}\|\Delta_{21}^*\|^2_2-2\mathbb{E}_{\bepsilon}\|\Xi\|_{2}^2\nonumber\\
   \ge & \frac{1}{2}\|h(\cdot,\bX)(\bH^\infty)^{-1}\by^*\|_2^2 - O\rbr{\eta^2 k^2n^2} \nonumber\\
   &-O\rbr{\frac{1}{n}}-O\rbr{\frac{n^2\tau^2}{\lambda_0^2\delta^2}} - \frac{{\rm poly}\rbr{n, \frac{1}{\lambda_0}, \frac{1}{\delta}}}{m^{\frac{1}{2}}\tau}.
\end{align}
Let $k \leq C_1\rbr{\frac{1}{\eta n}}$ for some constant $C_1>0$ such that the the second term of (\ref{PredErSmallK}) can be bounded by $\frac{1}{8}\|h(\cdot,\bX)(\bH^\infty)^{-1}\by^*\|_2^2$. Let $\tau \leq C_2\rbr{\frac{\delta\lambda_0}{n}}$ for some constant $C_2>0$ such that the fourth term in (\ref{PredErSmallK}) can be bounded by $\frac{1}{8}\|h(\cdot,\bX)(\bH^\infty)^{-1}\by^*\|_2^2$. Note that we can also choose $m$ such that the fifth term in (\ref{PredErSmallK}) is bounded by $\frac{1}{8}\|h(\cdot,\bX)(\bH^\infty)^{-1}\by^*\|_2^2$. Therefore, we have
\begin{align}\label{pf41Ksmall}
    \mathbb{E}_{\bepsilon} \|\hat{f}_k-f^*\|_{2}^2 \ge &  C^*_2\|h(\cdot,\bX)(\bH^\infty)^{-1}\by^*\|_2^2-O\rbr{\frac{1}{n}}\nonumber\\
    \ge & C_3^*\|f^*\|_2^2 - O\rbr{\frac{1}{n}},
\end{align}
where the last inequality is because of Lemma \ref{lem:intpRKHS}, and $C^*_2>0$ is a constant. By taking $n$ large enough such that the second term in \eqref{pf41Ksmall} is smaller than $C_3^*\|f^*\|_2^2/2$, we finish the proof.

% Since $K=O\rbr{\frac{1}{\eta n}}$ and $\|h(\cdot,\bX)(\bH^\infty)^{-1}\by^*\|_2^2=O(1)$, we can assume the second term of (\ref{PredErSmallK}) as $c_2\|h(\cdot,\bX)(\bH^\infty)^{-1}\by^*\|_2^2$, where $0<c_2<\frac{1}{2}$. Suppose that $\tau = O\rbr{\frac{\delta\lambda_0}{n}}$ and the forth term of (\ref{PredErSmallK}) can be written as $c_3\|h(\cdot,\bX)(\bH^\infty)^{-1}\by^*\|_2^2$, such that $0<c_3<\frac{1}{2}$ and $c_2+c_3<\frac{1}{2}$. Thus, we have
% \begin{align*}
%     \mathbb{E}_{\bepsilon} \|\hat{f}_K-f^*\|_{2}^2 \ge C^*_2\|h(\cdot,\bX)(\bH^\infty)^{-1}\by^*\|_2^2-O\rbr{\frac{1}{n}},
% \end{align*}
% where $C^*_2>0$ is a constant.
\subsection{Proof of Theorem \ref{thm:kernel_earlystop}}
Let's first introduce the GD update for the kernel ridge regression. By the representer
theorem \citep{kimeldorf1971some}, the kernel estimator can be written as 
\[
\hat{f}(\bx) = \sum_{i=1}^n \omega_i h(\bx, \bx_i) := h(\bx, \bX)\bomega,
\]
where $\bomega=(\omega_1,\ldots, \omega_n)$ is the coefficient vector.
Consider using the squared loss 
\[
\Phi(\bomega) = \frac{1}{2}\sum_{i=1}^n (\hat{f}(\bx_i)-y_i)^2.
\]
Let $\bomega_k$ be the $\bomega$ at the $k$-th GD iteration and choose $\bomega_0 = \mathbf{0}$. Then, the GD update rule for estimating $\bomega$ can be expressed as
\begin{align}
\label{eqn:kernel_gd}
\bomega_{k+1} = \bomega_{k} - \eta \rbr{(\bH^\infty)^2\bomega - \bH^\infty \by} 
\end{align}
In the formulation of the stopping rule, two quantities play an important role: first,
the {running sum} of the step sizes
$\alpha_j := \sum_{i=0}^{j}{\eta_i},$
and secondly, the eigenvalues $\hat{\lambda}_1 \geq \hat{\lambda}_2 \geq
\cdots \geq \hat{\lambda}_n \geq 0$ of the empirical kernel matrix $H^\infty$, which are computable from the data. Recall the definition of the optimal stopping time $k^*$ as in (\ref{eqn:kstar}). The following lemma establishes the $L_2$ estimation results for $\hat{f}_{k^*}$ for kernels with polynomial eigendecay.
% The stopping time $T$ is defined as
% \begin{align*}
% k^* & := \argmin \biggr \{ j \in \NN \, \mid
% \hat{\cR}_{H} \big(1/\sqrt{\alpha_{j}}\big) > (2 e \sigma
% \alpha_{j})^{-1} \biggr \} - 1,
% \end{align*}
% where $\hat{\cR}_{H}$ is the local empirical
% Rademacher complexity that
% \begin{align*}
% \hat{\cR}_{H}(\epsilon) & := \biggr[ \frac{1}{n}
%   \sum_{i=1}^n \min \big \{ \hat{\lambda}_i, \epsilon^2 \big \}
%   \biggr]^{1/2}.
% \end{align*}

\begin{lemma} [Corollary 1 in \cite{raskutti2014early}]
\label{lemma:early}
Suppose that variables $\{\bx_i\}_{i=1}^n$ are sampled i.i.d. and the kernel class $\cN$ satisfies the
polynomial eigenvalue decay
$\lambda_j  \lesssim j^{-2\nu}$ for some $\nu > 1/2$.
Then there is a universal constant $C$ such
that
\begin{align*}
% \label{EqnSobolevUpper}
\EE \|\hat{f}_{k^*}- f^*\|_2^2 & \leq C
\bigg(\frac{\sigma^2}{n} \bigg)^{\frac{2 \nu}{2 \nu+1}}.
\end{align*}
Moreover, if $\lambda_j \asymp j^{-2 \nu}$ for all $j = 1, 2, \ldots$, then for all iterations
  $k = 1, 2, \ldots$,
\begin{align*}
\EE \|\hat{f}_{k^*}- f^*\|_2^2 & \geq \frac{ \sigma^2}{4} \min
\big \{ 1, \; \frac{(\alpha_{k})^{\frac{1}{2
      \nu}}}{n} \big \}.
\end{align*}
\end{lemma}

% First we consider the eigenvalues of the RKHS of kernel $h(\bs,\bt)$ defined in (\ref{ntkh}).

By Lemma \ref{lemeigendecay}, apply Lemma \ref{lemma:early} with $2\nu=d/(d-1)$ and the running sum of the step sizes $\alpha_k =k\eta$ gives the convergence rate. 
% Let $T$ be the embedding operator of $\mathcal{N}$ into $L_2(\mathbb{S}^{d-1})$, and $T^*$ be the adjoint of $T$. Therefore, we have
% \begin{align*}
% T^*v(\bs) = \int_\Omega h(\bs,\bt)v(\bt)d\bt,\qquad v\in L_2(\mathbb{S}^{d-1}), \qquad \bs\in \mathbb{S}^{d-1}.
% \end{align*}
% By Theorem 5.10 in \cite{edmunds1987spectral}, for all $k\in \mathbb{N}$, $a_k(T)=\nu_k(T)$, where $a_k(T)$ denotes the approximation number for the embedding operator (as well as the integral operator), and $\nu_k$ denotes the singular value of $T$. By Theorem 5.7 in \cite{edmunds1987spectral}, $T^*T\varphi_k=\nu_k^2\varphi_k$, and $T^*T\varphi_k=T^*\varphi_k=\lambda_k\varphi_k$, we have $\lambda_k=\nu_k^2$. Therefore, we have the approximation number $a_k(T)=\sqrt{\lambda_k}$ for $j=0,1,...$. By Lemma \ref{lemmercerh}, the approximation number can be bounded if we can establish a relationship between $\mu_k$ in Lemma \ref{lemmercerh} and $\lambda_k$.

% Then the $k$-th approximation number $a_k(T)=\sqrt{\lambda_k} \asymp k^{-\frac{d}{2(d-1)}}$.

Moreover, if $k\to\infty$, i.e., interpolation of training data, the lower bound result in Lemma \ref{lemma:early} implies
$\EE \|f_{\hat{T}}- f^*\|_2^2 \gtrsim  \sigma^2$ that doesn't converge to 0. 

% The population local Rademacher complexity is
% \begin{align*}
% \cR_{\mathbb{H}}(\epsilon)&:=\rbr{\frac{1}{n}\sum_{j=0}^\infty \min\{\lambda_j, \epsilon^2\}}^{1/2}\\
% &\le \sqrt{\frac{M}{n}}\epsilon + O\rbr{\sqrt{\frac{1}{n}\sum_{j=M}^\infty j^{-\frac{d}{d-1}}}}\\
% &\le \sqrt{\frac{M}{n}}\epsilon + O\rbr{\frac{M^{-\frac{1}{2(d-1)}}}{\sqrt{n}}}
% \end{align*}
% for any $M\in\NN$. Choosing $M\asymp \epsilon^{-2(d-1)/d}$ yields $\cR_{\mathbb{H}}(\epsilon)\asymp \epsilon^{1/d}$.
% The optimal $K$ is then the minimum such $k\in\NN$ that 
% $\alpha_k^{(2d-1)/2d}\ge C$ for some constant $C$, which reduces to $K \asymp n^2/\lambda_0 $.

\section{Proofs of main theorems in Section \ref{secwithp}}

\subsection{Proof of Theorem \ref{thm:withPenalty}}

Let $\bu_D(l) = (u_{D,1}(l),...,u_{D,n}(l))^\top\in \RR^n$ be the predictions on the points $\bx_1,...,\bx_n$ using the modified GD at the $k$-th iteration. The idea of the proof is to establish a relationship between $\by-\bu_D(l)$ and $\by-\bu_D(l+1)$ for all $l=0,1,...$, so that we can obtain a relationship between $\bu_D(l+1)$ and $\bu_D(0)$. Based on this relationship, we can show that $\bu_D(l+1)$ is close to $\bH^\infty(C\mu I+\bH^\infty)^{-1}\by$, which is $\hat f$.

Consider event 
\begin{align*}
    A_{ir} = \{\exists \bw\in\RR^d: \|\bw -(1-\eta_2\mu)^k\bw_r(0)\|_2 \leq R, \mathbb{I}\{\bx_i^\top\bw_r(0)\geq 0\}\neq \mathbb{I}\{\bx_i^\top\bw\geq 0\} \},
\end{align*}
where $R$ will be determined later. Set $S_i = \{r\in [m]:\mathbb{I}\{A_{ir}\}=0\}$ and $S_i^\perp = [m]\backslash S_i$. Then $A_{ir}$ happens if and only if $|\bw_r(0)^\top \bx_i| < R/(1-\eta_2\mu)^k$. By concentration inequality of Gaussian, we have $\mathbb{P}(A_{ir}) = \mathbb{P}(|\bw_r(0)^\top \bx_i| < R/(1-\eta_2\mu)^k\leq \frac{2R}{\sqrt{2\pi}\tau(1-\eta_2\mu)^k}$. Thus, it follows the union bound inequality that with probability at least $1-\delta$ we have
\begin{align}\label{thmwithpsperp}
    \sum_{i=1}^n |S_i^\perp|\leq \frac{CmnR}{\delta(1-\eta_2\mu)^k},
\end{align}
where $C$ is a positive constant.

We first study the difference between two predictions $\bu_D(l+1)$ and $\bu_D(l)$. For any $i\in [n]$, we have
\begin{align}\label{thmwithpdecuD}
    u_{D,i}(l+1) - (1-\eta_2\mu)u_{D,i}(l) = & \frac{1}{\sqrt{m}}\sum_{r=1}^{m} a_r (\sigma(\bw_{D,r}(l+1)^\top \bx_i) - (1-\eta_2\mu)\sigma(\bw_{D,r}(l)^\top \bx_i))\nonumber\\
    = & \frac{1}{\sqrt{m}}\sum_{r\in S_i^\perp} a_r (\sigma(\bw_{D,r}(l+1)^\top \bx_i) - (1-\eta_2\mu)\sigma(\bw_{D,r}(l)^\top \bx_i))\nonumber\\
    & + 
    \frac{1}{\sqrt{m}}\sum_{r\in S_i} a_r (\sigma(\bw_{D,r}(l+1)^\top \bx_i) - (1-\eta_2\mu)\sigma(\bw_{D,r}(l)^\top \bx_i))\nonumber\\
    = & I_{1,i}(l) + I_{2,i}(l).
\end{align}
The first term $I_{1,i}(l)$ can be bounded by
\begin{align}\label{thmwithpdecuDI1}
    I_{1,i}(l) = & \frac{1}{\sqrt{m}}\sum_{r\in S_i^\perp} a_r (\sigma(\bw_{D,r}(l+1)^\top \bx_i) - (1-\eta_2\mu)\sigma(\bw_{D,r}(l)^\top \bx_i))\nonumber\\
    \leq & \frac{1}{\sqrt{m}}\sum_{r\in S_i^\perp}\left|(\bw_{D,r}(l+1)  -  (1-\eta_2\mu)\bw_{D,r}(l))^\top\bx_i\right|\nonumber\\
    \leq & \frac{1}{\sqrt{m}}\sum_{r\in S_i^\perp}\|\bw_{D,r}(l+1)  -  (1-\eta_2\mu)\bw_{D,r}(l)\|_2\nonumber\\
    = & \frac{1}{\sqrt{m}}\sum_{r\in S_i^\perp}\|\frac{\eta_1}{\sqrt{m}}a_r\sum_{j=1}^n(u_{D,j}(l) - y_j)\mathbb{I}_{r,j}(l)\bx_j\|_2\nonumber\\
    \leq & \frac{\eta_1}{m}\sum_{r\in S_i^\perp}\sum_{j=1}^n |u_{D,j}(l) - y_j|\nonumber\\
    \leq & \frac{\eta_1\sqrt{n}|S_i^\perp|}{m}\|\bu_{D}(l)-\by\|_2.
\end{align}
In \eqref{thmwithpdecuDI1}, the second and the last inequalities are by the Cauchy-Schwarz inequality.
% I think we can have a better estimation of $\mathbb{I}_{r,j}(k)$. Because $\bx_j$ is kind of spread out. Therefore, the number of $\mathbb{I}_{r,j}(k) = 1$ should be smaller. 
The second term $I_{2,i}(l)$ can be bounded by
\begin{align}\label{thmwithpdecuDI2}
    I_{2,i}(l) = & \frac{1}{\sqrt{m}}\sum_{r\in S_i} a_r (\sigma(\bw_{D,r}(l+1)^\top \bx_i) - (1-\eta_2\mu)\sigma(\bw_{D,r}(l)^\top \bx_i))\nonumber\\
    = &  \frac{1}{\sqrt{m}}\sum_{r\in S_i} a_r\mathbb{I}_{r,i}(l) (\bw_{D,r}(l+1) - (1-\eta_2\mu)\bw_{D,r}(l))^\top \bx_i\nonumber\\
    = & -\frac{1}{\sqrt{m}}\sum_{r\in S_i} a_r\mathbb{I}_{r,i}(l) \left(\frac{\eta_1}{\sqrt{m}}a_r\sum_{j=1}^n(u_{D,j}(l) - y_j)\mathbb{I}_{r,j}(l)\bx_j  \right)^\top \bx_i\nonumber\\
    = & -\frac{\eta_1}{m}\sum_{j=1}^n(u_{D,j}(l) - y_j)\bx_j^\top \bx_i\sum_{r\in S_i}\mathbb{I}_{r,i}(l)\mathbb{I}_{r,j}(l)\nonumber\\
    = & -\eta_1\sum_{j=1}^n(u_{D,j}(l) - y_j)\bH_{ij}(l) + I_{3,i}(l),
\end{align}
where
\begin{align*}
    I_{3,i}(l) = \frac{\eta_1}{m}\sum_{j=1}^n(u_{D,j}(l) - y_j)\bx_j^\top \bx_i\sum_{r\in S_i^\perp}\mathbb{I}_{r,i}(l)\mathbb{I}_{r,j}(l).
\end{align*}
The term $I_{3,i}(l)$ in \eqref{thmwithpdecuDI2} can be bounded by
\begin{align}\label{thmwithpdecuDI3}
    |I_{3,i}(l)|\leq & \left|\frac{\eta_1}{m}\sum_{j=1}^n(u_{D,j}(l) - y_j)\bx_j^\top \bx_i\sum_{r\in S_i^\perp}\mathbb{I}_{r,i}(l)\mathbb{I}_{r,j}(l)\right|\nonumber\\
    \leq & \frac{\eta_1}{m}|S_i^\perp|\sum_{j=1}^n |u_{D,j}(l) - y_j|\nonumber\\
    \leq & \frac{\eta_1\sqrt{n}|S_i^\perp|}{m}\|\bu_{D}(l)-\by\|_2.
\end{align}
Plugging \eqref{thmwithpdecuDI1} and \eqref{thmwithpdecuDI2} into \eqref{thmwithpdecuD}, we have 
\begin{align*}
    u_{D,i}(l+1) - (1-\eta_2\mu)u_{D,i}(l) = -\eta_1\sum_{j=1}^n(u_{D,j}(l) - y_j)\bH_{ij}(l) + I_{1,i}(l) + I_{3,i}(l),
\end{align*}
which leads to
\begin{align}\label{thmwithpdecuDquan}
    \bu_D(l+1) -  (1-\eta_2\mu)\bu_D(l) = -\eta_1\bH(l)(\bu_D(l) - \by) + \bI(l),
\end{align}
where $\bI(l)=(I_{1,1}(l) + I_{3,1}(l),...,I_{1,n}(l) + I_{3,n}(l))^\top$. By the triangle inequality, we have 
\begin{align}\label{thmwithpdecuD2}
    \|\bu_D(l+1) -  (1-\eta_2\mu)\bu_D(l)\|_2 \leq & \|\eta_1\bH(l)(\bu_D(l) - \by)\|_2 + \|\bI(l)\|_2.
\end{align}
By \eqref{thmwithpsperp}, \eqref{thmwithpdecuDI1}, and \eqref{thmwithpdecuDI3}, the term $\|\bI(l)\|_2$ in \eqref{thmwithpdecuD2} can be bounded by
\begin{align}\label{thmwithpdecuD2term2}
    \|\bI(l)\|_2 \leq & \sum_{i=1}^n |I_{3,i}(l)| + |I_{1,i}(l)| \leq \sum_{i=1}^n\frac{2\eta_1\sqrt{n}|S_i^\perp|}{m}\|\bu_D(l) - \by\|_2\nonumber\\
    \leq & \frac{2\eta_1\sqrt{n}}{m}\frac{CmnR}{\delta(1-\eta_2\mu)^k}\|\bu_D(l) - \by\|_2 = \frac{2C\eta_1n^{3/2}R}{\delta(1-\eta_2\mu)^k}\|\bu_D(l) - \by\|_2.
\end{align}
Gershgorin's theorem \citep{varga} implies
\begin{align*}
\lambda_{\max}(H(l))\leq \max_{j}\sum_{i=1}^nH_{ij}(l)\leq n.
\end{align*}
Therefore, the term $\|\eta_1\bH(l)(\bu_D(l) - \by)\|_2$ in \eqref{thmwithpdecuD2} can be bounded by
\begin{align}\label{thmwithpdecuD2term1}
    \|\eta_1\bH(l)(\bu_D(l) - \by)\|_2 \leq \eta_1\lambda_{\max}(H(l))\|\bu_D(l) - \by\|_2\leq \eta_1n \|\bu_D(l) - \by\|_2.
\end{align}
By \eqref{thmwithpdecuD2} and \eqref{thmwithpdecuD2term2}, $\|\by-\bu_D(l+1)\|_2$ can be bounded by
\begin{align}\label{thmwithpymDk}
    \|\by-\bu_D(l+1)\|_2^2 = & \|\by-(1-\eta_2\mu)\bu_D(l)\|_2^2 - 2(\by-(1-\eta_2\mu)\bu_D(l))^\top(\bu_D(l+1) - (1-\eta_2\mu)\bu_D(l))\nonumber\\
    & + \|\bu_D(l+1) - (1-\eta_2\mu)\bu_D(l)\|_2^2\nonumber\\
    = & \|\by-(1-\eta_2\mu)\bu_D(l)\|_2^2 + 2\eta_1(\by-(1-\eta_2\mu)\bu_D(l))^\top\bH(l)(\bu_D(l) - \by)\nonumber\\
    & - 2\eta_1(\by-(1-\eta_2\mu)\bu_D(l))^\top\bI(l) + \|\bu_D(l+1) - (1-\eta_2\mu)\bu_D(l)\|_2^2\nonumber\\
    = & T_1 + T_2 + T_3 + T_4.
\end{align}
The first term $T_1$ can be bounded by
\begin{align}\label{thmwithpymDkT1}
    T_1 = & \|\by-(1-\eta_2\mu)\bu_D(l)\|_2^2 \nonumber\\
    = & \eta_2^2\mu^2\|\by\|_2^2 + (1-\eta_2\mu)^2\|\by-\bu_D(l)\|_2^2 + 2\eta_2\mu(1-\eta_2\mu)\by^\top(\by-\bu_D(l)) \nonumber\\
    \leq &  (\eta_2^2\mu^2 + \eta_2\mu)\|\by\|_2^2 + (1 + \eta_2\mu)(1-\eta_2\mu)^2\|\by-\bu_D(l)\|_2^2.
\end{align}
The second term $T_2$ can be bounded by
\begin{align}\label{thmwithpymDkT2}
    T_2 = & 2\eta_1(\by-(1-\eta_2\mu)\bu_D(l))^\top\bH(l)(\bu_D(l) - \by) \nonumber\\
    = & 2\eta_1(1-\eta_2\mu)(\by-\bu_D(l))^\top\bH(l)(\bu_D(l) - \by) + 2\eta_1\eta_2\mu\by^\top\bH(l)(\bu_D(l) - \by)\nonumber\\
    = & -2\eta_1(1-\eta_2\mu)(\by-\bu_D(l))^\top\bH(l)(\by-\bu_D(l)) + 2\eta_1\eta_2\mu\by^\top\bH(l)(\bu_D(l) - \by)\nonumber\\
    \leq & 4\eta_1\eta_2\mu n\|\by\|_2^2 +4\eta_1\eta_2\mu n\|\bu_D(l) - \by\|_2^2.
    % \leq & -2\eta_1(1-\eta_2\mu)\lambda_0\|\by-\bu_D(l)\|_2^2 + 4\eta_1\eta_2\mu n\|\by\|_2^2 +4\eta_1\eta_2\mu n\|\bu_D(l) - \by\|_2^2.
\end{align}
Using \eqref{thmwithpdecuD2term2}, the third term $T_3$ can be bounded by
\begin{align}\label{thmwithpymDkT3}
    T_3 = & - 2\eta_1(\by-(1-\eta_2\mu)\bu_D(l))^\top\bI(l) \nonumber\\
    = & - 2\eta_1(1-\eta_2\mu)(\by-\bu_D(l))^\top\bI(l) + 2\eta_1\eta_2\mu \by^\top\bI(l)\nonumber\\
    \leq & 2\eta_1(1-\eta_2\mu)\frac{2C\eta_1n^{3/2}R}{\delta(1-\eta_2\mu)^k}\|\bu_D(l) - \by\|_2 + 4\eta_1\eta_2\mu\|\by\|_2^2 + 4\eta_1\eta_2\mu\|\bI(l)\|_2^2\nonumber\\
    \leq & 2\eta_1(1-\eta_2\mu)\frac{2C\eta_1n^{3/2}R}{\delta(1-\eta_2\mu)^k}\|\bu_D(l) - \by\|_2^2 + 4\eta_1\eta_2\mu\|\by\|_2^2 + 4\eta_1\eta_2\mu\left(\frac{2C\eta_1n^{3/2}R}{\delta(1-\eta_2\mu)^k}\right)^2\|\bu_D(l) - \by\|_2^2.
\end{align}
The fourth term $T_4$ can be bounded by
\begin{align}\label{thmwithpymDkT4}
    T_4 = & \|\bu_D(l+1) - (1-\eta_2\mu)\bu_D(l)\|_2^2\nonumber\\
    \leq & 2\|\eta_1\bH(l)(\bu_D(l) - \by)\|_2^2 + 2\|\bI(l)\|_2^2 \nonumber\\
    \leq & 2\eta_1^2n^2\|\bu_D(l) - \by\|_2^2 + 2\left(\frac{2C\eta_1n^{3/2}R}{\delta(1-\eta_2\mu)^k}\right)^2\|\bu_D(l) - \by\|_2^2.
\end{align}
Plugging \eqref{thmwithpymDkT1} - \eqref{thmwithpymDkT4} into \eqref{thmwithpymDk}, we have
\begin{align}\label{thmwithpymD2}
    & \|\by-\bu_D(l+1)\|_2^2\nonumber\\
    \leq & (\eta_2^2\mu^2 + \eta_2\mu)\|\by\|_2^2 + (1 + \eta_2\mu)(1-\eta_2\mu)^2\|\by-\bu_D(l)\|_2^2 + 4\eta_1\eta_2\mu n\|\by\|_2^2 +4\eta_1\eta_2\mu n\|\bu_D(l) - \by\|_2^2\nonumber\\
    % & -2\eta_1(1-\eta_2\mu)\lambda_0\|\by-\bu_D(l)\|_2^2 + 4\eta_1\eta_2\mu n\|\by\|_2^2 +4\eta_1\eta_2\mu n\|\bu_D(l) - \by\|_2^2\nonumber\\
    & + 2\eta_1(1-\eta_2\mu)\frac{2C\eta_1n^{3/2}R}{\delta(1-\eta_2\mu)^k}\|\bu_D(l) - \by\|_2^2 + 4\eta_1\eta_2\mu\|\by\|_2^2 + 4\eta_1\eta_2\mu\left(\frac{2C\eta_1n^{3/2}R}{\delta(1-\eta_2\mu)^k}\right)^2\|\bu_D(l) - \by\|_2^2\nonumber\\
    & + 2\eta_1^2n^2\|\bu_D(l) - \by\|_2^2 + 2\left(\frac{2C\eta_1n^{3/2}R}{\delta(1-\eta_2\mu)^k}\right)^2\|\bu_D(l) - \by\|_2^2\nonumber\\
    % \leq & \bigg((\eta_2^2\mu^2 + \eta_2\mu)+ 4\eta_1\eta_2\mu n + 4\eta_1\eta_2\mu\bigg)\|\by\|_2^2\nonumber\\
    % + & \bigg( (1 + \eta_2\mu)(1-\eta_2\mu)^2-2\eta_1(1-\eta_2\mu)\lambda_0 +4\eta_1\eta_2\mu n+ 2\eta_1(1-\eta_2\mu)\frac{2C\eta_1n^{3/2}R}{\delta(1-\eta_2\mu)} \nonumber\\
    % & + 4\eta_1\eta_2\mu\left(\frac{2C\eta_1n^{3/2}R}{\delta(1-\eta_2\mu)}\right)^2 +  2\eta_1^2n^2  + 2\left(\frac{2C\eta_1n^{3/2}R}{\delta(1-\eta_2\mu)}\right)^2\bigg)\|\bu_D(l) - \by\|_2^2\nonumber\\
    = & a_1 \|\by\|_2^2 + a_2\|\bu_D(l) - \by\|_2^2,
\end{align}
where
\begin{align*}
    a_1 = & (\eta_2^2\mu^2 + \eta_2\mu)+ 4\eta_1\eta_2\mu n + 4\eta_1\eta_2\mu \leq 2 \eta_2\mu + 8\eta_1\eta_2\mu n,\\
    a_2 = & (1 + \eta_2\mu)(1-\eta_2\mu)^2 +4\eta_1\eta_2\mu n+ 2\eta_1(1-\eta_2\mu)\frac{2C\eta_1n^{3/2}R}{\delta(1-\eta_2\mu)^k} \nonumber\\
    & + 4\eta_1\eta_2\mu\left(\frac{2C\eta_1n^{3/2}R}{\delta(1-\eta_2\mu)^k}\right)^2 +  2\eta_1^2n^2  + 2\left(\frac{2C\eta_1n^{3/2}R}{\delta(1-\eta_2\mu)^k}\right)^2\\
    \leq & 1 - \rbr{\eta_2\mu - 4\eta_1\eta_2\mu n - 2 \eta_1\frac{2C\eta_1n^{3/2}R}{\delta(1-\eta_2\mu)^k} - 2\eta_1^2n^2}\\
    % = & (1 + \eta_2\mu)(1-\eta_2\mu)^2-2\eta_1(1-\eta_2\mu)\lambda_0 +4\eta_1\eta_2\mu n+ 2\eta_1(1-\eta_2\mu)\frac{2C\eta_1n^{3/2}R}{\delta(1-\eta_2\mu)} \nonumber\\
    % & + 4\eta_1\eta_2\mu\left(\frac{2C\eta_1n^{3/2}R}{\delta(1-\eta_2\mu)}\right)^2 +  2\eta_1^2n^2  + 2\left(\frac{2C\eta_1n^{3/2}R}{\delta(1-\eta_2\mu)}\right)^2\\
    % \leq & 1 - \eta_2^2\mu^2 - \eta_2\mu + \eta_2^3\mu^3-2\eta_1(1-\eta_2\mu)\lambda_0 +4\eta_1\eta_2\mu n+ 2\eta_1(1-\eta_2\mu)\frac{2C\eta_1n^{3/2}R}{\delta(1-\eta_2\mu)} \nonumber\\
    % & + 4\eta_1\eta_2\mu\left(\frac{2C\eta_1n^{3/2}R}{\delta(1-\eta_2\mu)}\right)^2 +  2\eta_1^2n^2  + 2\left(\frac{2C\eta_1n^{3/2}R}{\delta(1-\eta_2\mu)}\right)^2\\
    % \leq & 1 - \eta_2\mu - 2\eta_1(1-\eta_2\mu)\lambda_0 + 4\eta_1\eta_2\mu n + 2(1 + \eta_1\eta_2\mu + \eta_1)\frac{2C\eta_1n^{3/2}R}{\delta(1-\eta_2\mu)} + 2\eta_1^2n^2\\
    = & 1 - \nu_0.
\end{align*}
By the conditions imposed on $\eta_1, \eta_2, \mu, m$, the dominating terms in $a_1$ and $\nu_0$ are both $\eta_2\mu$. Thus
$a_1 = o(1/n)$, $\nu_0 = o(1/n)$ and $a_1/\nu_0 = O(1)$.
Using \eqref{thmwithpymD2} iteratively, we have  
\begin{align}\label{thmwithpymD3}
    \|\by-\bu_D(l+1)\|_2^2\leq & a_1 \|\by\|_2^2 + a_2\|\bu_D(l) - \by\|_2^2\nonumber\\
    \leq & ... \leq \sum_{i=0}^l(1-\nu_0)^i(a_1\|\by\|_2^2) + (1-\nu_0)^{l+1}\|\by-\bu_D(0)\|_2^2\\
    \leq & \frac{a_1\|\by\|_2^2}{\nu_0} + (1-\nu_0)^{l+1}\|\by-\bu_D(0)\|_2^2.
\end{align}

% {\bf Note we need $\eta_1 \sim \eta_2$ and $\eta_1n \rightarrow 0$.}

% \textbf{alk;jdfs;lkaweujfpiouajefso;ialjsdkf}

By the modified GD rule, we have
\begin{align*}
    \bw_{D,r}(l+1) - (1-\eta_2\mu)\bw_{D,r}(l) = & -\frac{\eta_1}{\sqrt{m}}a_r\sum_{j=1}^n(u_{D,j}(l) - y_j)\mathbb{I}_{r,j}(l)\bx_j,
\end{align*}
which implies
\begin{align}\label{thmwithpgdw}
    \|\bw_{D,r}(l+1) - (1-\eta_2\mu)\bw_{D,r}(l)\|_2 \leq & \frac{\eta_1\sqrt{n}}{\sqrt{m}}\|\bu_D(l) - \by\|_2 \leq \frac{C\eta_1n}{\sqrt{m}}
\end{align}
for some constant $C$. Using \eqref{thmwithpgdw} iteratively yields
\begin{align}\label{ineqxiaowjin}
    &\|\bw_{D,r}(l+1) - (1-\eta_2\mu)^{l+1}\bw_{D,r}(0)\|_2\nonumber \\
    \leq & \|\bw_{D,r}(l+1) - (1-\eta_2\mu)\bw_{D,r}(l)\|_2 + \|(1-\eta_2\mu)\bw_{D,r}(0)- (1-\eta_2\mu)^{l+1}\bw_{D,r}(l)\|_2\nonumber\\
    \leq & \frac{C\eta_1n}{\sqrt{m}} + (1-\eta_2\mu)\|\bw_{D,r}(l)- (1-\eta_2\mu)^{l}\bw_{D,r}(0)\|_2\nonumber\\
    \leq & ...\leq \sum_{i=0}^l(1-\eta_2\mu)^i \frac{C\eta_1n}{\sqrt{m}}\leq \frac{C\eta_1n}{\eta_2\mu\sqrt{m}}.
\end{align}
By similar approach as in the proof of Lemma C.2 of \cite{du2018gradient}, we can show that with probability at least $1-\delta$ with respect to random initialization,
\begin{align*}
    \|\bZ(l) -\bZ(0)\|_F^2 \leq \frac{2nR}{\sqrt{2\pi}\tau\delta(1-\eta_2\mu)^k} +\frac{n}{m} = O\left(\frac{\eta_1n^2}{(1-\eta_2\mu)^k\eta_2\mu\sqrt{m}\delta^{3/2}\tau}\right), \forall l\in [k],
\end{align*}
and
\begin{align*}
     \|\bH(l) -\bH(0)\|_F\leq \frac{4n^2R}{\sqrt{2\pi}\tau} + \frac{2n^2\delta}{m} = O\left(\frac{\eta_1n^3}{(1-\eta_2\mu)^k\eta_2\mu\sqrt{m}\delta^{3/2}\tau}\right),\forall l\in[k].
\end{align*}
By Lemma C.3 of \cite{du2018gradient}, we have with probability at least $1-\delta$ with respect to random initialization,
\begin{align}\label{thmwithpH0Hinfty}
    \|\bH(0) - \bH^{\infty}\|_F = O\left(\frac{n\sqrt{\log(n/\delta)}}{\sqrt{m}}\right).
\end{align}
By \eqref{thmwithpdecuDquan}, we have
\begin{align*}
    \bu_D(l+1) -  (1-\eta_2\mu)\bu_D(l) = & -\eta_1\bH(l)(\bu_D(l) - \by) + \bI(l)\nonumber\\
    = & -\eta_1\bH^\infty(\bu_D(l) - \by) +\bI(l) - \eta_1(\bH(l)-\bH^\infty)(\bu_D(l) - \by),
\end{align*}
which yields
\begin{align}\label{thmwithpudK1}
    \bu_D(l+1) - B = \left((1-\eta_2\mu)I - \eta_1\bH^\infty\right)(\bu_D(l) - B) + \bI(l) - \eta_1(\bH(l)-\bH^\infty)(\bu_D(l) - \by),
\end{align}
where
\begin{align}\label{thm51B}
    B = (\eta_2\mu I+\eta_1\bH^\infty)^{-1}\eta_1\bH^\infty\by = \eta_1\bH^\infty(\eta_2\mu I+\eta_1\bH^\infty)^{-1}\by.
\end{align}
Iteratively using \eqref{thmwithpudK1}, we have 
\begin{align}\label{thmwithputotal}
    \bu_D(l+1) - B = &\left((1-\eta_2\mu)I - \eta_1\bH^\infty\right)^{l+1}(\bu_D(0) - B)\nonumber\\
    & + \sum_{i=0}^l \left((1-\eta_2\mu)I - \eta_1\bH^\infty\right)^{i}(\bI(l-i) - \eta_1(\bH(l-i)-\bH^\infty)(\bu_D(l-i) - \by))\nonumber\\
    = & \left((1-\eta_2\mu)I - \eta_1\bH^\infty\right)^{l+1}(\bu_D(0) - B) + e_l,
\end{align}
where
\begin{align}\label{thmwithpudek}
    e_l = \sum_{i=0}^l \left((1-\eta_2\mu)I - \eta_1\bH^\infty\right)^{i}(\bI(l-i) - \eta_1(\bH(l-i)-\bH^\infty)(\bu_D(l-i) - \by)).
\end{align}
The term $e_l$ can be bounded by
\begin{align}\label{thmwithpudekb}
    \|e_l\|_2 = & \|\sum_{i=0}^l \left((1-\eta_2\mu)I - \eta_1\bH^\infty\right)^{i}(\bI(l-i) - \eta_1(\bH(l-i)-\bH^\infty)(\bu_D(l-i) - \by))\|_2\nonumber\\
    \leq & \sum_{i=0}^l\|(1-\eta_2\mu)I - \eta_1\bH^\infty\|_2^i(\|\bI(l-i)\|_2 + \eta_1\|\bH(l-i)-\bH^\infty\|_2\|\bu_D(l-i) - \by\|_2)\nonumber\\
    \leq & \sum_{i=0}^l (1-\eta_2\mu)^iO\bigg(\frac{2C\eta_1^2n^{5/2}}{\eta_2\mu\sqrt{m}\delta^{3/2}(1-\eta_2\mu)^k}+ \frac{\eta_1^2n^{7/2}}{(1-\eta_2\mu)^k\eta_2\mu\sqrt{m}\delta^2\tau} \bigg)\nonumber\\
    = & O\bigg(\frac{\eta_1^2n^{7/2}}{\eta_2^2\mu^2\sqrt{m}\delta^{2}(1-\eta_2\mu)^k\tau}\bigg).
\end{align}
By \eqref{thmwithputotal} and taking $l= k-1$, with probability at least $1-\delta$ with respect to the random initialization, the difference $\bu_D(k) - B$ can be bounded by
\begin{align*}
    \|\bu_D(k) - B\|_2\leq & \|\left((1-\eta_2\mu)I - \eta_1\bH^\infty\right)^{k}(\bu_D(0) - B)\|_2 + \|e_k\|_2\nonumber\\
    = & O\left(\sqrt{n}(1-\eta_2\mu - \eta_1 \lambda_0)^k + \frac{n^{7/2}}{\mu^2\sqrt{m}\delta^{2}(1-\eta_2\mu)^k\tau}\right)\nonumber\\
    = & O\left(\sqrt{n}(1-\eta_2\mu)^k + \frac{n^{7/2}}{\mu^2\sqrt{m}\delta^{2}(1-\eta_2\mu)^k\tau}\right).
\end{align*}
This implies that 
\begin{align*}
     \|\bu_D(k) - B\|_2 = O_{\PP}\left(\sqrt{n}(1-\eta_2\mu)^k + \frac{n^{7/2}}{\mu^2\sqrt{m}(1-\eta_2\mu)^k\tau}\right).
\end{align*}
By choosing $m={\rm poly}(n,1/\tau,1/\lambda_0)$ such that  $\frac{n^{7/2}}{\mu^2\sqrt{m}(1-\eta_2\mu)^k\tau}\leq \sqrt{n}(1-\eta_2\mu)^k$, we finish the proof of \eqref{thmwithpstate1}.

Now consider ${\rm vec}(\bW_D(l+1))$. Direct calculation shows that
\begin{align}\label{thmwithWD}
    {\rm vec}(\bW_D(l+1)) = & (1-\eta_2\mu){\rm vec}(\bW_D(l)) - \eta_1 \bZ(l)(\bu_D(l)-\by)\nonumber\\
    = & (1-\eta_2\mu){\rm vec}(\bW_D(l))  - \eta_1 \bZ(0)(\bu_D(l)-\by) - \eta_1 (\bZ(l)-\bZ(0))(\bu_D(l)-\by)\nonumber\\
    = & (1-\eta_2\mu)^{l+1}{\rm vec}(\bW_D(0)) - \eta_1\bZ(0)\sum_{i=0}^l (1-\eta_2\mu)^{i}(\bu_D(l-i)-\by) \nonumber\\
    & - \sum_{i=0}^l (1-\eta_2\mu)^{i}\eta_1 (\bZ(l)-\bZ(0))(\bu_D(l)-\by).
\end{align}
Plugging
\begin{align*}
    \bu_D(l+1) = \left((1-\eta_2\mu)I - \eta_1\bH^\infty\right)^{l+1}(\bu_D(0) - B) + e_l + B
\end{align*}
into \eqref{thmwithWD}, we have
\begin{align}\label{thmwithWD1}
    & {\rm vec}(\bW_D(l+1)) - (1-\eta_2\mu)^{l+1}{\rm vec}(\bW_D(0))\nonumber\\
    = &  - \eta_1\bZ(0)\sum_{i=0}^l (1-\eta_2\mu)^{i}\left((1-\eta_2\mu)I - \eta_1\bH^\infty\right)^{l-i}(\bu_D(0) - B)\nonumber\\
    & - \eta_1\bZ(0)\sum_{i=0}^l  (1-\eta_2\mu)^{i}(e_{l-i-1} + B-\by) - \sum_{i=0}^l (1-\eta_2\mu)^{i}\eta_1 (\bZ(l)-\bZ(0))(\bu_D(l)-\by)\nonumber\\
    = & \eta_1\bZ(0)\sum_{i=0}^l (1-\eta_2\mu)^{i}\left((1-\eta_2\mu)I - \eta_1\bH^\infty\right)^{l-i} \eta_1\bH^\infty(\eta_2\mu I+\eta_1\bH^\infty)^{-1}\by\nonumber\\
    & - \eta_1\bZ(0)\sum_{i=0}^l (1-\eta_2\mu)^{i}\left((1-\eta_2\mu)I - \eta_1\bH^\infty\right)^{l-i}\bu_D(0)\nonumber\\
    & - \eta_1\bZ(0)\sum_{i=0\nonumber}^l  (1-\eta_2\mu)^{i}e_{l-i-1} - \eta_1\bZ(0)\sum_{i=0}^l  (1-\eta_2\mu)^{i}(B-\by)\\
    & - \sum_{i=0}^l (1-\eta_2\mu)^{i}\eta_1 (\bZ(l)-\bZ(0))(\bu_D(l)-\by)\nonumber\\
    = & E_1 - E_2 + E_3 - T_5 - E_4.
\end{align}
Let
\begin{align}\label{thmwithTk}
    \bT_l = & \sum_{i=0}^l (1-\eta_2\mu)^{i}\left((1-\eta_2\mu)I - \eta_1\bH^\infty\right)^{l-i}\nonumber\\
    = & (1-\eta_2\mu)^{l}\sum_{i=0}^l \left(I - \frac{\eta_1}{(1-\eta_2\mu)}\bH^\infty\right)^{i}
\end{align}
and
\begin{align}\label{thmwitha1}
    \ba_1 = & \eta_1\bH^\infty(\eta_2\mu I+\eta_1\bH^\infty)^{-1}\by.
\end{align}
The first term $E_1$ can be bounded by
\begin{align}\label{thmwithWDE1}
    \|E_1\|_2^2  = & \|\eta_1\bZ(0)\bT_l \ba_1\|_2^2\nonumber\\
    = & \eta_1^2\ba_1^\top \bT_l \bZ(0)^\top \bZ(0)\bT_l \ba_1\nonumber\\
    = & \eta_1^2\ba_1^\top \bT_l \bH^{\infty}\bT_l \ba_1 + \eta_1^2\ba_1^\top \bT_l (\bH(0)-\bH^{\infty})\bT_l \ba_1\nonumber\\
    = & \eta_1^2\ba_1^\top \bT_l \bH^{\infty}\bT_l \ba_1 + \eta_1^2O\left(\frac{n\sqrt{\log(n/\delta)}}{\sqrt{m}}\right)\ba_1^\top \bT_l^2\ba_1.
\end{align}
By \eqref{thmwithTk}, we have
\begin{align*}
     \bT_l = & (1-\eta_2\mu)^{l}\sum_{j=1}^n \frac{1 - (1 - \frac{\eta_1}{(1-\eta_2\mu)}\lambda_j)^{l+1}}{\frac{\eta_1}{(1-\eta_2\mu)}\lambda_j}\bv_j\bv_j^\top \preceq \frac{(1-\eta_2\mu)^{l}}{\eta_1\lambda_0} \bI,
\end{align*}
and
\begin{align*}
    \bT_l\bH^{\infty}\bT_l = & (1-\eta_2\mu)^{2l}\sum_{j=1}^n \left(\frac{1 - (1 - \frac{\eta_1}{(1-\eta_2\mu)}\lambda_j)^{2l+2}}{\frac{\eta_1}{(1-\eta_2\mu)}\lambda_j}\right)^2\lambda_j\bv_j\bv_j^\top \preceq \frac{(1-\eta_2\mu)^{l+1}}{\eta_1^2}(\bH^{\infty})^{-1}.
\end{align*}
Therefore,
\begin{align*}
    \eta_1^2\ba_1^\top \bT_l \bH^{\infty}\bT_l \ba_1 \leq & (1-\eta_2\mu)^{2l+2}\ba_1^\top(\bH^{\infty})^{-1}\ba_1,\\
    \eta_1^2O\left(\frac{n\sqrt{\log(n/\delta)}}{\sqrt{m}}\right)\ba_1^\top \bT_l^2\ba_1 \leq & O\left(\frac{n^2(1-\eta_2\mu)^{2l}\sqrt{\log(n/\delta)}}{\sqrt{m}\lambda_0^2}\right).
\end{align*}
Together with \eqref{thmwithWDE1}, we have
\begin{align}\label{thm51E1}
    \|E_1\|_2^2  = (1-\eta_2\mu)^{2l+2}\ba_1^\top(\bH^{\infty})^{-1}\ba_1 + O\left(\frac{n^2(1-\eta_2\mu)^{2l}\sqrt{\log(n/\delta)}}{\sqrt{m}\lambda_0^2}\right).
\end{align}
By similar approach, the second term $E_2$ can be bounded by
\begin{align}\label{thmwithWDE2}
    \|E_2\|_2^2 = & \|\eta_1\bZ(0)\sum_{i=0}^l (1-\eta_2\mu)^{i}\left((1-\eta_2\mu)I - \eta_1\bH^\infty\right)^{l-i}\bu_D(0)\|_2^2\nonumber\\
    = & \eta_1^2\bu_D(0)^\top \bT_1(l) \bZ(0)^\top \bZ(0)\bT_1(l) \bu_D(0)\nonumber\\
    = & \eta_1^2\bu_D(0)^\top \bT_1(l) \bH^{\infty}\bT_1(l) \bu_D(0) + \eta_1^2\bu_D(0)^\top \bT_1(l) (\bH(0)-\bH^{\infty})\bT_1(l) \bu_D(0)\nonumber\\
    = & (1-\eta_2\mu)^{2l+2}\bu_D(0)^\top(\bH^{\infty})^{-1}\bu_D(0) + O\left(\frac{n^2(1-\eta_2\mu)^{2l}\sqrt{\log(n/\delta)}}{\sqrt{m}\lambda_0^2}\right).
\end{align}
By \eqref{thmwithpudekb}, the third term $E_3$ can be bounded by
\begin{align}\label{thmwithWDE3}
    \|E_3\|_2^2 = &\|\eta_1\bZ(0)\sum_{i=0\nonumber}^l  (1-\eta_2\mu)^{i}e_{l-i-1}\|_2^2\nonumber\\
    = & \eta_1^2\left(\sum_{i=0}^l  (1-\eta_2\mu)^{i}e_{l-i-1}\right)^\top\bH(0)\left(\sum_{i=0}^l  (1-\eta_2\mu)^{i}e_{l-i-1}\right)\nonumber\\
    = & O\bigg(\frac{\eta_1^6n^8}{\eta_2^6\mu^6m\delta^{4}(1-\eta_2\mu)^{2k}\tau^2}\bigg).
\end{align}
The fourth term $E_4$ can be bounded by 
\begin{align}\label{thmwithWDE4}
    \|E_4\|_2^2 = & \|\sum_{i=0}^l (1-\eta_2\mu)^{i}\eta_1 (\bZ(l)-\bZ(0))(\bu_D(l)-\by)\|_2^2\nonumber\\
    = & O\left(\frac{\eta_1^3n^3}{(1-\eta_2\mu)^{k}\eta_2^3\mu^3\sqrt{m}\delta^{3/2}\tau}\right).
\end{align}
Note that
\begin{align*}
    B-\by = & \eta_1\bH^\infty(\eta_2\mu I+\eta_1\bH^\infty)^{-1}\by - \by\\
    = & (\eta_1\bH^\infty - \eta_2\mu I -\eta_1\bH^\infty)(\eta_2\mu I+\eta_1\bH^\infty)^{-1}\by\\
    = & -\eta_2\mu (\eta_2\mu I+\eta_1\bH^\infty)^{-1}\by.
\end{align*}
Therefore, the remaining term $T_5$ can be bounded by
\begin{align*}
    \|T_5\|_2^2= & \|\eta_1\bZ(0)\sum_{i=0}^l  (1-\eta_2\mu)^{i}(B-\by)\|_2^2 \nonumber\\
    \leq & \eta_1^2 \by^\top(\eta_2\mu I+\eta_1\bH^\infty)^{-1}\bH^\infty(\eta_2\mu I+\eta_1\bH^\infty)^{-1}\by\nonumber\\
    \leq  &\by^\top(\eta_2\mu/\eta_1 I+\bH^\infty)^{-1}\bH^\infty(\eta_2\mu/\eta_1 I +\bH^\infty)^{-1}\by.
\end{align*}
By the assumption that $\eta_2 \asymp \eta_1$, the term $T_5$ can be further bounded by
\begin{align}\label{thmwithWDT5}
    \|T_5\|_2^2 \leq & \by^\top(C\mu I+\bH^\infty)^{-1}\bH^\infty(C\mu I + \bH^\infty)^{-1}\by.
\end{align}
The right-hand side of \eqref{thmwithWDT5} is $\|\hat f\|_{\mathcal{N}}^2$, where $\hat f$ is defined in \eqref{eqsolukrr}. The term $\|\hat f\|_{\mathcal{N}}^2$ can be bounded by some constant as in Theorem \ref{thmkrr}. This also implies
\begin{align}\label{thm51E1term2}
    \ba_1^\top(\bH^{\infty})^{-1}\ba_1 = \eta_1^2\by^\top(\eta_2\mu I+\eta_1\bH^\infty)^{-1}\bH^\infty(\eta_2\mu I+\eta_1\bH^\infty)^{-1}\by = O(1).
\end{align}
Note also that 
\begin{align}\label{thm51E2term2}
    \bu_D(0)^\top(\bH^{\infty})^{-1}\bu_D(0) = O\rbr{\frac{n\tau^2}{\lambda_0}}.
\end{align}
By the assumptions of Theorem \ref{thm:withPenalty}, plugging \eqref{thmwithWDE1}-\eqref{thm51E2term2} into \eqref{thmwithWD1}, and taking the iteration number at $k$, we can conclude that
\begin{align}\label{thmwithpWsum}
    & \|{\rm vec}(\bW_D(k)) - (1-\eta_2\mu)^{k}{\rm vec}(\bW_D(0))\|_2^2\nonumber\\
    = & O((1-\eta_2\mu)^{2k}) + O\left(\frac{n^2(1-\eta_2\mu)^{2k-2}\sqrt{\log(n/\delta)}}{\sqrt{m}\lambda_0^2}\right)\nonumber\\
    & + O\left(\frac{n\tau^2}{\lambda_0}(1-\eta_2\mu)^{2k}\right) + O\left(\frac{n^2(1-\eta_2\mu)^{2k-2}\sqrt{\log(n/\delta)}}{\sqrt{m}\lambda_0^2}\right)\nonumber\\
    & + O\bigg(\frac{n^8}{\mu^6m\delta^{4}(1-\eta_2\mu)^{2k}\tau^2}\bigg) + O\left(\frac{n^3}{(1-\eta_2\mu)^{k}\mu^3\sqrt{m}\delta^{3/2}\tau}\right) +O(1)\nonumber\\
    = & O(1),
\end{align}
where the last equality is because we can select some polynomials such that all the terms in \eqref{thmwithpWsum} except the $O(1)$ term converge to zero, and $\exp(-2\eta_2\mu k)\leq (1-\eta_2\mu)^{k}\leq \exp(-\eta_2\mu k)$ for sufficiently large $n$. This finishes the proof of \eqref{thmwithpstate2} in Theorem \ref{thm:withPenalty}.

\subsection{Proof of Theorem \ref{thm:withPenaltyl2small}}
For notational simplification, we use $\hat f_k = f_{\bW(k),\ba}$. Similar to the proof of Theorem \ref{thm:gd}, we define 
\begin{align}\label{tildef}
    \tilde{f}_k(\bx)={\rm vec}(\bW_D(k))^\top\bz_0(\bx),
\end{align}
where $\bz_0(\bx)=\bz(\bx)|_{\bW_D=\bW_D(0)}$.
Then we can write the following decomposition
\begin{align}\label{eq2Xinpf52}
    \hat{f}_k(\bx)-f^*(\bx) = & (\hat{f}_k(\bx)- \tilde{f}_k(\bx)) + (\tilde{f}_k(\bx) - \hat f(\bx)) + (\hat f(\bx) - f^*(\bx))\nonumber\\
    = & \Delta_1(\bx)+\Delta_2(\bx)+\Delta_3(\bx),
\end{align}
where $\hat f$ is as in \eqref{eqsolukrr}. In the rest of the proof, we show $\Delta_1(\bx)$, $\Delta_2(\bx)$, and $\Delta_3(\bx)$ are all small.

It follows from Theorem \ref{thmkrr} that
\begin{align}\label{Delta3b52}
    \|\Delta_3\|_{2}^2=O_{\PP}\rbr{n^{-\frac{d}{2d-1}}}.
\end{align} 
Next, we consider $\Delta_1$. From \eqref{ineqxiaowjin}, it can be seen that
\begin{align}\label{ineqxiaowjin52}
    &\|\bw_{D,r}(k) - (1-\eta_2\mu)^{k}\bw_{D,r}(0)\|_2\leq \frac{C\eta_1n}{\eta_2\mu\sqrt{m}}.
\end{align}
Define event
\begin{align*}
    B_{D,r}(\bx)=\{|(1-\eta_2\mu)^{k}\bw_{D,r}(0)^\top\bx|\leq R_1\}, \forall r\in [m],
\end{align*}
where $R_1 = \frac{C\eta_1n}{\eta_2\mu\sqrt{m}}$. If $\mathbb{I}\{B_{D,r}(\bx)\}=0$, then we have $\mathbb{I}_{r,k}(\bx)=\mathbb{I}_{r,0}(\bx)$, where $\mathbb{I}_{r,k}(\bx)=\mathbb{I}\{\bw_{D,r}(k)^\top\bx\ge 0\}$. 
Therefore, for any fixed $\bx$,
\begin{align*}
%   \Delta_1  &=  \\
     |\Delta_1(\bx)|&=   |\hat{f}_k(\bx)- \tilde{f}_k(\bx)|\\ &=\left|\frac{1}{\sqrt{m}}\sum_{r=1}^m a_r (\mathbb{I}_{r,k}(\bx)-\mathbb{I}_{r,0}(\bx))\bw_{D,r}(k)^\top\bx\right|\\
        &= \left|\frac{1}{\sqrt{m}}\sum_{r=1}^m a_r \mathbb{I}\{B_{D,r}(\bx)\} (\mathbb{I}_{r,k}(\bx)-\mathbb{I}_{r,0}(\bx))\bw_{D,r}(k)^\top\bx\right|\\
        &\le \frac{1}{\sqrt{m}}\sum_{r=1}^m \mathbb{I}\{B_{D,r}(\bx)\}|\bw_{D,r}(k)^\top\bx|\\
        &\le \frac{1}{\sqrt{m}} \sum_{r=1}^m \mathbb{I}\{B_{D,r}(\bx)\}\rbr{|(1-\eta_2\mu)^{k}\bw_{D,r}(0)^\top\bx|+ |\bw_{D,r}(k)^\top\bx-(1-\eta_2\mu)^{k}\bw_r(0)^\top\bx |}\\
        &\le  \frac{2R_1}{\sqrt{m}} \sum_{r=1}^m \mathbb{I}\{B_{D,r}(x)\}.
\end{align*}
Note that $\|\bx\|_2=1$, which implies that $\bw_{D,r}(0)^\top\bx$ is distributed as $N (0,\tau^2)$. Therefore, we have 
\begin{align*}
    \mathbb{E}[\mathbb{I}\{B_{D,r}(x)\}] &= \PP\rbr{|(1-\eta_2\mu)^{k}\bw_{D,r}(0)^\top\bx|\leq R_1}\\
    &=\int_{-R_1/(1-\eta_2\mu)^{k}}^{R_1/(1-\eta_2\mu)^{k}}\frac{1}{\sqrt{2\pi}\tau}\exp\left\{-\frac{u^2}{2\tau^2}\right\}du\le \frac{2R_1}{\sqrt{2\pi}(1-\eta_2\mu)^{k}\tau}.
\end{align*}
By Markov's inequality, with probability at least $1-\delta$, we have 
\begin{align*}
     \sum_{r=1}^m \mathbb{I}\{B_{D,r}(x)\} \le \frac{2mR_1}{\sqrt{2\pi}(1-\eta_2\mu)^{k}\tau\delta}.
\end{align*}
Thus, we have with probability at least $1-\delta$,
\begin{align*}
    \|\Delta_1\|_{2} \le  \frac{2R_1}{\sqrt{m}} \|\sum_{r=1}^m \mathbb{I}\{B_{D,r}(\cdot)\}\|_{2}\le \frac{4\sqrt{m}R_1^2}{\sqrt{2\pi}(1-\eta_2\mu)^{k}\tau\delta}= O\rbr{\frac{n^2}{\sqrt{m}\lambda_0^2\delta^2(1-\eta_2\mu)^{k}\tau}},
\end{align*}
which implies
\begin{align}\label{Delta1b52}
    \|\Delta_1\|_{2} = O_{\PP}\rbr{\frac{n^2}{\sqrt{m}\lambda_0^2(1-\eta_2\mu)^{k}\tau}}.
\end{align}
Now we bound $\Delta_2$. Note that
Define $\bG_k=\sum_{j=0}^{k-1} \eta(\bI - \eta \bH^{\infty})^j$. Recalling that $\by=\by^*+\bepsilon$, for fixed $\bx$, we have 
\begin{align}\label{Delta2b52}
   \Delta_2(\bx) = & \tilde{f}_k(\bx) - \hat f(\bx)\nonumber\\
   = & \bz_0(\bx)^\top{\rm vec}(\bW_D(k))-h(\bx,\bX)(\bH^\infty + \eta_2\mu/\eta_1 I)^{-1}\by \nonumber\\
   = & \bz_0(\bx)^\top E_1 - \bz_0(\bx)^\top E_2 + \bz_0(\bx)^\top E_3 - \bz_0(\bx)^\top T_5 - \bz_0(\bx)^\top E_4\nonumber\\
   & + (1-\eta_2\mu)^k\bz_0(\bx)^\top {\rm vec}(\bW_D(0)) - h(\bx,\bX)(\bH^\infty + \eta_2\mu/\eta_1 I)^{-1}\by,%\nonumber\\
    %  =& \bz_0(\bx)^\top \bigl[ \bZ(0)\bG_k(\by-\bu(0))+\zeta(k) +{\rm vec}(\bW(0))\bigr]\nonumber\\
    %   =& \bigl[ h(\bx,\bX)(\bG_k-(\bH^\infty)^{-1})\by^*+h(\bx,\bX)\bG_k\bepsilon \bigr]+ \bigl[\bz_0(\bx)^\top\bZ(0)-h(\bx,\bX)\bigr]\bG_k\by\nonumber\\
    %   & + \bigl[\bz_0(\bx)^\top{\rm vec}(\bW(0)) + \bz_0(\bx)^\top\zeta(k)-\bz_0(\bx)^\top\bZ(0)\bG_k\bu(0)\bigr]\nonumber\\
    %   =& \Delta_{21}(\bx)+\Delta_{22}(\bx)+\Delta_{23}(\bx).
\end{align}
where $E_1$, $E_2$, $E_3$, $T_5$, $E_4$ are as in \eqref{thmwithWD1}. Noting that $\|z_0(\bx)\|_2 = O_{\PP}(1)$, we have that
\begin{align}
    |\bz_0(\bx)^\top E_1|^2\leq \|\bz_0(\bx)\|_2^2 \|E_1\|_2^2 = & O_{\PP}((1-\eta_2\mu)^{2k}) + O_{\PP}\left(\frac{n^2(1-\eta_2\mu)^{2k-2}\sqrt{\log(n)}}{\sqrt{m}\lambda_0^2}\right),\label{yiduixiao52E1}\\
    |\bz_0(\bx)^\top E_2|^2\leq \|\bz_0(\bx)\|_2^2 \|E_2\|_2^2 = & O_{\PP}\left(\frac{n\tau^2}{\lambda_0}(1-\eta_2\mu)^{2k}\right) + O_{\PP}\left(\frac{n^2(1-\eta_2\mu)^{2k-2}\sqrt{\log(n)}}{\sqrt{m}\lambda_0^2}\right),\label{yiduixiao52E2}\\
    |\bz_0(\bx)^\top E_3|^2\leq \|\bz_0(\bx)\|_2^2 \|E_3\|_2^2 = & O_{\PP}\bigg(\frac{\eta_1^6n^8}{\eta_2^6\mu^6m(1-\eta_2\mu)^{2k}\tau^2}\bigg),\label{yiduixiao52E3}\\
     |\bz_0(\bx)^\top E_4|^2\leq \|\bz_0(\bx)\|_2^2 \|E_4\|_2^2 = & O_{\PP}\left(\frac{n^3}{(1-\eta_2\mu)^{k}\mu^3\sqrt{m}\delta^{3/2}\tau}\right),\label{yiduixiao52E4}
\end{align}
where \eqref{yiduixiao52E1} is because of \eqref{thm51E1} and \eqref{thm51E1term2}, \eqref{yiduixiao52E2}  is because of \eqref{thmwithWDE2} and \eqref{thm51E2term2},  \eqref{yiduixiao52E3}  is because of \eqref{thmwithWDE3}, and  \eqref{yiduixiao52E4}  is because of \eqref{thmwithWDE4}. By Lemma \ref{lem:bounds4} (d), the term  $(1-\eta_2\mu)^k\bz_0(\bx)^\top {\rm vec}(\bW_D(0))$ in \eqref{Delta2b52} can be bounded by
\begin{align}\label{yiduixiao52z0w0}
    \|(1-\eta_2\mu)^k\bz_0(\cdot)^\top {\rm vec}(\bW_D(0))\|_2 = O_{\PP}((1-\eta_2\mu)^k\tau).
\end{align}
Define 
\begin{align*}
    B = \eta_1\bH^\infty(\eta_2\mu I+\eta_1\bH^\infty)^{-1}\by.
\end{align*}
Note that
\begin{align*}
    B-\by = & \eta_1\bH^\infty(\eta_2\mu I+\eta_1\bH^\infty)^{-1}\by - \by\\
    = & (\eta_1\bH^\infty - \eta_2\mu I -\eta_1\bH^\infty)(\eta_2\mu I+\eta_1\bH^\infty)^{-1}\by\\
    = & -\eta_2\mu (\eta_2\mu I+\eta_1\bH^\infty)^{-1}\by.
\end{align*}
Therefore, the remaining term in \eqref{Delta2b52} $-\bz_0(\bx)^\top T_5 - h(\bx,\bX)(\bH^\infty + \eta_2\mu/\eta_1 I)^{-1}\by$ can be bounded by
\begin{align}\label{eq2XXinpf52}
& -\bz_0(\bx)^\top T_5 - h(\bx,\bX)(\bH^\infty + \eta_2\mu/\eta_1 I)^{-1}\by \nonumber\\
  =  & -\bz_0(\bx)^\top\bZ(0)\sum_{i=0}^{k-1}  \eta_1(1-\eta_2\mu)^{i}(B-\by) - h(\bx,\bX)(\bH^\infty + \eta_2\mu/\eta_1 I)^{-1}\by\nonumber\\
  = & -\bz_0(\bx)^\top\bZ(0) \eta_1\frac{1-(1-\eta_2\mu)^{k}}{\eta_2\mu}(B-\by) - h(\bx,\bX)(\bH^\infty + \eta_2\mu/\eta_1 I)^{-1}\by\nonumber\\
    = & \bz_0(\bx)^\top\bZ(0) \eta_1(1-(1-\eta_2\mu)^{k})(\eta_2\mu I+\eta_1\bH^\infty)^{-1}\by - h(\bx,\bX)(\bH^\infty + \eta_2\mu/\eta_1 I)^{-1}\by\nonumber\\
    = & (\bz_0(\bx)^\top\bZ(0) - h(\bx,\bX))(\bH^\infty + \eta_2\mu/\eta_1 I)^{-1}\by - \eta_1(1-\eta_2\mu)^{k}\bz_0(\bx)^\top\bZ(0)(\eta_2\mu I+\eta_1\bH^\infty)^{-1}\by.
\end{align}
The first term in \eqref{eq2XXinpf52} can be bounded by
\begin{align}\label{eq2XXinpf52term1}
    & \|(\bz_0(\cdot)^\top\bZ(0) - h(\cdot,\bX))(\bH^\infty + \eta_2\mu/\eta_1 I)^{-1}\by\|_2 \nonumber\\
    \leq & \|(\bz_0(\cdot)^\top\bZ(0) - h(\cdot,\bX))\|_2\|(\bH^\infty + \eta_2\mu/\eta_1 I)^{-1}\by\|_2\nonumber\\
   =   &O_{\PP}\rbr{\frac{n\sqrt{\log(n)}\eta_1}{\sqrt{m}\eta_2\mu}},
\end{align}
where we utilize
\begin{align*}
    \|(\bH^\infty + \eta_2\mu/\eta_1 I)^{-1}\by\|_2^2 = \by^\top(\bH^\infty + \eta_2\mu/\eta_1 I)^{-2}\by\leq \frac{\eta_1^2}{\eta_2^2\mu^2}\|\by\|_2^2 = O_{\PP}\bigg(\frac{\eta_1^2}{\eta_2^2\mu^2} n\bigg),
\end{align*}
and Lemma \ref{lem:bounds4} (c).

The second term in \eqref{eq2XXinpf52} can be bounded by
\begin{align}\label{eq2XXinpf52term2}
    & \|(1-\eta_2\mu)^{k}\bz_0(\cdot)^\top\bZ(0)(\bH^\infty + \eta_2\mu/\eta_1 I)^{-1}\by\|_2\nonumber\\
    \leq & (1-\eta_2\mu)^{k}\|(\bz_0(\cdot)^\top\bZ(0) - h(\cdot,\bX))(\bH^\infty + \eta_2\mu/\eta_1 I)^{-1}\by\|_2 \nonumber\\
    &+(1-\eta_2\mu)^{k}\|h(\cdot,\bX)(\bH^\infty + \eta_2\mu/\eta_1 I)^{-1}\by\|_2\nonumber\\
    \leq &O_{\PP}\rbr{\frac{n\sqrt{\log(n)}\eta_1}{\sqrt{m}\eta_2\mu}} + (1-\eta_2\mu)^{k}\|h(\cdot,\bX)(\bH^\infty + \eta_2\mu/\eta_1 I)^{-1}\by\|_{\mathcal{N}} \nonumber\\
    = & O_{\PP}((1-\eta_2\mu)^{k}),
\end{align}
where the second inequality is because of \eqref{eq2XXinpf52term1} and the last equality is because of Theorem \ref{thmkrr} and the assumption $\eta_1\asymp \eta_2$. Plugging \eqref{yiduixiao52E1}-\eqref{eq2XXinpf52term2} to \eqref{Delta2b52}, we can conclude that
\begin{align}\label{Delta2b52X}
    \|\Delta_2\|_2 = o_{\PP}(n^{-\frac{d}{2d-1}}),
\end{align}
by choosing $k$ and $m$ as in Theorem \ref{thm:withPenaltyl2small}. Combining \eqref{Delta1b52}, \eqref{Delta2b52X}, and \eqref{Delta3b52} finishes the proof.

% The proof of Theorem \ref{thm:withPenaltyl2small} then can be done by the following steps. First, we bound the $\sup_{f\in \mathcal{F}}\|f\|_\infty$. Second, we bound the Rademacher complexity.

\section{Proof of lemmas in the Appendix}
\label{sec:kernel_decay}
\subsection{Proof of Lemma \ref{lemmercerh}}
The proof of Lemma \ref{lemmercerh} mainly from Appendix C of \cite{bietti2019inductive} and Appendix D of \cite{bach2017breaking}, with some modification. 

We first review some background of spherical harmonic analysis \citep{atkinson2012spherical,costas2014spherical}. Let $Y_{k,j}$ be the spherical harmonics of degree $k$ on $\mathcal{S}^{d-1}$, where $N(p.k) = \frac{2k+d-2}{k}\left(\begin{array}{c}
     k+d-3\\
     d-2 
\end{array}\right)$. Then $Y_{k,j}$ is an orthonormal basis of $L_2(\mathcal{S}^{p-1},d\xi)$, where $d\xi$ is the uniform measure on the sphere. Then we have
\begin{align}\label{eqrelYkjandPk}
    \sum_{j=1}^{N(d,k)}Y_{k,j}(\bs)Y_{k,j}(\bt) = N(d,k)P_k(\bs^\top\bt),
\end{align}
where $P_k$ is the $k$-th Legendre polynomial in dimension $d$, given by
\begin{align}\label{eqLeg}
    P_k(t) = &  (-1/2)^k \frac{\Gamma(\frac{d-1}{2})}{\Gamma(k + \frac{d-1}{2})}(1-t^2)^{(3-d)/2}\left(\frac{d}{dt}\right)^k(1-t^2)^{k+(d-3)/2}.
\end{align}
The polynomials $P_k$ are orthogonal in $L_2([-1,1])d\nu$, where the measure $d\nu = (1-t^2)^{(d-3)/2}dt$ with Lebesgue measure $dt$, and 
\begin{align}\label{eqintPk}
    \int_{[-1,1]}P_k^2(t)(1-t^2)^{(d-3)/2}dt = \frac{w_{d-1}}{w_{d-2}}\frac{1}{N(d,k)},
\end{align}
where $w_{d-1} = \frac{2\pi^{d/2}}{\Gamma(d/2)}$. Furthermore, it can be shown that \citep{atkinson2012spherical}
\begin{align}\label{eqrecPk}
    tP_k(t) = \frac{k}{2k+d-2}P_{k-1}(t) + \frac{k+d-2}{2k+d-2}P_{k+1}(t),
\end{align}
for $k\geq 1$, and for $j=0$ we have $tP_0(t) = P_1(t)$. This implies that for large $k$ enough, we have
\begin{align*}
    \mu_k = \frac{k}{2k+d-2}\mu_{0,k-1} + \frac{k+d-2}{2k+d-2}\mu_{0,k+1},
\end{align*}
where $\mu_{0,k-1}$ and $\mu_{0,k+1}$ are as in Lemma 17 of \cite{bietti2019inductive}. By Lemma 17 of \cite{bietti2019inductive}, we have $\mu_{0,k} \asymp k^{-d}$ for large $k$, if $k = 1$ mod 2. This finish the proof of Lemma \ref{lemmercerh}.
\subsection{Proof of Lemma \ref{lementbound}}
% We first present some lemmas used in the proof of Lemma \ref{lementbound}.

% Now we are ready to prove Lemma \ref{lementbound}. 
By Theorem 1 of \cite{brauchart2013characterization} and Lemma \ref{lemmercerh}, we can see that the function space $\mathcal{N}$ is a subspace of the Sobolev space $H^s(\mathcal{S}^{d-1})$. Therefore, the entropy of $\mathcal{N}(1)$ can be bounded if the entropy of $H^{d/2}(\mathcal{S}^{d-1})(1)$ can be bounded. By Theorem 1.2 of \cite{wang2014entropy}, we have that the $k$-th entropy number $e_k(T)$ can be bounded by $k^{-d/(2(d-1))}$. This implies that 
\begin{align*}
    H(\delta,\mathcal{N}(1),\|\cdot\|_{L_\infty}) \leq A \delta^{-\frac{2(d-1)}{d}}.
\end{align*}

\subsection{Proof of Lemma \ref{lemwithoutp1}}
The first inequality follows the fact that $h$ is positive definite, which implies the inverse of
\begin{align*}
    \left(\begin{array}{cc}
     h(\bs,\bs)    & h(\bX,\bs) \\
       h(\bs,\bX)  & \bh^{\infty}
    \end{array}\right)
\end{align*}
is positive definite. By block matrix inverse, we have the first inequality in Lemma \ref{lemwithoutp1} holds. 

The second inequality and third inequality are direct results of Theorem \ref{thmkrr} implies
\begin{align*}
    &\mathbb{E}_{\epsilon,\bX}(\|\hat{g}_n - g^*\|_{2}^2)\\ = &\int_{\mathbb{S}^{d-1}}(g^*(\bx) - h(\bx,\bX)(\bH^\infty + \mu \bI)^{-1}\by^*)^2 + h(\bx,\bX)(\bH^\infty + \mu \bI)^{-2}h(\bX,\bx) d\bx = O_{\mathbb{P}}(n^{-\frac{d}{2d-1}})
\end{align*}
for any function $g^*$ with $\|g^*\|_{\mathcal{N}}\leq 1$. Then we have
\begin{align*}
    \int_{\mathbb{S}^{d-1}}h(\bx,\bX)(\bH^\infty + \mu \bI)^{-2}h(\bX,\bx)d\bx= O_{\mathbb{P}}(n^{-\frac{d}{2d-1}}),
\end{align*}
which finishes the proof of the second equality. Let $g^*(\bx) = h(\bs,\bx)$, then we have
\begin{align*}
    \int_{\mathbb{S}^{d-1}}(h(\bs,\bx) - h(\bx,\bX)(\bH^\infty + \mu \bI)^{-1}h(\bX,\bs))^2d\bx = O_{\mathbb{P}}(n^{-\frac{d}{2d-1}}).
\end{align*}
By the interpolation inequality, we have
\begin{align*}
    & h(\bs,\bs) - h(\bs,\bX)(\bH^\infty + \mu \bI)^{-1}h(\bX,\bs))\\
    \leq & \|h(\bs,\cdot) - h(\cdot,\bX)(\bH^\infty + \mu \bI)^{-1}h(\bX,\bs))\|_{\infty}\\
    \leq & C\|h(\bs,\cdot) - h(\cdot,\bX)(\bH^\infty + \mu \bI)^{-1}h(\bX,\bs))\|_{2}^{1-\frac{d-1}{d}}\|h(\bs,\cdot) - h(\cdot,\bX)(\bH^\infty + \mu \bI)^{-1}h(\bX,\bs)\|_{\mathcal{N}}^{\frac{d-1}{d}}\\
    = &  O_{\mathbb{P}}(n^{-\frac{1}{2d-1}})(h(\bs,\bs) +h(\bs,\bX)(\bH^\infty + \mu \bI)^{-1}\bH^\infty(\bH^\infty + \mu \bI)^{-1}h(\bX,\bs))^{\frac{d-1}{d}}\\
    \leq & O_{\mathbb{P}}(n^{-\frac{1}{2d-1}})(h(\bs,\bs) +h(\bs,\bX)(\bH^\infty)^{-1}h(\bX,\bs))^{\frac{d-1}{d}} = O_{\mathbb{P}}(n^{-\frac{1}{2d-1}}),
\end{align*}
where the last inequality follows the first inequality of Lemma \ref{lemwithoutp1}.

\subsection{Proof of Lemma \ref{lem:intpRKHS}}
Given that $g$ and $f^*$ have the same value at all $\bx_i$'s, the empirical norm $\|g-f^*\|_n=0$.
Notice that both $g$ and $f^*$ are in the RKHS generated by the NTK $h$, denoted by $\cN$. Utilizing Lemma \ref{lementbound} and \ref{lemma:geer} similarly as in the proof of Theorem \ref{thmkrr}, we have $R, K = O(1)$ and $J_\infty(z,\cN)\lesssim z^{1/d}$, which leads to 
\[\sup_{h \in {\cG(R)}} \biggl | \| h \|_n^2 - \| h \|_2^2 \biggr | =O_{\PP}\rbr{\sqrt{\frac{{1}}{n}}},\]
where $\cG(R):=\{g\in\cN(1): \|g-g^*\|_2\le R\}$.
Therefore, we can conclude that $\|g-f^*\|_2=O_{\PP}(n^{-1/2})$.

\subsection{Proof of Lemma \ref{lem:bounds4}}
The proof of (a) and (b) can be found in \cite{arora2019fine}.

For (c), the $i$-th coordinates of $\bz_0(\bx)^\top\bZ(0)$ and $h(\bx,\bX)$ are
\begin{align*}
    \frac{1}{m}\sum_{r=1}^m \bx^\top\bx_i\mathbb{I}\{\bw^\top_r(0)\bx\ge 0\}\mathbb{I}\{\bw^\top_r(0)\bx_i\ge 0\}, \text{~~and~~} \mathbb{E}_{\bw\sim N(0,\bI)}[\bx^\top\bx_i\mathbb{I}\{\bw^\top\bx\ge 0\}\mathbb{I}\{\bw^\top\bx_i\ge 0\}],
\end{align*}
respectively. $\forall i \in  [n]$, $(\bz_0(\bx)^\top\bZ(0))_i$ is the average of $m$ i.i.d. random variables, which have expectation $h_i(\bx,\bX)$ and bounded in $[0,1]$. For any fixed $\bx$, by Hoeffding's inequality, with probability at least $1-\delta^*$,
\begin{align*}
    |(\bz_0(\bx)^\top\bZ(0))_i- h_i(\bx,\bX)| \le \sqrt{\frac{\log(2/\delta^*)}{2m}}
\end{align*}
holds. By defining $\delta=n\delta^*$ and applying a union bound over all $i\in [n]$, with probability at least $1-\delta$, we have
\begin{align*}
    \|\bz_0(\bx)^\top\bZ(0)-h(\bx,\bX)\|_2^2 = O\rbr{n\frac{\log(2n/\delta)}{2m}}
\end{align*}

For (d), since
\begin{align*}
    \bz_0(\bx)^{\top}{\rm vec}(\bW(0)) = \frac{1}{\sqrt{m}}\sum_{r=1}^m a_r \mathbb{I}\{\bw_r(0)^{\top}\bx\ge 0\}\bw_r(0)^{\top}\bx
\end{align*}
Define random variables $V_r$, $r\in [m]$ as
\begin{align*}
    V_r = a_r \mathbb{I}\{\bw_r(0)^{\top}\bx\ge 0\}\bw_r(0)^{\top}\bx
\end{align*}
Since 
\begin{align*}
    \bw_r(0)^{\top}\bx~\sim ~ N(0, \tau^2) \text{~~~~and~~~~} a_r~\sim~ {\rm unif}\{1,-1\}.
\end{align*}
It's easy to prove that $V_r$, $r\in [m]$ are i.i.d. with mean $0$ and sub-Gaussian parameter $\tau$. By Hoeffding's inequality, at fixed $\
bx$, with probability at least $1-\delta$, we have 
\begin{align*}
    \bigg| \frac{1}{\sqrt{m}}\sum_{r=1}^m V_r \bigg|  \le \sqrt{2}\tau \sqrt{\log (2/\delta)}.
\end{align*}
Thus $\|\bz_0(\cdot)^{\top}{\rm vec}(\bW(0))\|_2 = O\rbr{\tau \sqrt{\log (1/\delta)}}$.
% Since $\bZ(0)=(\bz_0(\bx_1),\cdots,\bz_0(\bx_n))^{\top}$, using union bound over $i\in [n]$, we can obtain
% \begin{align*}
%     \|\bZ(0)^{\top}{\rm vec}(\bW(0))\|_2 = O\rbr{\sqrt{n}\tau \sqrt{\log (n/\delta)}}
% \end{align*}

\section{More details and results for numerical experiments}
\label{sec:numerical}
\paragraph{Neural network setup} The neural network used in all experiments is a 2-layer ReLU neural network with $m = 500$ nodes in each hidden layer. All the weighs are initialized with the Glorot uniform initializer, also called as Xavier uniform initializer \citep{glorot2010understanding}, which is the default choice in the TensorFlow Keras Sequential module. All the weights are trained by RMSProp \citep{hinton2012neural} optimizer with the default setting, e.g. learning rate of $0.001$, etc. 
All ONN experiments are conducted using TensorFlow 2 with Python API.

\subsection{Simulated Data}
The learning rate for NTK+ES is $\eta=0.01$ and the GD update rule is as specified in (\ref{eqn:kernel_gd}). In the $\ell_2$-regularized methods, the tuning parameter $\mu$ for each task is chosen by cross validation. 
The validation dataset is of size 100 that is also noiseless and follows the same generating mechanism as the test dataset. 
For NTK$+\ell_2$, we use a grid search of interval $[0,1]$ with $\mu = 0.01, 0.02, \ldots, 1$ and for ONN$+\ell_2$, the $\mu$ candidates are $0.1,0.2,\ldots, 10$. In both cases, we observe that the optimal $\mu$ increases with the noise level $\sigma$. 
For $f_2^*$, we plot the chosen $\mu$ and $k^*$ for NTK$+\ell_2$ and NTK+ES respectively vs. $\sigma$. For each $\sigma$ value, the reported value is the average of 100 replications. The results are shown in 
Figure \ref{fig:cv}.

\begin{figure}[!htbp]
    \centering
    \includegraphics[scale = 0.5]{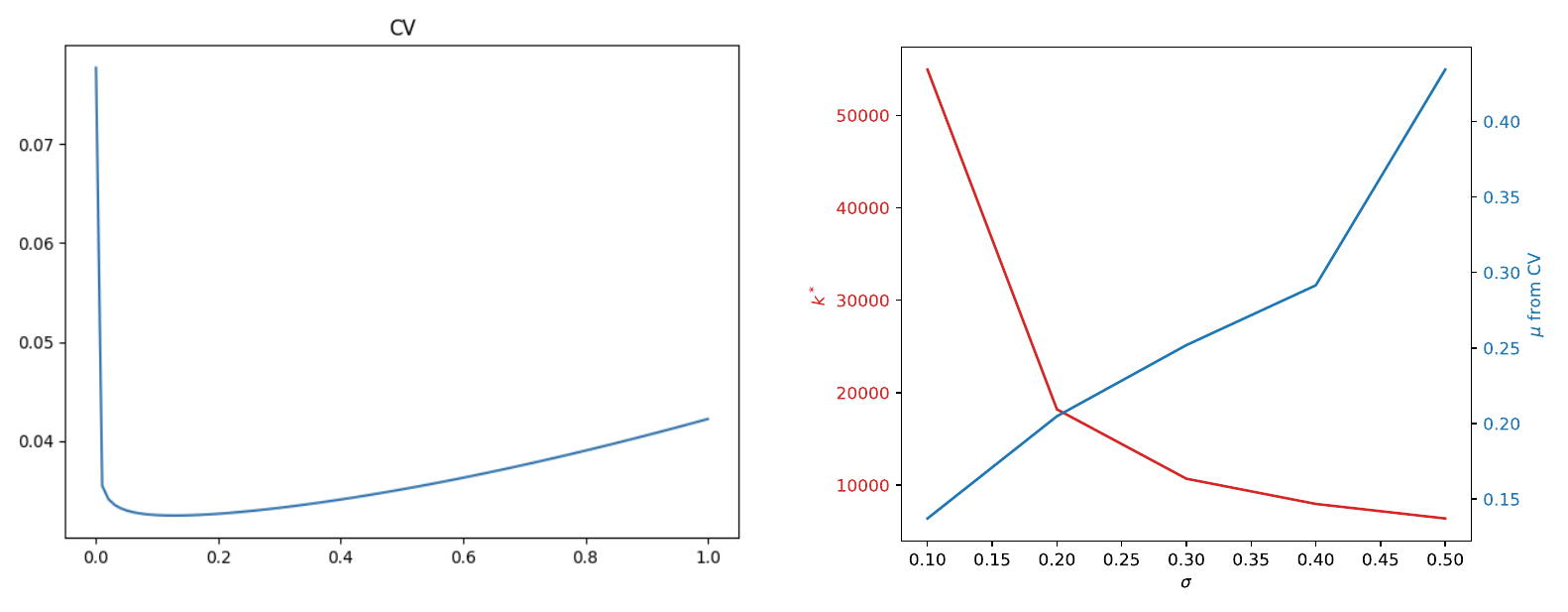}
    \caption{Left: Cross-validation of $\mu$ in NTK$+\ell_2$ for fitting $f^*_2$ when $\sigma=0.1$. The horizontal axis is values of $\mu$ (100 points from 0.01 to 1) and the vertical axis is the validation mean squared error. The cross-validated $\mu$ in this case is 0.13. 
    Right: Optimal stopping time $k^*$ in NTK+ES and cross-validated $\mu$ in NTK$+\ell_2$ for fitting $f_2^*$ are shown vs. $\sigma$. The optimal GD stopping time decrease with noise level while the best $\mu$ increases with $\sigma$. }
    \label{fig:cv}
\end{figure}

Figure \ref{fig:lossnoi} clearly demonstrates that ONN and NTK do not recover the true function well. As is explained in the paper, without regularization, overfitting the training data is harmful for the $L_2$ estimation. To illustrate this point, we show the trained estimators of $f_2^*$ for all the methods in Figure \ref{fig:visual0} when $\sigma=0.1$.

\begin{figure}[!htpb]
    \centering
    \includegraphics[width=1\textwidth]{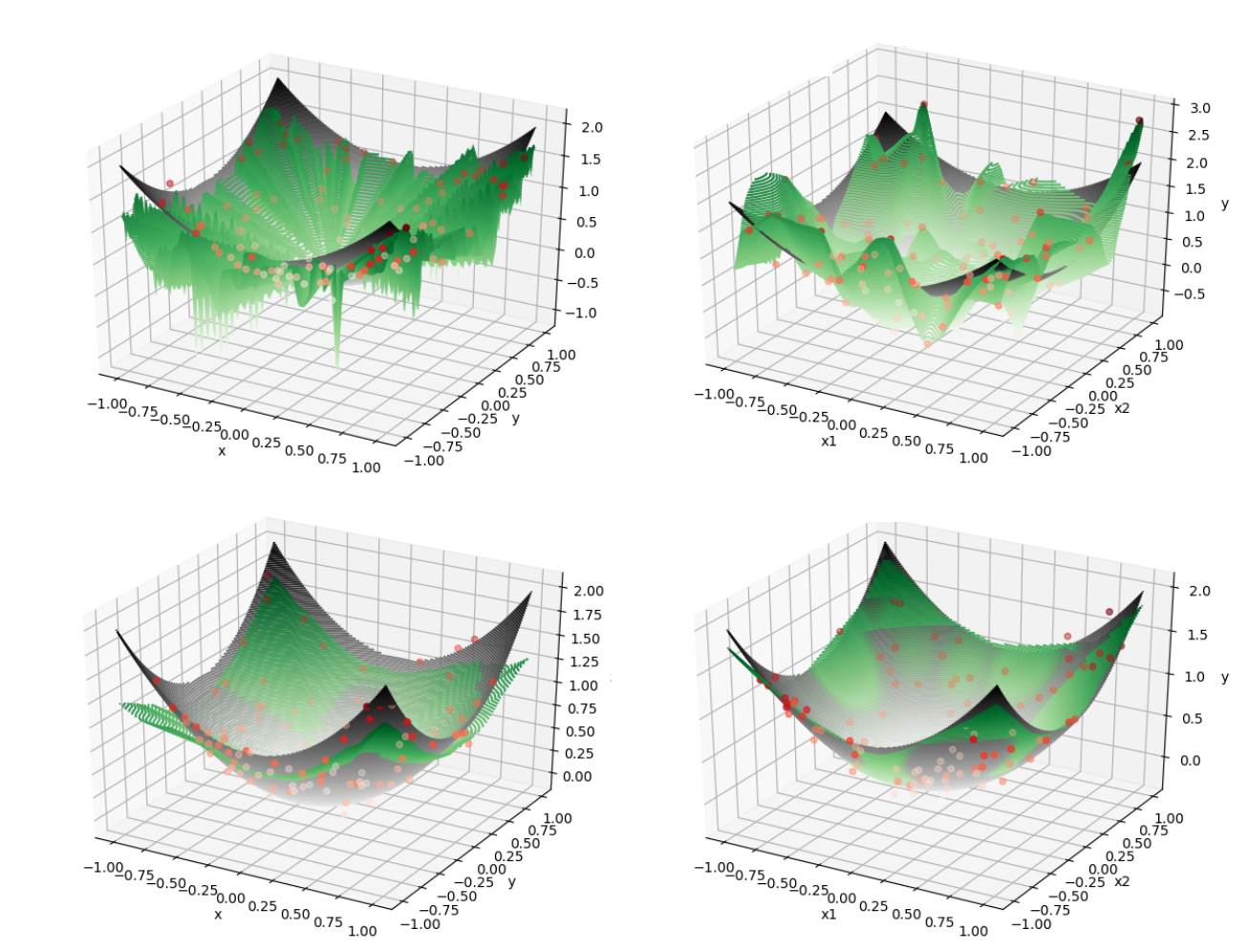}
    \caption{Visualizations for the trained estimators of NTK (top left), NTK$+\ell_2$ (bottom left), ONN (top right) and ONN$+\ell_2$ (bottom right). Training data are plotted as red dots. The green surface is the estimator and the grey surface is the true function $f^*_2$. Both surfaces are approximated by grid points $(i/100, j/100)$ for $i,j$ from $-100$ to 100. As can be seen in the top row, without regularization, the estimators overfit training data. The fitted estimators are very rough and don't recover the true function well. 
    }
    \label{fig:visual0}
\end{figure}
\newpage
\subsection{MNIST}
For images 5 and 8, the training and test split are the default.\footnote{http://yann.lecun.com/exdb/mnist/} We change label 5 and 8 to $-1$ and 1 respectively. No further pre-processing is done to the dataset. 
% We consider additive label noises in our experiments because it suits the nonparametric regression setting. 
For NTK+ES, the learning rate is $\eta=0.0001$ and the GD update rule is as specified in (\ref{eqn:kernel_gd}). To account for the high data dimension, we divide the NTK matrix $\bH^\infty$ by $d$. 
For the ONN+$\ell_2$ and NTK$+\ell_2$,
we choose $\mu$ by cross-validation and the candidates are $\mu= 1, 2, 5, 10, 20, 50, 100, 200, 500, 1000, 2000, 5000$ for ONN+$\ell_2$ and $\mu= 1, 2, 3,\ldots, 100$ for NTK+$\ell_2$.
The training/validation split is 80\%/20\% for cross-validation so the actual training data size is 9107 for all methods (ONN, NTK and NTK+ES do not use the validation dataset). 
The cross-validated $\mu$ for ONN$+\ell_2$ and optimal stopping time $k^*$ for NTK+ES are shown in Figure \ref{fig:mnist_mu}, together with the cross-validation results specifically for $\sigma=1$.

\begin{figure}[!htbp]
    \centering
    \includegraphics[width=1\textwidth]{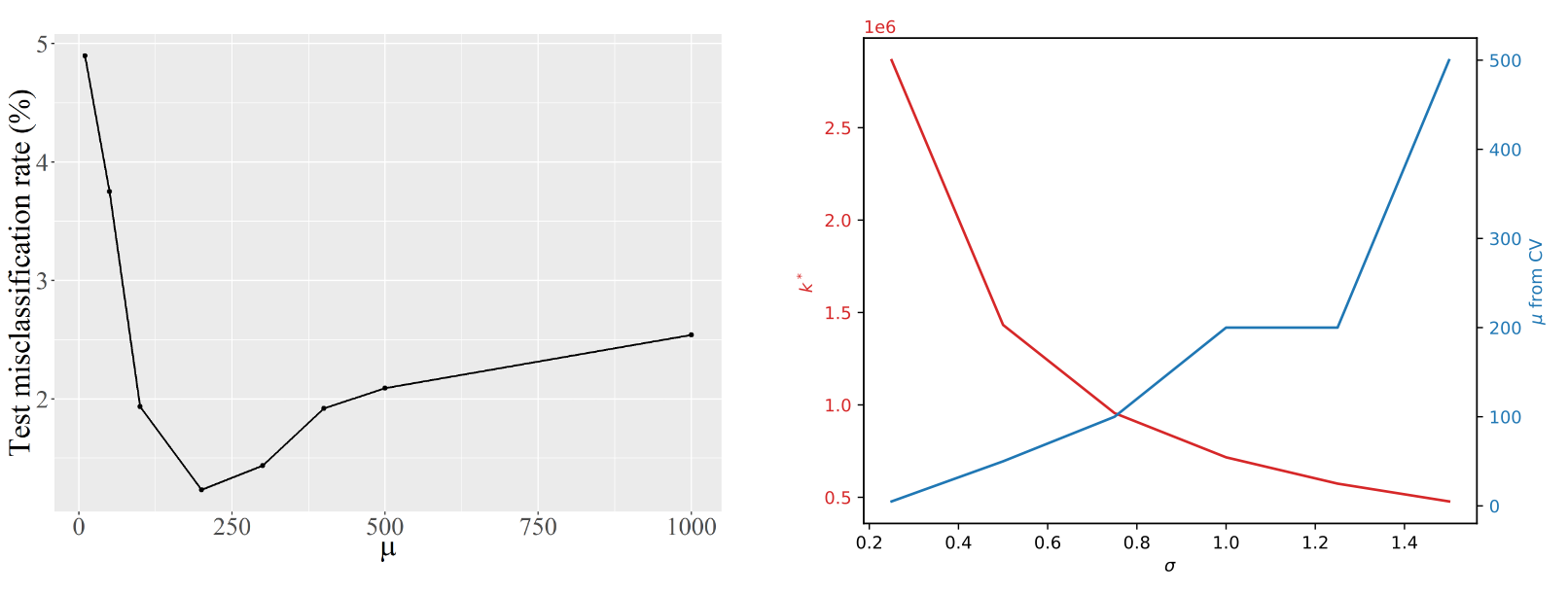}
    \caption{Left: Cross-validation result for $\mu$ in ONN$+\ell_2$ when $\sigma=1$ (with extra $\mu$ candidates of 300 and 400). In the range of $\mu=5$ to $\mu=1000$, we can clearly see a V-shape and the best $\mu$ in this case is 200. 
    Right: Optimal stopping time $k^*$ in NTK+ES and cross-validated $\mu$ in ONN$+\ell_2$ for MNIST dataset are shown vs. $\sigma$. The optimal stopping time decreases with noise level while the best $\mu$ increases with $\sigma$. }
    \label{fig:mnist_mu}
\end{figure}

\end{document}